\documentclass[jair,twoside,11pt,theapa]{article}
\usepackage{jair, theapa, rawfonts}

\ShortHeadings{DynamicHS: Streamlining Reiter's HS-Tree for Sequential Diagnosis}
{Rodler}
\firstpageno{1}

\usepackage{times}
\usepackage{soul}
\usepackage{url}
\usepackage[utf8]{inputenc}
\usepackage[small]{caption}
\usepackage{graphicx}
\usepackage{amsmath}
\usepackage{booktabs}
\usepackage{amsthm}
\usepackage{amsfonts}
\usepackage[T1]{fontenc}

\usepackage[multiple]{footmisc}

\usepackage{algcompatible}

\usepackage{fancybox}
\usepackage{amsfonts}

\usepackage{environ}
\usepackage{mathnotation}
\usepackage{multirow}

\usepackage{color}

\usepackage{xspace}
\usepackage{dashrule}
\usepackage[table,xcdraw]{xcolor}
\usepackage[inline]{enumitem}

\newcommand{\dpi}{\mathit{dpi}}
\newcommand{\sol}{\mathbf{diag}}
\newcommand{\conf}{\mathbf{conf}}
\newcommand{\dt}{\mathcal{D}^*}
\newcommand{\md}{\mathcal{D}}
\newcommand{\mD}{\mathbf{D}}
\newcommand{\mC}{\mathbf{C}}

\newcommand{\cost}{\mathsf{cost}}
\newcommand{\qqm}{\mathsf{heur}}
\newcommand{\state}{\mathsf{state}}
\newcommand{\nd}{\mathsf{nd}}
\newcommand{\node}{\mathsf{n}}
\newcommand{\ld}{\mathit{ld}}
\newcommand{\mo}{\mathcal{K}}
\newcommand{\mb}{\mathcal{B}}
\newcommand{\tax}{\mathit{ax}}
\newcommand{\Tp}{\mathit{P}}
\newcommand{\Tn}{\mathit{N}}
\newcommand{\tp}{\mathit{p}}
\newcommand{\tn}{\mathit{n}}
\newcommand{\mc}{\mathcal{C}}
\newcommand{\RQ}{{\mathit{R}}}
\newcommand{\Queue}{{\mathbf{Q}}}

\newcommand{\pr}{\mathit{pr}}

\newcommand{\XSup}[1]{\times_{(\supset #1)}}

\usepackage{algorithm}
\usepackage[noend]{algpseudocode}
\algrenewcommand\algorithmicrequire{\textbf{Input:}}
\algrenewcommand\algorithmicensure{\textbf{Output:}}
\algrenewcommand\alglinenumber[1]{\tiny #1:} 
\algnewcommand{\IfThen}[2]{
	\State \algorithmicif\ #1\ \algorithmicthen\ #2
}

\usepackage[all]{xy}
\usepackage{tikz}

\newcounter{examplecounter}
\newenvironment{example}{
	\refstepcounter{examplecounter}%
	
	\vspace{7pt}
	\noindent\textbf{Example \arabic{examplecounter}}%
	\quad
}{
	
	\vspace{7pt}
}

\newtheorem{theorem}{Theorem}

\newtheorem{property}{Property}
\newtheorem{problem}{Problem}

\begin{document}
	
	\title{DynamicHS: \\ Streamlining Reiter's Hitting-Set Tree for Sequential Diagnosis\thanks{\; Earlier and significantly shorter versions \protect\cite{rodler2020ecai,rodler_dynHS_dx,rodler2020socs} of this paper were accepted and presented at the \emph{31st International Workshop on Principles of Diagnosis (DX'20)}, the \emph{Symposium on Combinatorial Search (SoCS'20)}, as well as at the \emph{24th European Conference on Artificial Intelligence (ECAI'20)}. The present version extends these earlier versions in various regards, providing i.a.\ a more detailed discussion of the proposed algorithm, a correctness proof, additional examples, a more comprehensive treatment of related works, and a substantially augmented evaluation section.}}
	
	\author{\name Patrick Rodler \email patrick.rodler@aau.at \\
			\addr Universit\"atsstr.\ 65-67, 9020 Klagenfurt, Austria}

\maketitle             
\begin{abstract}
Given a system that does not work as expected, Sequential Diagnosis (SD) aims at suggesting a series of system measurements to isolate the true explanation for the system's misbehavior from a potentially exponential set of possible explanations. To reason about the best next measurement, SD methods usually require a sample of possible fault explanations at each step of the iterative diagnostic process. The computation of this sample can be accomplished by various diagnostic search algorithms. Among those, Reiter's HS-Tree is one of the most popular due its desirable properties and general applicability. Usually, HS-Tree is used in a stateless fashion throughout the SD process to (re)compute a sample of possible fault explanations in each iteration, each time given the latest (updated) system knowledge including all so-far collected measurements. At this, the built search tree is discarded between two iterations, although often large parts of the tree have to be rebuilt in the next iteration, involving redundant operations and calls to costly reasoning services.

As a remedy to this, we propose DynamicHS, a variant of HS-Tree that maintains state throughout the diagnostic session and additionally embraces special strategies to minimize the number of expensive reasoner invocations. DynamicHS provides an answer to a longstanding question posed by Raymond Reiter in his seminal paper from 1987, where he wondered whether there is a reasonable strategy to reuse an existing search tree to compute fault explanations after new system information is obtained.

We conducted extensive evaluations on real-world diagnosis problems from the domain of knowledge-based systems---a field where the usage of HS-Tree is state-of-the-art---under various diagnosis scenarios in terms of the number of fault explanations computed and the heuristic for measurement selection used. The results prove the reasonability of the novel approach and testify its clear superiority to HS-Tree wrt.\ computation time. More specifically: \emph{(1)}~DynamicHS required less time than HS-Tree in 96\,\% of the executed sequential diagnosis sessions. \emph{(2)}~DynamicHS exhibited substantial and statistically significant time savings over HS-Tree in most scenarios, with median and maximal savings of 52\,\% and 75\,\%, respectively. \emph{(3)}~The relative amount of saved time appears to neither depend on the number of computed fault explanations nor on the used measurement selection heuristic. \emph{(4)}~In the hardest (most time-intensive) cases per diagnosis scenario, DynamicHS achieved even higher savings than on average, and could avoid median and maximal time overheads of over 175\,\% and 800\,\%, respectively, as opposed to a usage of HS-Tree.

Remarkably, DynamicHS achieves these performance improvements while preserving all desirable properties as well as the general applicability of HS-Tree.     
\end{abstract}



%
%
%
\section{Introduction}
\emph{Model-based diagnosis} \cite{Reiter87,dekleer1987} is a popular, well-understood, domain-independent and principled paradigm that has over the last decades found widespread adoption for troubleshooting systems as different as programs, circuits, physical devices, knowledge bases, spreadsheets, production plans, robots, vehicles, or aircrafts \cite{ng1990model,sachenbacher1998electrics,gorinevsky2002model,steinbauer2005detecting,DBLP:conf/ieaaie/FelfernigMMST09,DBLP:journals/jair/FeldmanPG10a,jannach2010toward,wotawa2010fault,Rodler2015phd,rodler_teppan2020}.
The theory of model-based diagnosis 
\cite{Reiter87} assumes a \emph{system}
that is composed of a set of \emph{components}, and a formal \emph{system description}. The latter is given as a logical knowledge base and can be used to derive the expected behavior of the system by means of automated deduction systems. If the predicted system behavior, under the assumption that all components are functioning normally, 
is not in line with \emph{observations} made about the system,
the goal is to locate the abnormal system components that are responsible for this discrepancy. Given a \emph{diagnosis problem instance (DPI)}---consisting of the system description, the system components, and the observations---a \emph{diagnosis} is a set of components that, when assumed abnormal, 
leads to consistency between system description (predicted behavior) and observations (real behavior). 
In many cases, there 
is a substantial number
of different diagnoses given the initial system observations. However, only one of the diagnoses (which we refer to as the \emph{actual diagnosis}) pinpoints the actually faulty components. 
To isolate the actual diagnosis, \emph{sequential diagnosis} \cite{dekleer1987} methods collect additional system observations (called \emph{measurements}) to gradually refine the set of diagnoses. 

The basic idea behind measurements is to exploit the fact that different diagnoses predict different system behaviors or properties (e.g., intermediate values or outputs). Observing a system aspect for which the predictions of 
diagnoses disagree then leads to the invalidation of those diagnoses whose predictions are inconsistent with the observation.  
Since 
\begin{enumerate*}[label=\emph{(\roman*)}]
	\item the amount of useful information gained differs significantly for different possible measurements,
	\item performing a measurement is usually costly (as, e.g., a human operator needs to take action), and
	\item determining (nearly) optimal measurement points is beyond the reach of humans for sufficiently complex systems,
\end{enumerate*}
it is important to have a diagnosis system provide appropriate measurement suggestions automatically.
To this end, several (efficiently computable) \emph{heuristics} \cite{moret1982decision,dekleer1987,pattipati1990,Shchekotykhin2012,Rodler2013,DBLP:journals/corr/Rodler16a,rodler17dx_activelearning,DBLP:conf/ruleml/RodlerS18}, that define the goodness of measurements based on various information-theoretic considerations, have been proposed to deal with the problem that optimal measurement proposal is NP-hard \cite{hyafil1976}. These heuristics, in principle, assess the impact of measurements by comparing a-priori (before measurement is taken) and expected a-posteriori (after measurement outcome is known) situations regarding the (known) diagnoses and their properties (e.g., probabilities). That is, the evaluation of the expected utility of measurements requires the knowledge of a sample of diagnoses. This sample is often termed the \emph{leading diagnoses} and usually defined as the best diagnoses according to some preference criterion such as minimal cardinality (where a minimal \emph{number} of components is assumed abnormal) or maximal probability. Beside such criteria, for computational and pragmatic reasons, a least requirement is usually that only minimal diagnoses
are computed (cf.\ the ``principle of parsimony'' \cite{Reiter87}). A \emph{minimal diagnosis} is one that does not strictly contain (in the sense of set-inclusion) any diagnosis.

A generic \emph{sequential diagnosis process} (see Fig.~\ref{fig:SD_overview}) can thus be thought of as a recurring execution of 
\begin{enumerate*}[label=\emph{(\roman{*})}]
	\item the computation of a set of leading minimal diagnoses,
	\item the selection of the most informative measurement based on these,
	\item the conduction of measurement actions (by some oracle or user), and
	\item the exploitation of the measurement outcome to refine the system knowledge.
\end{enumerate*}
This iterative process continues until sufficient diagnostic certainty is obtained (e.g., one diagnosis has overwhelming probability). 

Decisive factors for successful and practicable sequential diagnosis are a low number of measurements and a low cost per measurement (effort for human operator), as well as reasonable system computation times (waiting time of human operator). While the first two factors depend on a suitable measurement selection, the latter depends largely on the efficiency of the 
algorithm for (leading) diagnoses computation \cite{Shchekotykhin2014,DBLP:journals/corr/Rodler2017,rodler2020mbd_sampling}.

One of the most popular and widely used such algorithms is Reiter's HS-Tree \cite{Reiter87}, which is adopted in various domains such as for the debugging of software \cite{wotawa2010fault,abreu2011simultaneous} or ontologies and knowledge bases \cite{friedrich2005gdm,kalyanpur2006thesis,Horridge2011a,meilicke2011thesis,Rodler2015phd}, or for the diagnosis of hardware \cite{friedrich1999model,zaman2013integrated} or of recommender and configuration systems \cite{DBLP:journals/ai/FelfernigFJS04,felfernig2008intelligent}.
The main reasons for 
the widespread adoption of HS-Tree are that
\begin{enumerate*}[label=\emph{(\roman*)}]
	\item \emph{it is broadly applicable}, because all it assumes is a system description in some (monotonic\footnote{\emph{Monotonicity} means that the entailment relation $\models$ of the logical language satisfies $\alpha \models \beta \implies \alpha \cup \gamma \models \beta$ for all sets of sentences $\alpha, \beta, \gamma$ over this language. That is, adding additional sentences $\gamma$ to a knowledge base $\alpha$ does not invalidate any conclusions that could be drawn from $\alpha$ beforehand.}) knowledge representation language for which a sound and complete inference method exists,
	\item \emph{it is sound and complete}, as it computes only and all\footnote{Given a non-empty set of minimal diagnoses for a DPI, the problem of deciding whether there is a minimal diagnosis for this DPI which is not in the given set is NP-hard \cite{Bylander1991}. Hence, the completeness of HS-Tree (or any other algorithm that computes minimal diagnoses) guarantees the computation of all minimal diagnoses for an arbitrary DPI (only) if unlimited time and memory are assumed. Note that in practical diagnosis applications the computation of (even multiple) minimal diagnoses is often accomplishable in reasonable time.} minimal diagnoses, and
	\item \emph{it computes diagnoses in best-first order} according to a given preference criterion.
\end{enumerate*}

\begin{figure}[t]
	\centering
	\includegraphics[width=0.6\columnwidth]{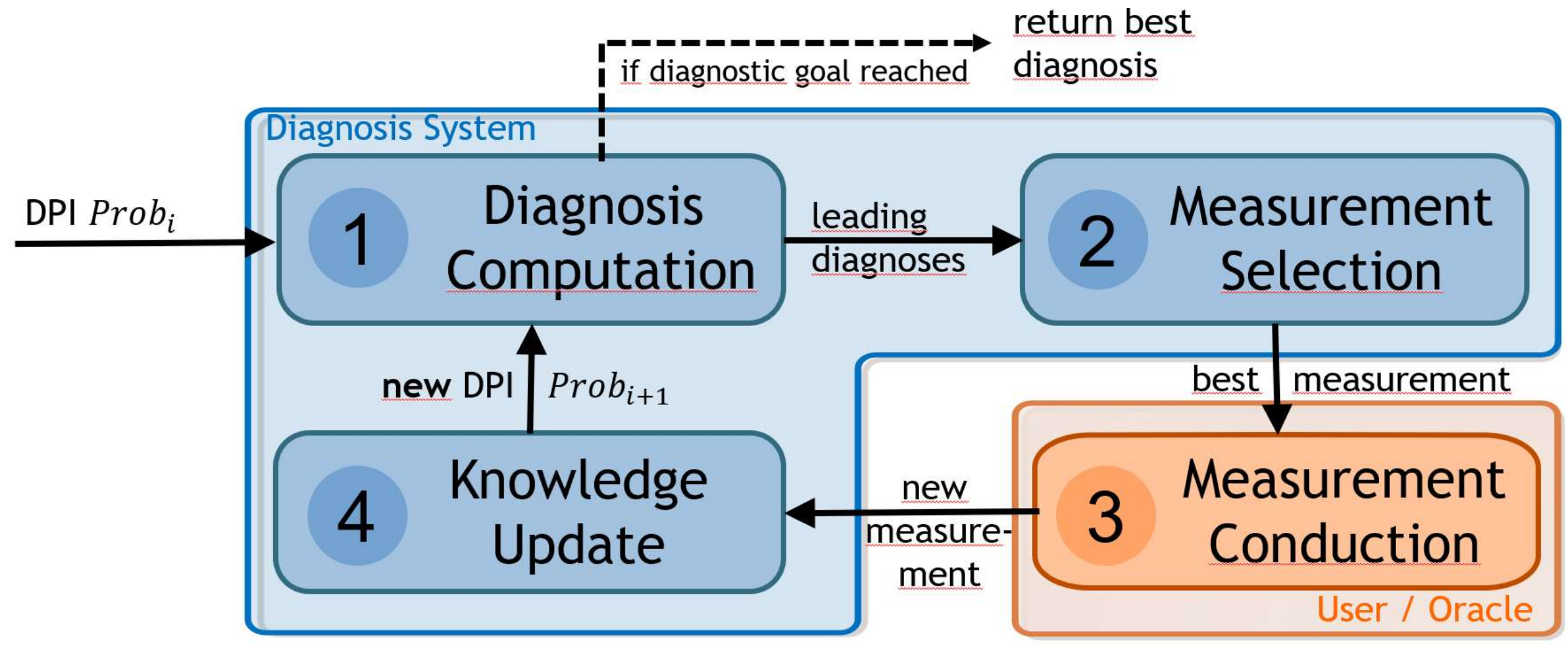}
	\caption{Overview of a generic sequential diagnosis process. Note that the diagnosis problem instance (DPI) changes between phases 4 and 1 in each iteration, i.e., $\mathit{Prob}_{i+1}$ is the reault of adding the new measurement obtained from phase 3 to $\mathit{Prob}_i$.}
	\label{fig:SD_overview}
\end{figure}

However, HS-Tree per-se does not encompass any specific provisions for being used in an iterative way. In other words, the DPI to be solved is assumed constant throughout the execution of HS-Tree. As a consequence of that, the question we address in this work is whether HS-Tree can be optimized for adoption \emph{in a sequential diagnosis scenario}, where the DPI to be solved is subject to successive change (information acquisition through measurements).
Already Raymond Reiter, in his seminal paper \cite{Reiter87} from 1987, 
asked: 
\begin{quotation}
	\emph{When new diagnoses do arise as a result of system measurements, can we determine these new diagnoses in a reasonable way from the (\dots) HS-Tree already computed in determining the old diagnoses?}
\end{quotation}
To the best of our knowledge, no study or algorithm has yet been proposed that sheds light on this very question.

As a result, 
sequential approaches which draw on HS-Tree for diagnosis computation
have to handle the varying DPI to be solved by re-invoking HS-Tree each time a new 
piece of system knowledge (measurement outcome)
is obtained. This amounts to a \emph{discard-and-rebuild} usage of HS-Tree, where the data structure (search tree) produced in one iteration is dropped prior to the next iteration, where a new 
one
is built from scratch. 
As the new tree obtained after incorporating the information about one measurement outcome usually quite closely resembles the existing tree, this approach generally requires substantial redundant computations. 
What even exacerbates this fact is that
these computations often involve a significant number of expensive reasoner calls. 
For instance, when debugging 
knowledge bases written in highly expressive logics such as OWL~2~DL, a single 
consistency check performed by
an inference service is already 2NEXPTIME-complete \cite{grau2008owl}.

Motivated by that, the \textbf{contribution of this work} is the development of 
DynamicHS, a novel \emph{stateful} variant of HS-Tree 
that pursues a \emph{reuse-and-adapt} strategy and is able to manage the dynamicity of the DPI throughout sequential diagnosis while avoiding the mentioned redundancy issues.
The main objective of this new algorithm is to allow for more efficient computations than HS-Tree while maintaining all the aforementioned advantages (generality, soundness, completeness, best-first property) of the latter. For that purpose, DynamicHS implements sophisticated techniques such as a time-efficient lazy update strategy or procedures to minimize expensive reasoning by trading it for cheaper 
(reasoning or set) operations.

In extensive evaluations we put DynamicHS to the test, using a benchmark of 20 real-world diagnosis problems from the domain of knowledge-based systems, which is one prominent field where HS-Tree is 
the prevalent tool for (sequential) diagnosis computation \cite{friedrich2005gdm,kalyanpur2006thesis,schlobach2007debugging,Horridge2011a,meilicke2011thesis,Shchekotykhin2012,Rodler2015phd,fu2016graph,baader2018weakening}. One major cause for the attractivity of HS-Tree 
in this domain are its mild assumptions made about the system description language 
and its resulting independence from the particular used language and reasoning engine, combined with the fact that a multitude of different logical languages are adopted in knowledge-based systems (with the goal to optimally trade off expressivity with reasoning efficiency).
For each diagnosis problem in our dataset, we experimented with DynamicHS and HS-Tree
under various \emph{diagnosis scenarios} wrt.\ the number of diagnoses computed per sequential diagnosis iteration and the employed measurement selection heuristic. The \textbf{main insights from our evaluations} are:
\begin{itemize}[noitemsep]
	\item DynamicHS is superior to HS-Tree in terms of computation time in 99.4\,\% of the investigated diagnosis scenarios, and in 96\,\% of all single sequential diagnosis sessions run. Roughly, these savings are achieved by trading less time (fewer redundant operations and reasoner calls) for more space (statefulness), where the additional memory required by DynamicHS was reasonable in the vast majority of the scenarios; and, whenever HS-Tree was applicable in our experiments in terms of memory requirements, DynamicHS was so as well.
	\item The runtime savings over HS-Tree achieved by DynamicHS are substantial and statistically significant in most scenarios, and reach median and maximal values of 52\,\% and 75\,\%. That is, HS-Tree requires up to an \emph{average} of four times the computation time of DynamicHS in the tested diagnosis scenarios. In \emph{single} diagnosis sessions, we observed that it took HS-Tree up to more than nine times as much time as DynamicHS---notably, while both algorithms always compute \emph{the same} solutions.
	\item The median runtime savings of DynamicHS per scenario appear to be neither dependent on the number of diagnoses computed nor on the measurement selection heuristic used.
	\item Considering the hardest cases per diagnosis scenario, which were up to one order of magnitude harder than the average cases, the time savings obtained by means of DynamicHS are even more substantial than on average, and reach median and maximal values of 64\,\% and 89\,\%, respectively.
\end{itemize}
The \textbf{organization of this work} is
as follows: Preliminaries on (sequential) model-based diagnosis are given in Sec.~\ref{sec:basics}. The novel algorithm DynamicHS is discussed in detail in Sec.~\ref{sec:dynamicHS}, where we follow a didactical approach by ``deriving'' DynamicHS from Reiter's well-known HS-Tree algorithm. Moreover, we prove the algorithm's correctness and detail some advanced techniques used by DynamicHS regarding tree management, reasoning operations and node storage. Related works are reviewed in Sec.~\ref{sec:related_work}. In Sec.~\ref{sec:eval}, we describe our experimental settings, the used test dataset, and thoroughly discuss the obtained results. More specifically, we first study the performance of DynamicHS, and then provide explanations for the outcomes through additional analyses. Finally, Sec.~\ref{sec:conclusion} concludes this work. 

\section{Preliminaries}
\label{sec:basics}
We briefly characterize the basic technical concepts used throughout this work, based on the framework of \cite{Shchekotykhin2012,Rodler2015phd} which is (slightly) more general \cite{rodler17dx_reducing} than Reiter's theory of diagnosis \cite{Reiter87}. In other words, any model-based diagnosis problem that can be formulated by means of Reiter's theory, can be equivalently stated by the following framework.
The main reason for using this more general framework is its ability to handle negative measurements (things that must \emph{not} be true for the diagnosed system)
which are helpful, e.g., for diagnosing knowledge bases \cite{DBLP:journals/ai/FelfernigFJS04,Shchekotykhin2012,schekotihin2018icbo_test-driven}.

\subsection{Diagnosis Problem Instance (DPI)}
We assume that the diagnosed system, consisting of a set of components $\setof{c_1,\dots,c_k}$, is described by a finite set of logical sentences $\mo \cup \mb$, where $\mo$ (possibly faulty sentences) includes knowledge about the behavior of the system components, and $\mb$ (correct background knowledge) comprises any additional available 
system knowledge and system observations. More precisely, there is a one-to-one relationship between sentences $\tax_i \in \mo$ and components $c_i$, where $\tax_i$ describes the nominal behavior
of $c_i$ (\emph{weak fault model}\footnote{\emph{Weak fault models}, in constrast to \emph{strong fault models}, define \emph{only the normal behavior} of the system components, and do not specify any behavior in case components are at fault \cite{feldman2008computing}. 
The weak fault model is also referred to as ``Ignorance of Abnormal Behavior'' property \cite{dekleer2008improved}.}). 
For instance, if $c_i$ is an AND-gate in a circuit, then $\tax_i := out(c_i) = and(in1(c_i),in2(c_i))$; $\mb$ in this example might encompass sentences stating, e.g., which components (gates) are connected by wires, 
or observed outputs of the circuit. The inclusion of a sentence $\tax_i$ in $\mo$ corresponds to the (implicit) assumption that $c_i$ is healthy. Evidence about the system behavior is captured by sets of positive ($\Tp$) and negative ($\Tn$) measurements \cite{dekleer1987,Reiter87,DBLP:journals/ai/FelfernigFJS04}. Each measurement is a logical sentence; positive ones, $\tp\in\Tp$, must be true, and negative ones, $\tn\in\Tn$, must not be true. The former can be, depending on the context, e.g., observations about the system, probes or required system properties. The latter model properties that must not hold for the system. For example, if the diagnosed system is a logical description of a university domain, a negative test case could be $\forall X\, (\mathit{researcher(X)} \rightarrow \mathit{professor(X)})$, i.e., that it must not be true that each researcher is a professor. 
We call $\tuple{\mo,\mb,\Tp,\Tn}$ a \emph{diagnosis problem instance (DPI)}. 

\begin{example}\label{ex:dpi}
Tab.~\ref{tab:example_DPI} depicts an example of a DPI, formulated in propositional logic. 
The ``system'' (which is the knowledge base itself in this case) comprises five ``components'' $c_1, \dots,c_5$, and the ``nominal behavior'' of $c_i$ is given by the respective axiom $\tax_i \in \mo$. There is neither any background knowledge ($\mb = \emptyset$) nor any positive test cases ($\Tp=\emptyset$) available from the start. But, there is one negative test case (i.e., $\Tn = \setof{\lnot A}$), which postulates that $\lnot A$ must \emph{not} be an entailment of the correct system (knowledge base). Note, however, that $\mo$ (i.e., the assumption that all ``components'' work nominally) in this case does entail $\lnot A$ (e.g., due to the axioms $\tax_1,\tax_2$) and therefore some axiom in $\mo$ must be faulty (i.e., some ``component'' is not healthy).\qed
\end{example}

\begin{table}
	\centering
	\footnotesize
	\caption{\small Example DPI stated in propositional logic.}
	\label{tab:example_DPI}
	\renewcommand{\arraystretch}{1}
	\begin{tabular}{@{}ccc@{}}
		\toprule
	\multicolumn{1}{ c  }{\multirow{2}{*}{$\mo\;=$} } & \multicolumn{2}{ l  }{$\{ \tax_1: A \to \lnot B$ \;\; $\tax_2: A \to B$ \;\; $\tax_3: A \to \lnot C$} \\
	& \multicolumn{2}{ l  }{$\phantom{\{} \tax_4: B \to C$ \;\;\,\hspace{4pt} $\tax_5: A \to B \lor C \qquad\qquad\,\quad\;\}$} \\
	\cmidrule{1-3}
	\multicolumn{3}{ l  }{$\mb =\emptyset \quad\qquad\qquad\qquad \Tp=\emptyset \quad\qquad\qquad\qquad \Tn=\setof{\lnot A}$} \\
	
	\bottomrule
\end{tabular}
\end{table}

\subsection{Diagnoses}
\label{sec:diagnoses}
Given that the system description along with the positive measurements (under the 
assumption $\mo$ that all components are healthy) is inconsistent, i.e., $\mo \cup \mb \cup \Tp \models \bot$, or some negative measurement is entailed, i.e., $\mo \cup \mb \cup \Tp \models \tn$ for some $\tn \in \Tn$, some assumption(s) about the healthiness of components, i.e., some sentences in $\mo$, must be retracted. We call such a set of sentences $\md \subseteq \mo$ a \emph{diagnosis} for the DPI $\tuple{\mo,\mb,\Tp,\Tn}$ iff $(\mo \setminus \md) \cup \mb \cup \Tp \not\models x$ for all $x \in \Tn \cup \setof{\bot}$. We say that $\md$ is a \emph{minimal diagnosis} for a DPI $\dpi$ iff there is no diagnosis $\md' \subset \md$ for $\dpi$. The set of minimal diagnoses is representative for all diagnoses (under the weak fault model \cite{Kleer1992}), i.e., 
the set of all diagnoses is exactly given by the set of all supersets of all minimal diagnoses.
Therefore, diagnosis approaches often restrict their focus to only minimal diagnoses. In the following, we denote the set of all minimal diagnoses for a DPI $\dpi$ by $\sol(\dpi)$. We furthermore denote by $\dt$ the (unknown) \emph{actual diagnosis} which pinpoints the actually faulty axioms, i.e., all elements of $\dt$ are in fact faulty and all elements of $\mo\setminus\dt$ are in fact correct.

\begin{example}\label{ex:diagnoses}
For our DPI $\dpi$ in Tab.~\ref{tab:example_DPI} we have four minimal diagnoses, given by $\md_1:=[\tax_1,\tax_3]$, $\md_2:=[\tax_1,\tax_4]$, $\md_3:=[\tax_2,\tax_3]$, and $\md_4 := [\tax_2,\tax_5]$, i.e., $\sol(\dpi) = \setof{\md_1,\dots,\md_4}$.\footnote{In this work, we denote diagnoses by square brackets.} For instance, $\md_1$ is a minimal diagnosis as $(\mo\setminus\md_1) \cup \mb\cup \Tp = \setof{\tax_2,\tax_4,\tax_5}$ is both consistent and does not entail the given negative test case $\lnot A$.\qed
\end{example}

\subsection{Conflicts}
\label{sec:conflicts}
Useful for the computation of diagnoses is the concept of a conflict \cite{dekleer1987,Reiter87}.
A conflict is a set of healthiness assumptions for components $c_i$ that cannot all hold given the current knowledge about the system. More formally, $\mc \subseteq \mo$ is a \emph{conflict} for the DPI $\tuple{\mo,\mb,\Tp,\Tn}$ iff $\mc \cup \mb \cup \Tp \models x$ for some $x \in \Tn \cup \setof{\bot}$. We call $\mc$ a \emph{minimal conflict} for a DPI $\dpi$ iff there is no conflict $\mc' \subset \mc$ for $\dpi$.
In the following, we denote the set of all minimal conflicts for a DPI $\dpi$ by $\conf(\dpi)$.
A (minimal) diagnosis for $\dpi$ is then a (minimal) hitting set of all minimal conflicts for $\dpi$ \cite{Reiter87}, where $X$ is a \emph{hitting set} of a collection of sets $\mathbf{S}$ iff $X \subseteq \bigcup_{S_i \in \mathbf{S}} S_i$ and $X \cap S_i \neq \emptyset$ for all $S_i \in S$.

\begin{example}\label{ex:conflicts}
For our running example, $\dpi$, in Tab.~\ref{tab:example_DPI}, there are four minimal conflicts, given by $\mc_1 := \langle\tax_1,\tax_2\rangle$, $\mc_2 := \langle\tax_2,\tax_3,\tax_4\rangle$, $\mc_3 := \langle\tax_1,\tax_3,\tax_5\rangle$, and $\mc_4 := \langle\tax_3,\tax_4,\tax_5\rangle$, i.e., $\conf(\dpi) = \setof{\mc_1,\dots,\mc_4}$.\footnote{In this work, we denote conflicts by angle brackets.} For instance, $\mc_4$, in CNF equal to $(\lnot A \lor \lnot C) \land (\lnot B \lor C) \land (\lnot A \lor B \lor C)$, is a conflict because, adding the unit clause $(A)$ to this CNF yields a contradiction, which is why the negative test case $\lnot A$ is an entailment of $\mc_4$. The minimality of the conflict $\mc_4$ can be verified by rotationally removing from $\mc_4$ a single axiom at the time and controlling for each so obtained subset that this subset is consistent and does not entail $\lnot A$.

For example, the minimal diagnosis $\md_1$ (see Example~\ref{ex:diagnoses}) is a hitting set of all minimal conflict sets because each conflict in $\conf(\dpi)$ contains $\tax_1$ or $\tax_3$. It is moreover a \emph{minimal} hitting set since the elimination of $\tax_1$ implies an empty intersection with, e.g., $\mc_1$, and the elimination of $\tax_3$ means that, e.g., $\mc_4$ is no longer hit.\qed   
\end{example}

Literature offers a variety of algorithms for conflict computation, e.g., \cite{junker04,marques2013minimal,shchekotykhin2015mergexplain}. Given a DPI $\dpi = \tuple{\mo,\mb,\Tp,\Tn}$ as input, one call to such an algorithm returns one minimal conflict for $\dpi$, if existent, and 'no conflict' otherwise. All algorithms require an appropriate theorem prover that is used as an oracle to perform consistency checks over the logic by which the DPI is expressed. In the worst case, none of the available algorithms requires fewer than 
$O(|\mo|)$ theorem prover calls per conflict computation \cite{marques2013minimal}. 
The performance of diagnosis computation methods depends largely on the complexity of consistency checking for the used logic and on the number of consistency checks executed (cf., e.g., \cite{pill2011eval_hitting_set_algos}). Since consistency checking is often NP-complete or beyond for practical problems \cite{kalyanpur2006thesis,grau2008owl,Horridge2011a,romero2012more,Rodler2019userstudy}, and diagnostic algorithms have no influence on the used system description language, it is pivotal to keep the number of conflict computations at a minimum.

\subsection{Sequential Diagnosis (SD) Problem}
We 
now define 
the 
sequential diagnosis (SD)
problem, which calls for a set of measurements for a given DPI, such that adding these measurements to the DPI implies diagnostic certainty, i.e., a single remaining minimal diagnosis for the new DPI. In the optimization version (OptSD), the goal is to find such a measurement set of minimal cost.   
\begin{problem}[(Opt)SD]\label{prob:SD}
	\textbf{Given:} A DPI $\tuple{\mo,\mb,\Tp,\Tn}$.
	\textbf{Find:} A (minimal-cost) set of measurements $\Tp' \cup \Tn'$ such that $|\sol(\tuple{\mo,\mb,\Tp\cup\Tp',\Tn\cup\Tn'})| = 1$.\footnote{There are different characterizations of the SD Problem in literature. E.g., it is sometimes formulated as the problem of finding a policy that suggests the next measurement given any possible situation that might be encountered starting from a given DPI. For the purpose of this paper, we basically stick to the conceptualization of \cite{Reiter87}, where measurements are sets of logical sentences.}
\end{problem}

A minimal-cost set of measurements $\Tp' \cup \Tn'$
minimizes $\sum_{m \in \Tp' \cup \Tn'} \cost(m)$, where $\cost(m)$ is the cost value assigned to measurement $m$. Possible cost functions range from simple unit cost assumptions for measurements ($\cost(m)= 1$ for all $m$) to sophisticated measurement time, effort, or complexity evaluations \cite{Rodler2015phd,DBLP:journals/corr/Rodler2017,Rodler2019userstudy}. 
%

\begin{example}\label{ex:seq_diag}
Let $\tuple{\mo,\mb,\Tp,\Tn}$ be the example DPI described in Tab.~\ref{tab:example_DPI}, and assume that $\dt = \md_2$, i.e., $[\tax_1,\tax_4]$ is the actual diagnosis. Then, one solution to Problem~\ref{prob:SD} is $\Tp'=\{A \land B\}$ and $\Tn'=\{C\}$. In other words, $\sol(\tuple{\mo,\mb,\Tp\cup\Tp',\Tn\cup\Tn'}) = \{\md_2\}$. An alternative solution would be $\Tp'=\{A \to \lnot C\}$ and $\Tn'=\{A\to C, A \to \lnot B\}$. If the solution cost is quantified by the number of measurements, then the first solution has a cost of 2, whereas the second has a cost of 3. We would thus prefer the first solution. 
If we alternatively adopt a more fine-grained cost analysis, considering also the estimated complexity of the measurements, e.g., that negation and implication are harder (to understand for an interacting domain expert who labels sentences as positive or negative measurements) than conjunction, then the first solution would be even more preferred to the second than under the assumption of uniform measurement costs. 
\qed
\end{example}
%
%

\subsection{Computational Issues in Sequential Diagnosis} 
\label{sec:computational_issues_in_SD}
In principle, the OptSD problem can be tackled by appropriate \emph{measurement point\footnote{
In the scope of this work, the relation between measurement and measurement point is as follows. A \emph{measurement} is a logical sentence that describes
the outcome of a \emph{measuring action} at some predefined \emph{measurement point}. For instance, when diagnosing a circuit, the measurement action could be the application of a voltmeter, and the measurement point a particular terminal of a gate in a circuit. The outcome would then be whether the measured wire $w$ is high (1) or low (0), yielding one of the (positive) measurements $w=1$ or $w=0$, respectively.}
selection\footnote{What we call \emph{measurement point selection} in this work is often simply referred to as \emph{measurement selection} in literature, e.g., \cite{friedrich1992choosing,chen2015approach}. Moreover, the measurement point selection might also include a prior \emph{measurement point computation}, for (e.g., knowledge-based) systems for which the possible measurement points are not explicitly given \cite{Shchekotykhin2012,DBLP:journals/corr/Rodler2017}.}} \cite{dekleer1987} approaches. 
However, due to the NP-hardness of OptSD \cite{hyafil1976}, 
sequential diagnosis systems, in general, can merely attempt to approximate its optimal solution. This is usually accomplished by means of a one-step-lookahead approach \cite{dekleer1992onesteplook} which evaluates (the gain of) possible measurement points based on the 
a-posteriori situations encountered for its different outcomes, weighted by their respective (estimated) probabilities. To make these estimations, a sample of (minimal) diagnoses, the \emph{leading diagnoses}, is usually taken as a basis \cite{dekleer1987,DBLP:journals/jair/FeldmanPG10a,Shchekotykhin2012,Rodler2015phd}. 
In the context of one-step-lookahead analysis, a range of measurement point evaluation \emph{heuristics} have been suggested in literature, such as the measurement point's information gain \cite{dekleer1987}, 
its ability to guarantee a maximal worst-case diagnoses invalidation rate \cite{moret1982decision,Shchekotykhin2012}, a dynamic combination thereof \cite{Rodler2013}, or the measurement point's goodness 
wrt.\ various active learning criteria \cite{DBLP:journals/corr/Rodler16a,rodler17dx_activelearning,DBLP:conf/ruleml/RodlerS18}. Besides these more sophisticated assessment criteria, a common minimal requirement to a rational measurement point selection strategy is that it suggests only \emph{discriminating measurement points}, i.e., for each measurement outcome there must be at least one 
leading 
diagnosis that is inconsistent with this outcome. In other words, proposing a discriminating measurement point guarantees the invalidation of some spurious fault hypothesis (or: the gain of some useful information) already a-priori, irrespective of the measurement outcome. 

While appropriate measurement point selection is one important factor determining the overall efficiency of the diagnostic process,
not only the cost associated with the \emph{measurement conduction} 
is of relevance from the viewpoint of a user interacting with a sequential diagnosis engine. In fact, there is another crucial factor which is left implicit in Problem~\ref{prob:SD}, namely the question how efficiently 
the \emph{measurement point computation} is accomplished by the diagnosis system.
Because, the \emph{overall} (user) time to diagnose a system is given by the overall measurement conduction time \emph{plus} the overall system computation time for the proposal of the measurement points. The following problem makes this explicit:
\begin{problem}[Efficient (Opt)SD]\label{prob:efficient_SD}
\textbf{Given:} A DPI $\tuple{\mo,\mb,\Tp,\Tn}$.
\textbf{Find:} A solution to (Opt)SD in minimal time. 
\end{problem}

Approaching a solution to Problem~\ref{prob:efficient_SD} involves the minimization of what we call the system \emph{reaction time}, i.e., the computation time required by the system after being informed about the outcome of one 
measuring action
until the provision of the next measurement point. The reaction time is composed of the time required for \emph{(i)}~the \emph{system knowledge update} (based on the information acquired through a new measurement), \emph{(ii)}~the (leading) \emph{diagnosis computation}, as well as for \emph{(iii)}~the \emph{measurement point selection}. These main factors affecting the sequential diagnosis time are summarized in Fig.~\ref{fig:categorization_of_time_factors_in_SD}. 

Whereas the efficiency-optimization of measurement point selection algorithms\footnote{Measurement selection algorithms \cite{dekleer1987} are given a set of leading diagnoses, the current system knowledge (DPI), as well as the available probabilistic information and a selection heuristic, and should output ``the best'' next measurement according to the given heuristic.}
can be largely or even fully decoupled from diagnosis computation and knowledge update algorithms, the latter two might need to be studied in common, depending on whether the diagnosis computation is stateless \cite{Shchekotykhin2012,Shchekotykhin2014} or stateful \cite{dekleer1987,Siddiqi2011,Rodler2015phd}.
More specifically, the elaborateness of the knowledge update process might range from almost trivial (addition of a new measurement to the DPI; update of probabilities) in the stateless case, to very sophisticated in the stateful case, where, e.g., the state of a persistent data structure used throughout sequential diagnosis must be additionally updated. 

We present in this work DynamicHS, a stateful diagnosis computation algorithm that addresses Problem~\ref{prob:efficient_SD} by \emph{focusing on the efficiency of the diagnosis computation and knowledge update processes} (cf.\ Fig.~\ref{fig:categorization_of_time_factors_in_SD}).
In fact, we will provide empirical evidence (in Sec.~\ref{sec:eval}) that the novel algorithm is more time-efficient for sequential diagnosis than 
Reiter's HS-Tree, which is a state-of-the-art diagnosis computation approach, e.g., in the knowledge base and ontology debugging domains \cite{friedrich2005gdm,kalyanpur2006thesis,schlobach2007debugging,Horridge2011a,meilicke2011thesis,Shchekotykhin2012,Rodler2015phd,fu2016graph,baader2018weakening}. In particular, DynamicHS aims at achieving an efficiency gain over HS-Tree while \emph{computing the same solution} to Problem~\ref{prob:SD} as HS-Tree, while \emph{preserving all the (desirable) properties} of HS-Tree, and while \emph{maintaining full compatibility with other approaches}
that tackle Problem~\ref{prob:efficient_SD} from the perspective of making the measurement point selection more efficient (such as \cite{rodler17dx_queries}), or that attack Problem~\ref{prob:SD} to achieve a reduction of measurement conduction time (such as \cite{Shchekotykhin2012,Rodler2013,rodler2019aie}).


\begin{figure}[t]
	\includegraphics[width=\columnwidth]{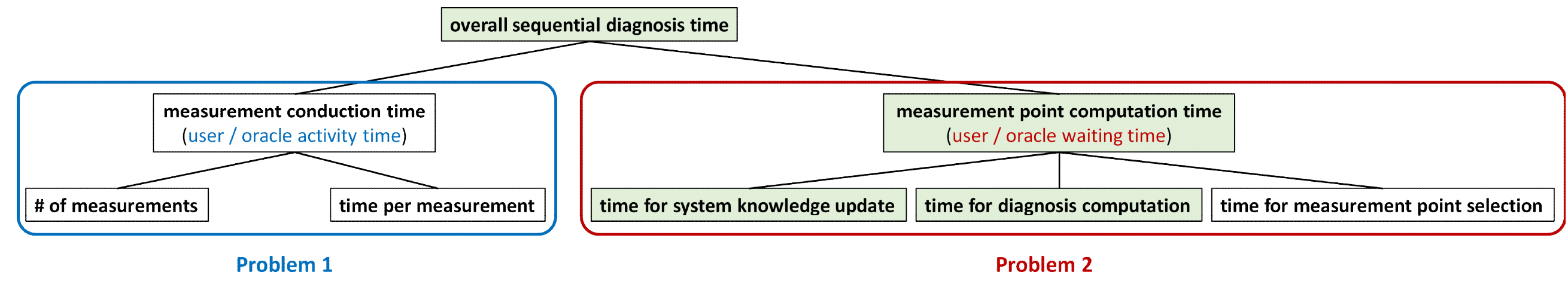}
	\caption{Categorization of the factors that influence the overall sequential diagnosis time. The blue / red frame refers to the OptSD / Efficient OptSD problem. The factors shaded green are those addressed by the DynamicHS algorithm suggested in this work. For the figure to reflect a more general viewpoint (including also resources other than time), one may replace each occurrence of ``time'' in the figure by ``cost''.}
	\label{fig:categorization_of_time_factors_in_SD}
\end{figure}

\section{DynamicHS Algorithm}
\label{sec:dynamicHS}
\subsection{Outline of DynamicHS}
\label{sec:dynHS_outline}
\subsubsection{Motivation and Principle}
Starting from Reiter's HS-Tree, the main observation that builds the ground for the development of DynamicHS is that, throughout a sequential diagnosis session, the initially given DPI is subject to gradual change.
That is, new measurements are successively added to it, and usually significant portions of a built hitting set tree remain unaffected after the transition from one DPI to the next. Thus, discarding the existing tree and (re)constructing a similar tree in the next iteration potentially involves a significant number of redundant (and expensive) operations. With DynamicHS, we account for that by proposing to replace a discard-and-rebuild paradigm by an adapt-and-reuse one---that is, the built hitting set tree is persistently maintained and accordingly adapted each time a new measurement comes in. 
These tree updates, which mainly include tree pruning and node relabeling steps, are required to guarantee the correctness of the diagnosis computation in the next iteration, i.e., 
for the \emph{newest} DPI. 
The underlying rationale of DynamicHS is to avoid unnecessary repeated actions and to cleverly leverage the reused search data structure to trade expensive reasoning against cheaper operations, such as set comparisons or less expensive (but equally informative) reasoning.
%
%
\subsubsection{Inputs and Outputs}
DynamicHS (depicted by Alg.~\ref{algo:dynamic_hs} on page~\pageref{algo:dynamic_hs}) is a procedure that computes a set of minimal diagnoses for a DPI 
$\langle\mo,\mb,\Tp \cup \Tp',\Tn \cup \Tn'\rangle$ 
in best-first order (as per some given goodness criterion, which we assume to be a probability measure $\pr$ in this work).
DynamicHS accepts the following arguments: 
\emph{(1)}~an initial DPI $\dpi_0 = \tuple{\mo,\mb,\Tp,\Tn}$,
\emph{(2)}~the already accumulated positive and negative measurements $\Tp'$ and $\Tn'$, respectively,
\emph{(3)}~a probability measure $\pr$ which is exploited to compute diagnoses in descending order based on their probabilities, 
\emph{(4)}~a stipulated number $\ld\geq 2$ of diagnoses to be returned,\footnote{The reason why at least two diagnoses should be computed is that \emph{(i)}~two or more minimal diagnoses are needed for the computation of a discriminating measurement point \cite{Rodler2015phd} (cf.\ Sec.~\ref{sec:computational_issues_in_SD}), and \emph{(ii)}~this way the stop criterion ``only a single minimal diagnosis remaining'' (cf.\ Alg.~\ref{algo:sequential_diagnosis}) guarantees that the sequential diagnosis process ends only if diagnostic certainty has been achieved (correct diagnosis among \emph{all} minimal diagnoses found).} 
\emph{(5)}~the set of those diagnoses returned by the previous run of DynamicHS that are consistent ($\mD_{\checkmark}$) and those that are inconsistent ($\mD_{\times}$) with the last added measurement, 
and
\emph{(6)}~a tuple of variables $\state$, which altogether describe DynamicHS's current state. 

\begin{algorithm}[!tbp]
	\scriptsize
	\caption{Sequential Diagnosis}\label{algo:sequential_diagnosis}
	\begin{algorithmic}[1]
		\Require DPI $\dpi_0 := \langle\mo,\mb,\Tp,\Tn\rangle$, 
		probability measure $\pr$ (to compute 
		diagnoses probabilities),
		number $\ld$ ($\geq 2$) of minimal diagnoses to be computed per iteration, 
		heuristic $\qqm$ for measurement selection, boolean $\mathit{dynamic}$ that governs which diagnosis computation algorithm is used (DynamicHS if true, HS-Tree otherwise)  
		\Ensure $\setof{\md}$ where $\md$ is the final diagnosis after solving SD (Problem~\ref{prob:SD})
		\vspace{6pt}
		
		\State $\Tp'\gets \emptyset, \Tn'\gets \emptyset$ \label{algoline:inter_onto_debug:var_init_start} \Comment{performed measurements}
		\State $\mD_{\checkmark} \gets \emptyset, \mD_{\times} \gets \emptyset$ \label{algoline:inter_onto_debug:var_init_mD_checkmark+mD_times} \Comment{variables describing state...}
		\State $\state \gets \tuple{[\,[]\,],[\,],\emptyset,\emptyset}$ \label{algoline:inter_onto_debug:var_init_end} \Comment{...of DynamicHS tree}
		\While{\true}  \label{algoline:inter_onto_debug:while}
		\If{$\mathit{dynamic}$}
		\State $\tuple{\mD,\state} \gets \Call{DynamicHS}{\dpi_0,\Tp',\Tn',\pr, \ld, \mD_{\checkmark}, \mD_{\times}, \state}$
		\label{algoline:inter_onto_debug:call_DynHS}
		\Else
		\State $\mD \gets \Call{HS-Tree}{\dpi_0,\Tp',\Tn',\pr, \ld}$
		\label{algoline:inter_onto_debug:call_HS-Tree}
		\EndIf
		\IfThen {$|\mD|=1$}
		{\Return $\mD$}
		\label{algoline:inter_onto_debug:if_goal_reached}
		\State $mp \gets \Call{computeBestMeasPoint}{\mD,\dpi_0,\Tp',\Tn', \pr, \qqm}$   \label{algoline:inter_onto_debug:calc_query} 
		\State $\mathit{outcome} \gets \Call{performMeas}{mp}$   \Comment{oracle inquiry (user interaction)} \label{algoline:inter_onto_debug:perform_meas}
		\State $\tuple{\Tp', \Tn'} \gets \Call{addMeas}{mp,outcome,\Tp', \Tn'}$ 
		\label{algoline:inter_onto_debug:add_meas}
		\If{$\mathit{dynamic}$}
		\State $\tuple{\mD_{\checkmark}, \mD_{\times}} \gets \Call{assignDiagsOkNok}{\mD,\dpi_0,\Tp', \Tn'}$
		\label{algoline:inter_onto_debug:partition_D_into_Dcheckmark_and_Dtimes}
		\EndIf
		\EndWhile
	\end{algorithmic}
	\normalsize
\end{algorithm}

\subsubsection{Embedding in Sequential Diagnosis Process}
\label{sec:embedding_in_seq_diag_process}
Alg.~\ref{algo:sequential_diagnosis} sketches a generic sequential diagnosis algorithm and shows how it accommodates DynamicHS (line~\ref{algoline:inter_onto_debug:call_DynHS}) or, alternatively, Reiter's HS-Tree (line~\ref{algoline:inter_onto_debug:call_HS-Tree}), as methods for iterative diagnosis computation. 
The algorithm 
reiterates a while-loop (line~\ref{algoline:inter_onto_debug:while}) 
until the solution space of minimal diagnoses includes only a single element.\footnote{Of course, less rigorous stopping criteria are possible, e.g., one might appoint a probability threshold and stop once the probability of one computed diagnosis exceeds this threshold \cite{dekleer1987}.} 
Since both DynamicHS and HS-Tree are complete (see later) 
and always attempt to compute 
at least two
diagnoses ($\ld \geq 2$), this stop criterion is met
 iff a diagnoses set $\mD$ with $|\mD| = 1$ is output (line~\ref{algoline:inter_onto_debug:if_goal_reached}). On the other hand, as long as $|\mD| > 1$, the algorithm 
seeks to acquire additional information
to rule out further elements in $\mD$.
To this end, the best next measurement point $mp$ is computed (\textsc{computeBestMeasPoint}, line~\ref{algoline:inter_onto_debug:calc_query}), using the current system information---$\dpi_0$, $\mD$, and acquired measurements $\Tp'$, $\Tn'$---as well as the given probabilistic information $\pr$ and some selection heuristic $\qqm$ (which defines what ``best'' means, cf.\ \cite{rodler17dx_activelearning}).
%
The conduction of the measurement at $mp$ (\textsc{performMeas}, line~\ref{algoline:inter_onto_debug:perform_meas}) is usually accomplished by a qualified 
user 
(\emph{oracle}) that interacts with the sequential diagnosis system, e.g., an electrical engineer for a defective circuit, or a domain expert in case of a faulty knowledge base.
The measurement point $mp$ along with its result $\mathit{outcome}$ are used to formulate a logical sentence $m$ that is either added to $\Tp'$ if $m$ constitutes a positive measurement, and to $\Tn'$ otherwise (\textsc{addMeas}, line~\ref{algoline:inter_onto_debug:add_meas}).\footnote{Depending on the diagnosed system, the notion of a ``measurement point'' might be quite differently interpreted. For instance, in case of a physical system \cite{dekleer1987}, it might be literally a point $mp$ in the system where to make a measurement with
some physical measurement instrument (e.g., a voltmeter). 
In case of a knowledge-based system \cite{Shchekotykhin2012}, on the other hand, it might be, e.g., a question $mp$ in the form of a logical sentence asked to an expert about the domain described by a problematic knowledge base; e.g., one could ask $mp := \forall X\, (\mathit{bird(X)} \rightarrow \mathit{canFly(X)})$ (``can all birds fly?''). Likewise, the ``outcome'' of a measurement might be
the measured value $v$ at the measurement point $mp$
in the former case, and a truth value answering the asked question $mp$ in the latter. 
The new ``measurement'' (logical sentence), formulated from the measurement point and the respective outcome, could be $m_{phys}:=mp=v$ in the former, and $m_{\mathit{KB}}:=mp$ in the latter case. At this, $m_{phys}$ would be added to the positive measurements $\Tp'$ (since $m_{phys}$ is now a fact), whereas $m_{\mathit{KB}}$ would be added to the positive (negative) measurements $\Tp'$ ($\Tn'$) if the answer to the asked question was positive (negative), i.e., $m_{\mathit{KB}}$ must (not) be true in the correct knowledge base (cf.\ \cite{Rodler2015phd}).}
%
%
Finally, if DynamicHS is adopted, 
the set of diagnoses $\mD$ is partitioned 
into those consistent ($\mD_{\checkmark}$) and those inconsistent ($\mD_{\times}$) with the newly added measurement $m$ (\textsc{assignDiagsOkNok}, line~\ref{algoline:inter_onto_debug:partition_D_into_Dcheckmark_and_Dtimes}).

\subsubsection{Properties}
DynamicHS is \emph{stateful} in that one and the same tree data structure is used throughout the entire sequential diagnostic process. Apart from that, it preserves all (desirable) properties of Reiter's HS-Tree, i.e., the \emph{soundness} (only minimal diagnoses for the current DPI are found) and \emph{completeness} (all minimal diagnoses for the current DPI can be found) of the diagnosis search in each iteration, as well as the generality (\emph{logics- and reasoner-independence}) and the feature to \emph{calculate diagnoses in most-preferred-first order}.
\subsubsection{Specialized Incorporated Techniques}
To keep the extent---and thus the impact on the computational complexity---of the regularly performed tree update actions at a minimum, DynamicHS incorporates a specialized \emph{lazy update policy}, where only such updates are executed whose omission would (in general) compromise the soundness or completeness of the diagnosis search. 
Second, as regards the tree pruning, DynamicHS features an \emph{efficient two-tiered redundancy testing technique} for evaluating tree branches with regard to whether they are necessary (still needed) or redundant (ready to be discarded). The principle behind this technique is to first execute a more performant sound but incomplete test, 
and second, only if necessary, a sound and complete test. 
%
Moreover, preliminarily unneeded \emph{``duplicate'' tree branches are stored in a space-saving manner} and only used on demand to reconstruct a proper replacement branch if the ``original'' associated with the duplicate has become redundant.  
Finally, DynamicHS embraces various \emph{strategies to
avoid the overall amount of logical reasoning} 
involved in its computations,
since the calls of a logical inference engine normally account for most of the computation time of a diagnosis system (cf., e.g., \cite{pill2011eval_hitting_set_algos}).


\subsection{Description of DynamicHS}
\label{sec:description_of_DynHS}
In this section we describe the DynamicHS algorithm,
given by Alg.~\ref{algo:dynamic_hs}, in more detail. It inherits many of its aspects from Reiter's HS-Tree, depicted by Alg.~\ref{algo:hs}. 
Hence, we first recapitulate HS-Tree and then focus on the differences to and idiosyncrasies of DynamicHS.

\subsubsection{Reiter's HS-Tree}
\label{sec:Reiters_HS-Tree}
We briefly repeat the functioning of Reiter's HS-Tree (Alg.~\ref{algo:hs}), which computes minimal diagnoses for a DPI $\dpi=\tuple{\mo,\mb,\Tp\cup\Tp',\Tn\cup\Tn'}$ in a sound, complete\footnote{Unlike Reiter \cite{Reiter87}, we assume that only \emph{minimal} conflicts are used as node labels. Thus, the issue pointed out by \cite{greiner1989correction} does not arise.} and best-first way. 

Starting from a priority queue of unlabeled nodes $\Queue$, initially
comprising only an unlabeled root node, $\emptyset$, 
the algorithm continues to remove and label the first ranked node from $\Queue$ (\textsc{getAndDele-teFirst})
until all nodes are labeled ($\Queue = [\,]$) or $\ld$ minimal diagnoses have been computed.
The possible node labels are minimal conflicts (for internal tree nodes) and 
$\mathit{valid}$
as well as 
$\mathit{closed}$
(for leaf nodes). 
All minimal conflicts that have already been computed and used as node labels are stored in the (initially empty) set $\mC_{calc}$. 
Each edge in the constructed tree has a label. For ease of notation (in Alg.~\ref{algo:hs}), 
the set of edge labels along the branch from the root node of the tree to a node $\nd$ is associated with $\nd$, i.e., $\nd$ stores this set of labels. E.g., the node at location $\textcircled{\tiny 12}$ in iteration~1 of Fig.~\ref{fig:hs_example} on page~\pageref{fig:hs_example} is referred to as $\setof{1,2,5}$.
Once the tree has been completed ($\Queue = [\,]$), i.e., all nodes are labeled, the minimal diagnoses 
for $\dpi$ are given exactly by the set of all nodes labeled $\mathit{valid}$; formally: $\sol(\dpi) = \setof{\nd\mid \nd \text{ is labeled }\mathit{valid}}$.

\begin{algorithm}[!tb]
	\scriptsize
	\caption{HS-Tree} \label{algo:hs}
	\begin{algorithmic}[1]
		\Require 
		tuple $\tuple{\dpi, \Tp', \Tn',  \pr, \ld}$ comprising
		\begin{itemize}[noitemsep]
			\item a DPI $\dpi = \langle\mo,\mb,\Tp,\Tn\rangle$
			\item the acquired sets of positive ($\Tp'$) and negative ($\Tn'$) measurements so far 
			\item a function $\pr$ assigning a fault probability to each element in $\mo$
			\item the number $\ld$ of leading minimal diagnoses to be computed 
		\end{itemize}
		\Ensure 
		$\mD$, the set of the $\ld$ (if existent) most probable (as per $\pr$) minimal diagnoses wrt.\ $\langle\mo,\mb,\Tp\cup\Tp',\Tn\cup\Tn'\rangle$
		\vspace{10pt}
		\Procedure{hsTree}{$\dpi, \Tp', \Tn',  \pr, \ld$}
		\State $\mD_{calc}, \mC_{calc} \gets \emptyset$
		\State $\Queue \gets [\emptyset]$							
		\While{$\Queue \neq [\,] \land \left(\;|\mD_{calc}| < \ld\;\right)$}	
		\label{algoline:hs:while}
		\State $\mathsf{node} \gets \Call{getAndDeleteFirst}{\Queue}$	\label{algoline:hs:getfirst}
		\State $\tuple{L,\mathbf{C}} \gets \Call{label}{\dpi, \mathsf{node}, \mathbf{C}_{calc}, \mD_{calc}, \Queue}$\label{algoline:hs:label}
		\State $\mathbf{C}_{calc} \gets \mathbf{C}$    				\label{algoline:hs:update_Ccalc}
		\If{$L = \mathit{valid}$}\label{algoline:hs:L=valid}
		\State $\mD_{calc} \gets \mD_{calc} \cup \setof{\mathsf{node}}$		\label{algoline:hs:update_Dcalc}
		\ElsIf{$L = \mathit{closed}$}\label{algoline:hs:do_nothing}  \Comment{do nothing $\to\;$ $\mathsf{node}$ discarded}
		\Else 	\Comment{$L$ must be a minimal conflict set}
		\For{$e \in L$}						
		\State $\Queue \gets \Call{insertSorted}{ \mathsf{node} \cup \setof{e}, \Queue, \pr}$\label{algoline:hs:generate_nodes}
		\EndFor
		\EndIf
		\EndWhile
		\State \Return $\mD_{calc}$
		\EndProcedure
		\vspace{10pt}
		\Procedure{\textsc{label}}{$\langle\mo,\mb,\Tp,\Tn\rangle,\mathsf{node},\mathbf{C}_{calc},\mD_{calc}, \Queue$}    \label{algoline:hs:procedure_label}
		\For{$\mathsf{nd} \in \mD_{calc}$}		\Comment{(L1)}									\label{algoline:hs:non_min_crit_start}
		\If{$\mathsf{node} \supseteq \mathsf{nd}$}  \Comment{non-minimality}
		\State \Return $\tuple{\mathit{closed},\mathbf{C}_{calc}}$														\label{algoline:hs:non_min_crit_end}
		\EndIf
		\EndFor
		\For{$\mathsf{nd} \in \Queue$}	\Comment{(L2)}																									\label{algoline:hs:duplicate_crit_start}
		\If{$\mathsf{node} = \mathsf{nd}$}  \Comment{remove duplicates}
		\State \Return $\tuple{\mathit{closed},\mathbf{C}_{calc}}$														\label{algoline:hs:duplicate_crit_end}
		\EndIf
		\EndFor
		\For{$\mc \in \mathbf{C}_{calc}$}		\Comment{(L3)}									\label{algoline:hs:reuse_crit_start}
		\If{$\mc \cap \mathsf{node} = \emptyset$}\label{algoline:hs:test_cs_not_hit}  \Comment{reuse conflict $\mc$}
		\State \Return $\tuple{\mc,\mathbf{C}_{calc}}$																\label{algoline:hs:reuse_crit_end}
		\EndIf
		\EndFor
		\State $L\gets \Call{findMinConflict}{\langle\mo\setminus\mathsf{node},\mb,\Tp,\Tn\rangle_\RQ}$	\label{algoline:hs:qx}  \Comment{(L4)}
		\If{$L$ = \text{'no conflict'}}   \Comment{$\mathsf{node}$ is a diagnosis}
		\State \Return $\tuple{\mathit{valid},\mathbf{C}_{calc}}$														\label{algoline:hs:return_valid}
		\Else										\Comment{$L$ is \emph{new} minimal conflict set ($\notin \mathbf{C}_{calc}$)}
		\State $\mathbf{C}_{calc} \gets \mathbf{C}_{calc} \cup \setof{L}$							\label{algoline:hs:add_conflict}  
		\State \Return $\tuple{L,\mathbf{C}_{calc}}$\label{algoline:hs:return_computed_cs}
		\EndIf
		\EndProcedure
	\end{algorithmic}
	\normalsize
\end{algorithm}

To label a node $\nd$, the algorithm calls a labeling function (\textsc{label}) 
which executes the following tests in the given order and returns as soon as a label for $\nd$ has been determined:
\begin{enumerate}[noitemsep,leftmargin=*,label=(L\arabic*)]
	\item \label{enum:hstree:label:non-min} \emph{(non-minimality):} Check if $\nd$ is non-minimal (i.e., whether there is a node $\node$ with label 
	$\mathit{valid}$
	where $\nd \supseteq \node$). If so, $\nd$ is labeled by 
	$\mathit{closed}$.
	\item \label{enum:hstree:label:duplicate} \emph{(duplicate):} Check if $\nd$ is duplicate (i.e., whether $\nd = \node$ for some other $\node$ in $\Queue$). If so, $\nd$ is labeled by 
	$\mathit{closed}$.
	\item \label{enum:hstree:label:conflict_reuse} \emph{(reuse label):} Scan $\mC_{calc}$ for some $\mc$ such that $\nd \cap \mc = \emptyset$. If so, $\nd$ is labeled by $\mc$. 
	\item \label{enum:hstree:label:getMinConflict} \emph{(compute label):} Invoke \textsc{findMinConflict}, a \emph{(sound and complete) minimal conflicts searcher}, e.g., QuickXPlain \cite{junker04,rodler2020qx}, to get a minimal conflict for $\langle\mo\setminus\nd,\mb,\Tp\cup\Tp',\Tn\cup\Tn'\rangle$. If a minimal conflict $\mc$ is output, $\nd$ is labeled by $\mc$.
	Otherwise, if 'no conflict' is returned, then $\nd$ is labeled $\mathit{valid}$.
\end{enumerate}
All nodes labeled by 
$\mathit{closed}$
or 
$\mathit{valid}$
have no successors and are leaf nodes.
For each node $\nd$ labeled by a minimal conflict $L$, 
one outgoing edge is constructed for each element $e\in L$, where this edge is labeled by $e$ and pointing to a newly created unlabeled node $\nd \cup \setof{e}$. 
Each new node is added to $\Queue$ such that $\Queue$'s sorting is preserved (\textsc{insertSorted}).
$\Queue$ might be, e.g., \emph{(i)}~a FIFO queue, entailing that HS-Tree computes diagnoses in minimum-cardinality-first order (\emph{breadth-first search}), or \emph{(ii)}~sorted in descending order by $\pr$, where most probable diagnoses are generated first (\emph{uniform-cost search}; for details see \cite[Sec.~4.6]{Rodler2015phd}). 

Importantly, $\pr$ must in any case be specified in a way the probability for each system component (element of $\mo$) is below $0.5$.\footnote{This condition \cite{Rodler2015phd} is fairly benign as it can be established from any probability model $\pr$ 
by simply choosing an arbitrary fixed $c\in(0,0.5)$ and by setting $\pr_{\mathit{adj}}(\tax) := c \cdot\pr(\tax)$ for all $\tax \in \mo$. Observe that this adjustment does not affect the relative probabilities or the fault probability order of components in that $\pr_{\mathit{adj}}(\tax)/\pr_{\mathit{adj}}(\tax') = k$ whenever $\pr(\tax)/\pr(\tax') = k$.}
In other words, each system component must be more likely to be nominal than to be faulty. This additional assumption is 
necessary for the soundness of HS-Tree, i.e., to guarantee that \emph{only} minimal diagnoses are returned. The crucial point is that, for a queue $\Queue$ ordered by $\pr$, this requirement guarantees that node $\mathsf{n} \in \Queue$ will be processed prior to node $\mathsf{n}' \in \Queue$ whenever $\mathsf{n} \subset \mathsf{n}'$. For more information consider \cite[p.~73 et seq.]{Rodler2015phd}.

Finally, note the \emph{statelessness} of Reiter's HS-Tree, reflected by $\Queue$ initially including \emph{only an unlabeled root node}, and $\mC_{calc}$ being initially \emph{empty}. That is, an HS-Tree is built from scratch in each iteration, every time for different measurement sets $\Tp',\Tn'$.

\subsubsection{Dynamicity of DPI in Sequential Diagnosis}
In the course of the sequential diagnosis process (Alg.~\ref{algo:sequential_diagnosis}), where additional system knowledge is gathered in terms of measurements, 
the DPI is subject to gradual change---it is dynamic. 
At this, each addition of a new discriminating measurement\footnote{A \emph{discriminating measurement} is a measurement at a discriminating measurement point (cf. Sec.~\ref{sec:basics}). That is, adding such a measurement to (the positive or negative measurements of) the DPI causes the invalidation of at least one diagnosis.} to the DPI 
also effectuates a transition of the sets of (minimal) diagnoses and (minimal) conflicts. Whereas this fact is of no concern to a stateless diagnosis computation strategy, it has to be carefully taken into account when engineering a stateful approach. 

\subsubsection{Towards Stateful Hitting Set Computation}
To understand the necessary design decisions to devise a sound and complete stateful hitting set algorithm in the light of the said DPI dynamicity, we next discuss a few more specifics of the conflicts and diagnoses evolution throughout
sequential diagnosis:\footnote{A more formal treatment and proofs for all statements can be found in \cite[Sec.\ 12.4]{Rodler2015phd}.}
\begin{property}\label{property:impact_of_DPI_transition}
Let $\dpi_j = \tuple{\mo,\mb,\Tp,\Tn}$ be a DPI and let $T$ be 
Reiter's HS-Tree for $\dpi_j$ (executed until) producing the 
leading diagnoses $\mD$ where $|\mD|\geq 2$. 
Further, let the measurement $m$ describe the outcome of measuring at a discriminating measurement point computed from $\mD$,
and let $\dpi_{j+1}$ be the DPI resulting from $\dpi_j$ through the addition of $m$ to either $\Tp$ or $\Tn$. Then:
\begin{enumerate}[leftmargin=*]
	\item \label{property:impact_of_DPI_transition:enum:T_must_be_updated} $T$ is not a correct HS-Tree for $\dpi_{j+1}$, i.e., (at least) some node labeled by $\mathit{valid}$ in $T$ is incorrectly labeled.\vspace{2pt} \\
	\emph{That is, to reuse $T$ for $\dpi_{j+1}$, $T$ must be appropriately updated.}
	\item \label{property:impact_of_DPI_transition:enum:diags_can_only_grow} Each $\md \in \sol(\dpi_{j+1})$ is either equal to or a superset of some $\md' \in \sol(\dpi_{j})$.\vspace{2pt} \\ 
	\emph{That is, minimal diagnoses can grow or remain unchanged, but cannot shrink. Consequently, to reuse $T$ for sound and complete minimal diagnosis computation for $\dpi_{j+1}$, existing nodes must never be reduced---either a node is left as is, deleted as a whole, or (prepared to be) extended.}
	\item \label{property:impact_of_DPI_transition:enum:conflicts_can_only_shrink} For all $\mc \in \conf(\dpi_{j})$ there is a $\mc' \in \conf(\dpi_{j+1})$ such that $\mc' \subseteq \mc$.\vspace{2pt} \\
	\emph{That is, existing minimal conflicts can only shrink or remain unaffected, but cannot grow. Hence, priorly computed minimal conflicts (for an old DPI) might not be minimal for the current DPI. In other words, conflict node labels of $T$ can, but do not need to, be correct for $\dpi_{j+1}$.}
	\item \label{property:impact_of_DPI_transition:enum:some_conflict_shrinks_or_new_conflict_arises} (a) There is some $\mc \in \conf(\dpi_{j})$ for which there is some $\mc' \in \conf(\dpi_{j+1})$ with $\mc' \subset \mc$, and/or \\(b) there is some $\mc_{\mathit{new}} \in \conf(\dpi_{j+1})$ where $\mc_{\mathit{new}} \not\subseteq \mc''$ and $\mc_{\mathit{new}} \not\supseteq \mc''$ for all $\mc'' \in \conf(\dpi_{j})$.\vspace{2pt} \\
	\emph{That is, at least one minimal conflict is reduced in size, and/or at least one entirely new minimal conflict (which is not in any subset-relationship with existing ones) arises. Some existing node in $T$ which represents a minimal diagnosis for $\dpi_{j}$ \emph{(a)} can be deleted since it would not be present when using $\mc'$ as node label in $T$ wherever $\mc$ is used, or \emph{(b)} must be extended to constitute a diagnosis for $\dpi_{j+1}$, since it does not hit (has an empty intersection with) the conflict $\mc_{\mathit{new}}$.}
\end{enumerate}
\end{property}
\subsubsection{Major Modifications to Reiter's HS-Tree}
Based on Property~\ref{property:impact_of_DPI_transition}, the following principal amendments to Reiter's HS-Tree are necessary to make it a properly-working stateful diagnosis computation method. We exemplify these modifications by referring to Figs.~\ref{fig:dynhs_example} and \ref{fig:hs_example} (on pages~\pageref{fig:dynhs_example} and \pageref{fig:hs_example}), which are discussed in detail in Example~\ref{ex:algo_description} and which describe sequential diagnosis executions (as per Alg.~\ref{algo:sequential_diagnosis}) of DynamicHS and HS-Tree, respectively, for our example DPI given in Tab.~\ref{tab:example_DPI}. 
\begin{enumerate}[label=\textbf{(Mod\arabic*)}, wide, labelwidth=!, labelindent=0pt]
	\item \label{enum:mod:tree_update} A tree update is executed at the beginning of each DynamicHS execution, where the hitting set tree produced for a previously relevant DPI is adapted to a tree that allows to compute minimal diagnoses for the current DPI in a sound, complete and best-first manner.\vspace{3pt}\\
	\emph{Justification:} The necessity of this update is attested by Property~\ref{property:impact_of_DPI_transition}.\ref{property:impact_of_DPI_transition:enum:T_must_be_updated}. The actions taken in the course of the update are motivated by the remaining bullet points of Property~\ref{property:impact_of_DPI_transition}.\vspace{3pt}\\
	\emph{Example:} Between iterations~1 and 2 in Fig.~\ref{fig:dynhs_example}, DynamicHS runs a tree update, which adapts the tree constructed in iteration~1 for $\dpi_0$ to one that can be reused to tackle $\dpi_1$ in iteration~2. This update effectively involves the removal of the redundant nodes (minimal diagnoses for $\dpi_0$, indicated by a $\checkmark$) numbered $\textcircled{\scriptsize 6}$ and $\textcircled{\scriptsize 8}$, the reduction of the conflicts numbered $\textcircled{\scriptsize 2}$ and $\textcircled{\scriptsize 3}$ to their proper subsets $\tuple{2,4}$ and $\tuple{1,5}$ which are now (for $\dpi_1$) minimal conflicts, and the reinsertion of node $\textcircled{\tiny 10}$ into the queue $\Queue$ (because, after the deletion of the two mentioned redundant nodes, there is no longer a ``witness'' diagnosis testifying the non-minimality of node $\textcircled{\tiny 10}$).

	\item \label{enum:mod:nonmin_diags_and_duplicates_stored} Non-minimal diagnoses (test \ref{enum:hstree:label:non-min} in HS-Tree) and duplicate nodes (test \ref{enum:hstree:label:duplicate}) are stored in collections $\mD_{\supset}$ and $\Queue_{dup}$, respectively, instead of being closed and discarded.\vspace{3pt}\\ 
	\emph{Justification:} Property~\ref{property:impact_of_DPI_transition}.\ref{property:impact_of_DPI_transition:enum:diags_can_only_grow} suggests to store non-minimal diagnoses, as they might constitute (sub-branches of) minimal diagnoses in the next iteration. 
	Property~\ref{property:impact_of_DPI_transition}.\ref{property:impact_of_DPI_transition:enum:some_conflict_shrinks_or_new_conflict_arises}(a) suggests to record all duplicates for completeness of the diagnosis search. Because, some active node $\mathsf{nd}$ representing this duplicate in the current tree could become obsolete due to the shrinkage of some conflict, and the duplicate might be non-obsolete and eligible to turn active and replace $\mathsf{nd}$ in the tree.\vspace{3pt}\\
	\emph{Example:} 
	Node $\textcircled{\scriptsize 4}$ in iteration~1 of Fig.~\ref{fig:dynhs_example} 
	is added to $\Queue_{dup}$ (DynamicHS), whereas the corresponding node 
	$\textcircled{\scriptsize 7}$
	in iteration~1 of Fig.~\ref{fig:hs_example} 
	is not further considered (HS-Tree). Similarly, node $\textcircled{\tiny 10}$ is stored in $\mD_{\supset}$ in case of DynamicHS (Fig.~\ref{fig:dynhs_example}), whereas the same node is simply closed and eliminated in HS-Tree (Fig.~\ref{fig:hs_example}).

	\item \label{enum:mod:nodes_are_lists} Each node $\mathsf{nd}$ is no longer identified as the \emph{set} of edge labels from the root to $\mathsf{nd}$, but as an \emph{ordered list} of these edge labels. In addition, an ordered list of the conflicts used to label internal nodes along the branch from the root to $\mathsf{nd}$ is stored in terms of $\mathsf{nd.cs}$.\vspace{3pt}\\
	\emph{Justification:} This modification is required due to Property~\ref{property:impact_of_DPI_transition}.\ref{property:impact_of_DPI_transition:enum:some_conflict_shrinks_or_new_conflict_arises}. The necessity of storing both the edge labels and the internal node labels as lists 
	has its reason in the replacement of obsolete tree branches by stored duplicates. In fact, any duplicate used to replace a node must correspond to the same \emph{set} of edge labels as the replaced node. However, in the branch of the obsolete node, some node-labeling conflict has been reduced to make the node redundant, whereas for a suitable duplicate replacing the node, no redundancy-causing changes to conflicts along its branch have occurred. By storing only sets of edge labels, we could not differentiate between the redundant and the non-redundant nodes.\vspace{3pt}\\
	\emph{Example:} For the node $\mathsf{nd}$ at location $\textcircled{\scriptsize 9}$ in iteration~1 of Fig.~\ref{fig:dynhs_example}, we have $\mathsf{nd} = [2,5]$ and $\mathsf{nd.cs} = [\tuple{1,2},\tuple{1,3,5}]$. 
	%
	%
	\item \label{enum:mod:conflict_minimality_check} Before reusing a conflict $\mc$ to label a node (labeling test \ref{enum:hstree:label:conflict_reuse} in HS-Tree), a minimality check for $\mc$ is performed. If a proper subset $X$ of $\mc$ is identified as a conflict for the current DPI in the course of this check, $X$ is used to prune obsolete tree branches, to replace node-labeling conflicts that are supersets of $X$ by $X$, and to update $\mC_{calc}$ in that $X$ is added and all of its supersets are deleted.\vspace{3pt}\\
	\emph{Justification:} By Property~\ref{property:impact_of_DPI_transition}.\ref{property:impact_of_DPI_transition:enum:conflicts_can_only_shrink}, stored conflicts in $\mC_{calc}$ and those appearing as labels in the existing tree (elements of the lists $\mathsf{nd.cs}$ for the different nodes $\mathsf{nd}$) might not be minimal for the current DPI (since they might have been computed for a prior DPI). This minimality check helps both to prune the tree (reduction of the number of nodes) and to make sure that any extension to the tree uses only minimal conflicts wrt.\ the current DPI as internal node labels (avoidance of the introduction of unnecessary new edges).\vspace{3pt}\\ 
	\emph{Example:} Let us assume that a new node should be labeled in iteration~2 of DynamicHS (Fig.~\ref{fig:dynhs_example}) and $\mc:=\tuple{3,4,5}$ is a conflict that satisfies the reuse-test \ref{enum:hstree:label:conflict_reuse}. Then the performed minimality check would return $X:=\tuple{4,5}$ 
	as a subset of $\mc$ (since $\mc$ is a minimal conflict for $\dpi_0$, but a non-minimal one for $\dpi_1$, cf.\ Tab.~\ref{tab:diag_conf_evolution_example}).
	
	\item \label{enum:mod:storing_the_state} The current state of DynamicHS (in terms of the so-far produced hitting set tree) is stored over all its iterations executed throughout the sequential diagnosis process (Alg.~\ref{algo:sequential_diagnosis}) by means of the tuple $\mathsf{state}$.\vspace{3pt}\\
	\emph{Justification:} Statefulness of DynamicHS.\vspace{3pt}\\
	\emph{Example:} Consider iteration~2 for both HS-Tree (Fig.~\ref{fig:hs_example}) and DynamicHS (Fig.~\ref{fig:dynhs_example}). In the latter case, for the new DPI $\dpi_1$, the tree built in iteration~1 is reused, whereas in the former a new one is built. One implication of this is, e.g., that DynamicHS does not need to (fully) recompute the conflicts $\tuple{1,2}$, $\tuple{2,4}$ and $\tuple{1,5}$.

\end{enumerate}
\begin{algorithm}[!htbp]
	\scriptsize
	\caption{DynamicHS} \label{algo:dynamic_hs}
	\begin{algorithmic}[1]
		\Require 
		\textcolor{white}{.}
		tuple $\tuple{\dpi, \Tp', \Tn',  \pr, \ld, \mD_{\checkmark}, \mD_{\times},  \state}$ comprising
		\begin{itemize}[noitemsep]
			\item a DPI $\dpi = \langle\mo,\mb,\Tp,\Tn\rangle$
			\item the acquired sets of positive ($\Tp'$) and negative ($\Tn'$) measurements so far 
			\item a function $\pr$ assigning a fault probability to each element in $\mo$
			\item the number $\ld$ of leading minimal diagnoses to be computed 
			\item the set $\mD_{\checkmark}$ of all elements of the set $\mD_{calc}$ (returned by the previous \textsc{dynamicHS} run) which are minimal diagnoses wrt.\ $\langle\mo,\mb,\Tp\cup\Tp',\Tn\cup\Tn'\rangle$
			\item the set $\mD_{\times}$ of all elements of the set $\mD_{calc}$ (returned by the previous \textsc{dynamicHS} run) which are no diagnoses wrt.\ $\langle\mo,\mb,\Tp\cup\Tp',\Tn\cup\Tn'\rangle$
			\item $\state = \tuple{\Queue, \Queue_{dup}, \mD_{\supset}, \mC_{\mathit{calc}}}$ where
			\begin{itemize}[noitemsep]
				\item $\Queue$ is the current queue of unlabeled nodes,
				\item $\Queue_{dup}$ is the current queue of duplicate nodes,
				\item $\mD_{\supset}$ is the current set of computed non-minimal diagnoses,
				\item $\mC_{calc}$ is the current set of computed minimal conflict sets.
			\end{itemize}
		\end{itemize}

		\Ensure 
		tuple $\tuple{\mD,\state}$ where
		\begin{itemize}[noitemsep]
			\item $\mD$ is the set of the $\ld$ (if existent) most probable (as per $\pr$) minimal diagnoses wrt.\
			$\langle\mo,\mb,\Tp\cup\Tp',\Tn\cup\Tn'\rangle$
			\item $\state$ is as described above
		\end{itemize}  
		
		\vspace{3pt}
		\Procedure{dynamicHS}{$\dpi, \Tp', \Tn',  \pr, \ld, \mD_{\checkmark}, \mD_{\times}, \state $}
		\State $\mD_{calc} \gets \emptyset$\label{algoline:dyn:Dcalc_gets_emptyset}
		\State $\state \gets \Call{updateTree}{\dpi, \Tp', \Tn',  \pr, \mD_{\checkmark}, \mD_{\times}, \state}$\label{algoline:dyn:update_tree}     
		\While{$\Queue \neq [\,] \land \left(\;|\mD_{calc}| < \ld\;\right)$}	
		\label{algoline:dyn:while}																								\State $\mathsf{node} \gets \Call{getAndDeleteFirst}{\Queue}$\label{algoline:dyn:get_first}			  \Comment{$\mathsf{node}$ is processed}		
		\If{$\mathsf{node} \in \mD_{\checkmark}$}\label{algoline:dyn:node_in_Dcheckmark}  \Comment{$\mD_{\checkmark}$ includes only min...}
		\State $L \gets \mathit{valid}$\label{algoline:dyn:set_L_to_valid}      \Comment{...diags wrt.\ current DPI}         
		\Else
		\State $\tuple{L,\state} \gets \Call{dLabel}{\dpi, \Tp', \Tn', \pr, \mathsf{node}, \mD_{calc}, \state}$\label{algoline:dyn:dlabel}
		\EndIf
		\If{$L = \mathit{valid}$}\label{algoline:dyn:if_L_valid}  
		\State $\mD_{calc} \gets \mD_{calc} \cup \setof{\mathsf{node}}$\label{algoline:dyn:add_to_Dcalc}      \Comment{$\mathsf{node}$ is a min diag wrt.\ current DPI}
		\ElsIf{$L = \mathit{nonmin}$}								
		\State $\mD_{\supset} \gets \mD_{\supset} \cup \setof{\mathsf{node}}$ \label{algoline:dyn:add_to_Dsupset}  \Comment{$\mathsf{node}$ is a non-min diag wrt.\ current DPI}
		\Else 	
		\For{$e \in L$}\label{algoline:dyn:for_e_in_L}            \Comment{$L$ is a min conflict  wrt.\ current DPI}
		\State $\mathsf{node}_e \gets \Call{append}{\mathsf{node},e}$ \label{algoline:dyn:add_ax_to_node}       \Comment{$\mathsf{node}_e$ is generated}   
		\State $\mathsf{node}_{e}.\mathsf{cs} \gets \Call{append}{\mathsf{node.cs},L}$ \label{algoline:dyn:add_cs_to_node.cs}
		\If{$\mathsf{node}_e \in \Queue \lor \mathsf{node}_e \in \mD_{\supset}$}   \label{algoline:dyn:check_node_already_in_Q}                      \Comment{$\mathsf{node}_e$ is (\emph{set-equal}) duplicate of some node in $\Queue$ or $\mD_{\supset}$}
		\State $\Queue_{dup} \gets \Call{insertSorted}{ \mathsf{node}_e, \Queue_{dup}, \mathit{card}, <}$ 
		\label{algoline:dyn:add_to_Qdup}
		\Else
		\State $\Queue \gets \Call{insertSorted}{ \mathsf{node}_e, \Queue, \pr, >}$\label{algoline:dyn:add_to_Queue} 
		\EndIf
		\EndFor
		\EndIf
		\EndWhile
		\State \Return $\tuple{\mD_{calc}, \state}$ \label{algoline:dyn:return}
		\EndProcedure
		\vspace{6pt}
		\Procedure{\textsc{dLabel}}{$\langle\mo,\mb,\Tp,\Tn\rangle,\Tp', \Tn', \pr, \mathsf{node}, \mD_{calc}, \state$} 
		\For{$\mathsf{nd} \in \mD_{calc}$}\label{algoline:dlabel:non-min_crit_start}
		\If{$\mathsf{node} \supset \mathsf{nd}$}    \Comment{$\mathsf{node}$ is a non-min diag}
		\State \Return $\tuple{\mathit{nonmin},\state}$
		\EndIf
		\EndFor\label{algoline:dlabel:non-min_crit_end}
		\For{$\mc \in \mathbf{C}_{calc}$}\label{algoline:dlabel:reuse_start} \Comment{$\mC_{calc}$ includes conflicts wrt.\ current DPI}
		\If{$\mc \cap \mathsf{node} = \emptyset$}\label{algoline:dlabel:if_C_cap_node=emptyset}    \Comment{reuse (a subset of) $\mc$ to label $\mathsf{node}$} 
		\State $X \gets \Call{findMinConflict}{\langle\mc,\mb,\Tp\cup\Tp',\Tn\cup\Tn'\rangle}$\label{algoline:dlabel:qx_1} 
		\If{$X = \mc$} \label{algoline:dlabel:if_X=C}
		\State \Return $\tuple{\mc,\state}$\label{algoline:dlabel:return_C}
		\Else      \Comment{$X \subset \mc$} \label{algoline:dlabel:else}
		\State $\tuple{\state,\mD_{calc}} \gets \Call{prune}{X,\{\state,\mD_{calc}\}}$ \label{algoline:dlabel:prune}
		\State \Return $\tuple{X,\state}$ \label{algoline:dlabel:return_X}
		\EndIf
		\EndIf
		\EndFor\label{algoline:dlabel:reuse_end}
		\State $L\gets \Call{findMinConflict}{\langle\mo\setminus\mathsf{node},\mb,\Tp\cup\Tp',\Tn\cup\Tn'\rangle}$\label{algoline:dlabel:qx_2} 
		\If{$L$ = \text{'no conflict'}}						\Comment{$\mathsf{node}$ is a diag}
		\State \Return $\tuple{\mathit{valid},\state}$\label{algoline:dlabel:return_valid}
		\Else						\Comment{$L$ is a \emph{new} min conflict ($\notin \mathbf{C}_{calc}$)}
		\State $\mathbf{C}_{calc} \gets \mathbf{C}_{calc} \cup \setof{L}$\label{algoline:dlabel:add_new_cs}
		\State \Return $\tuple{L,\state}$\label{algoline:dlabel:return_new_cs}
		\EndIf
		\EndProcedure
		\algstore{myalg}
	\end{algorithmic}
	\end{algorithm}

		\begin{algorithm}[!htbp]
		\scriptsize
		\caption*{\textbf{Algorithm \ref{algo:dynamic_hs}} DynamicHS (continued)}
		\begin{algorithmic}[1]
		\algrestore{myalg}
		\Procedure{\textsc{updateTree}}{$\dpi, \Tp', \Tn',  \pr, \mD_{\checkmark}, \mD_{\times}, \state$}
		\For{$\mathsf{nd} \in \mD_{\times}$} \label{algoline:update:process_Dtimes_start}
		\Comment{search for redundant nodes among invalidated diags}
		\If{$\Call{redundant}{\mathsf{nd},\dpi}$}  \label{algoline:update:call_redundant_function}
		\State $\tuple{\state,\mD_{\checkmark},\mD_{\times}} \gets \Call{prune}{X,\{\state,\mD_{\checkmark},\mD_{\times}\}}$ \label{algoline:update:prune}
		\EndIf
		\EndFor
		\For{$\mathsf{nd} \in \mD_{\times}$}\label{algoline:update:reinsert_D_of_Dx_to_Q}
		\Comment{add all (non-pruned) nodes in $\mD_{\times}$ to $\Queue$}
		\State $\Queue \gets \Call{insertSorted}{ \mathsf{nd}, \Queue, \pr, >}$\label{algoline:update:insert_sorted_0}
		\State $\mD_{\times} \gets \mD_{\times} \setminus \setof{\mathsf{nd}}$ \label{algoline:update:delete_from_Dtimes}
		\EndFor \label{algoline:update:process_Dtimes_end}
		\For{$\mathsf{nd} \in \mD_{\supset}$}\label{algoline:update:process_Dsupset_start} \Comment{add all (non-pruned) nodes in $\mD_{\supset}$ to $\Queue$, which...}
		\State $\mathit{nonmin} \gets \false$ \Comment{...are no longer supersets of any diag in $\mD_{\checkmark}$}
		\For{$\mathsf{nd}' \in \mD_{\checkmark}$}
		\If{$\mathsf{nd} \supset \mathsf{nd}'$}    
		\State $\mathit{nonmin} \gets \true$
		\State \textbf{break} 
		\EndIf
		\EndFor
		\If{$\mathit{nonmin} = \false$}
		\State $\Queue \gets \Call{insertSorted}{ \mathsf{nd}, \Queue, \pr, >}$\label{algoline:update:insert_sorted_0.5}
		\State $\mD_{\supset} \gets \mD_{\supset} \setminus \setof{\mathsf{nd}}$\label{algoline:update:delete_from_Dsupset}
		\EndIf
		\EndFor \label{algoline:update:process_Dsupset_end}
		\For{$\md \in \mD_{\checkmark}$}\label{algoline:update:process_Dcheckmark_start}  \Comment{add known min diags in $\mD_{\checkmark}$ to $\Queue$ to find diags...}
		\State $\Queue \gets \Call{insertSorted}{ \md, \Queue, \pr, >}$\label{algoline:update:insert_sorted_1}
		\Comment{...in best-first order (as per $\pr$)}
		\EndFor \label{algoline:update:process_Dcheckmark_end}
		\State \Return $\state$
		\EndProcedure	
	\end{algorithmic}
	\normalsize
\end{algorithm}

\subsubsection{DynamicHS: Algorithm Walkthrough}
\label{sec:algo_walkthrough}
We now explicate the workings of DynamicHS, depicted by Alg.~\ref{algo:dynamic_hs}.

Like HS-Tree, DynamicHS is basically processing a priority queue $\Queue$ of nodes (while-loop; lines~\ref{algoline:dyn:while} and \ref{algoline:dyn:get_first}). That is, in each iteration of the while-loop, the top-ranked node $\mathsf{node}$ is removed from $\Queue$ to be labeled (\textsc{getAndDeleteFirst}).    
Before calling the labeling function (\textsc{dLabel}), however, the algorithm checks if $\mathsf{node}$ is among the already computed minimal diagnoses $\mD_{\checkmark}$ from the previous iteration which are consistent with all measurements in $\Tp' \cup \Tn'$ (line~\ref{algoline:dyn:node_in_Dcheckmark}). If so, the node is directly labeled $\mathit{valid}$ (line~\ref{algoline:dyn:set_L_to_valid}). Otherwise the \textsc{dLabel} function is invoked to compute a label for $\mathsf{node}$ (line~\ref{algoline:dyn:dlabel}).\vspace{5pt}

\textsc{dLabel}: First, the non-minimality check is performed, just as done in HS-Tree, see \ref{enum:hstree:label:non-min} above (lines~\ref{algoline:dlabel:non-min_crit_start}--\ref{algoline:dlabel:non-min_crit_end}).
If the check is negative, a conflict-reuse check is carried out next (lines~\ref{algoline:dlabel:reuse_start}--\ref{algoline:dlabel:reuse_end}). Note that the duplicate check \ref{enum:hstree:label:duplicate}, which is done at this stage in HS-Tree, is obsolete since no duplicate nodes 
can ever be elements of $\Queue$. This is evaded by directly making the duplicate check after new nodes have been constructed and are to be added to $\Queue$ (see lines \ref{algoline:dyn:check_node_already_in_Q}
 and \ref{algoline:dyn:add_to_Qdup}).
The conflict-reuse check starts equally as in HS-Tree. However, if a suitable conflict $\mc$ for reuse is found in the set of so-far computed conflicts $\mC_{calc}$ (line~\ref{algoline:dlabel:if_C_cap_node=emptyset}), then \textsc{findMinConflict} (cf.\ \ref{enum:hstree:label:getMinConflict} above) is employed to check the minimality of the conflict wrt.\ the \emph{current} DPI (line~\ref{algoline:dlabel:qx_1}). If a proper subset $X$ of $\mc$ is found which is a conflict wrt.\ the current DPI (line~\ref{algoline:dlabel:else}), 
then $X$ is used to prune the current hitting set tree (line~\ref{algoline:dlabel:prune}) using the \textsc{prune} function (see below). Finally, depending on the minimality check, either the conflict $\mc$ from $\mC_{calc}$, if minimal, or the computed subset $X$ of $\mc$ is used to label $\mathsf{node}$ (lines~\ref{algoline:dlabel:return_C} and \ref{algoline:dlabel:return_X}).
In case no conflict is eligible for being reused, \textsc{findMinConflict} is called (line~\ref{algoline:dlabel:qx_2}) to test if a minimal conflict (wrt.\ the current DPI) to label $\mathsf{node}$ exists at all, or if $\mathsf{node}$ already corresponds to a (minimal) diagnosis wrt.\ the current DPI (again, this part, i.e., lines~\ref{algoline:dlabel:qx_2}--\ref{algoline:dlabel:return_new_cs}, is equal to HS-Tree, cf.\ \ref{enum:hstree:label:getMinConflict} above). As a last point, note that \textsc{dLabel} gets and returns the tuple $\mathsf{state}$ as an argument. This tuple includes collections of nodes storing the current state of the hitting set tree maintained by DynamicHS. The reason for passing $\mathsf{state}$ to \textsc{dLabel} is that the pruning actions potentially performed in the course of the reuse-check might modify $\mathsf{state}$ (while all other actions executed during \textsc{dLabel} can never alter $\mathsf{state}$).\vspace{5pt}

The output of \textsc{dLabel} is then processed by DynamicHS (lines~\ref{algoline:dyn:if_L_valid}--\ref{algoline:dyn:return}) in order to assign $\mathsf{node}$ to its appropriate node collection (note that the labels of processed nodes are implicitly stored by assigning the nodes to corresponding collections). More specifically,
$\mathsf{node}$ is assigned to $\mD_{calc}$ if the returned label is $\mathit{valid}$ (line~\ref{algoline:dyn:add_to_Dcalc}), and to $\mD_{\supset}$ if the label is $\mathit{nonmin}$ (line~\ref{algoline:dyn:add_to_Dsupset}). If, on the other hand, the label is a minimal conflict $L$, then a child node $\mathsf{node}_e$ is created for each element $e \in L$ and assigned to either $\Queue_{dup}$ (line~\ref{algoline:dyn:add_to_Qdup}) if there is a node in $\Queue$ or in $\mD_{\supset}$ that is \emph{set}-equal to $\mathsf{node}_e$, or to $\Queue$ otherwise (line~\ref{algoline:dyn:add_to_Queue}). At this, $\mathsf{node}_e$ is constructed from $\mathsf{node}$ via the \textsc{append} function (lines \ref{algoline:dyn:add_ax_to_node} and \ref{algoline:dyn:add_cs_to_node.cs}), which appends the element $e$ to the list $\mathsf{node}$, and the conflict $L$ to the list $\mathsf{node.cs}$ (cf.\ \ref{enum:mod:nodes_are_lists} above). 
Note that both $\Queue$ and $\Queue_{dup}$ are priority queues. $\Queue$ is ordered by descending node probability as per the given measure $\pr$, and $\Queue_{dup}$ by ascending node cardinality, i.e., nodes with fewer edge labels along their branch from the root node are prioritized. The reason for this will become fully evident when we discuss the \textsc{prune} function (see below).

The described processing of nodes from $\Queue$ is successively continued until one of the stop criteria applies (while-loop, line~\ref{algoline:dyn:while}). That is, the algorithm returns the set of computed minimal diagnoses $\mD_{calc}$ along with the current tree state $\mathsf{state}$ when either all nodes have been processed ($\Queue = [\,]$), i.e., the tree has been built to its entirety, or when the stipulated number $\ld$ of leading diagnoses have been found ($|\mD_{calc}| = \ld$).\vspace{5pt}

\textsc{updateTree}: The idea here is to adapt the current tree in a way it constitutes a basis for finding all and only minimal diagnoses in highest-probability-first order for the \emph{current} DPI. 
As discussed above, the minimal diagnoses can only remain the same or be augmented in size (Property~\ref{property:impact_of_DPI_transition}.\ref{property:impact_of_DPI_transition:enum:diags_can_only_grow}), whereas for existing conflicts it is the case that they can only remain the same or shrink in size (Property~\ref{property:impact_of_DPI_transition}.\ref{property:impact_of_DPI_transition:enum:conflicts_can_only_shrink}), or new conflicts can arise (Property~\ref{property:impact_of_DPI_transition}.\ref{property:impact_of_DPI_transition:enum:some_conflict_shrinks_or_new_conflict_arises}). 
The principle followed by \textsc{updateTree} is to search for non-minimal conflicts to be updated, and tree branches to be pruned, among the nodes that corresponded to minimal diagnoses for the previous DPI, but have been invalidated by the latest added measurement (the elements of $\mD_{\times}$, cf.\ line~\ref{algoline:inter_onto_debug:partition_D_into_Dcheckmark_and_Dtimes} in Alg.~\ref{algo:sequential_diagnosis}).

\label{def:redundancy} Regarding the pruning of tree branches, we call a node $\mathsf{nd}$ \emph{redundant} (wrt.\ a DPI $\dpi$) iff there is some index $j$ and a minimal conflict $X$ wrt.\ $\dpi$ such that the conflict $\mathsf{nd.cs}[j] \supset X$ and the element $\mathsf{nd}[j] \in \mathsf{nd.cs}[j] \setminus X$.\footnote{
By $\mathsf{nd}[j]$ and $\mathsf{nd.cs}[j]$, respectively, we refer to the $j$-th element of the ordered lists $\mathsf{nd}$ and $\mathsf{nd.cs}$ (cf.\ \ref{enum:mod:nodes_are_lists} above), i.e., to the $j$-th edge label and to the $j$-th conflict along node $\mathsf{nd}$'s branch starting from the root.} Moreover, we call $X$ a \emph{witness of redundancy for $\mathsf{nd}$} (wrt.\ $\dpi$).
In simple words, $\mathsf{nd}$ is redundant iff the branch from the root node to $\mathsf{nd}$ would not exist given that the (current) minimal conflict $X$ had been used instead of the (old, formerly minimal, but by now) non-minimal conflict $\mathsf{nd.cs}[j]$.

If such a redundant node is detected among the elements of $\mD_{\times}$ (function \textsc{redundant}, drawing on call(s) to \textsc{findMinConflict}; explained in detail later), then the \textsc{prune} function (see later) is called given the witness of redundancy of the redundant node as an argument (lines \ref{algoline:update:process_Dtimes_start}--\ref{algoline:update:prune}). After each node in $\mD_{\times}$ has been processed,\footnote{Note that $\mD_{\times}$ itself might change during the execution of the tree pruning (line~\ref{algoline:update:prune}) in that redundant nodes might be deleted from it and corresponding replacement nodes (constructed) from the stored duplicates might be added to it. Thus, the semantics of the for-loop in line~\ref{algoline:update:process_Dtimes_start} is that it stops iff each node in $\mD_{\times}$ \emph{at the present time} has been processed.} the remaining nodes in $\mD_{\times}$ (those that are non-redundant and thus have not been pruned) are re-added to $\Queue$ in prioritized order according to $\pr$ (lines \ref{algoline:update:reinsert_D_of_Dx_to_Q}--\ref{algoline:update:delete_from_Dtimes}), as accomplished by the function \textsc{insertSorted}. 
Likewise, all non-pruned nodes in $\mD_{\supset}$ (note that pruning always considers all collections $\Queue_{dup}$, $\Queue$, $\mD_{\checkmark}$, $\mD_{\times}$ and $\mD_{\supset}$) which are no longer supersets of any \emph{known} minimal diagnosis, are added to $\Queue$ again (lines \ref{algoline:update:process_Dsupset_start}--\ref{algoline:update:delete_from_Dsupset}).
Finally, the known minimal diagnoses that have been returned by the previous execution of DynamicHS and are consistent with the latest added measurement (the elements of $\mD_{\checkmark}$), are put back to the ordered queue $\Queue$. The justification for this step is given by the fact that, for the current DPI, there might be ``new'' minimal diagnoses that are more probable than the ones known from the previous iteration. Omitting this step therefore would (generally) compromise the best-first property of the hitting set computation.\vspace{5pt} 

\textsc{prune}: Using its first argument $X$, the function runs through the node collections $\Queue_{dup}$, $\Queue$, $\mD_{\supset}$ 
and $\mD_{calc}$ 
when called in line~\ref{algoline:dlabel:prune}, and through $\Queue_{dup}$, $\Queue$, $\mD_{\supset}$, $\mD_{\times}$ and  $\mD_{\checkmark}$ when called in line~\ref{algoline:update:prune} (cf.\ second argument), and 
\begin{itemize}[noitemsep,leftmargin=*]
	\item \emph{(relabeling of old conflicts)} replaces all labels $\mathsf{nd.cs}[i]$ which are proper supersets of $X$ by $X$ for all nodes $\mathsf{nd}$ and for all $i = 1,\dots,|\mathsf{nd}|$, and
	\item \emph{(deletion of redundant nodes)} deletes each redundant node $\mathsf{nd}$ for which $X$ is a witness of redundancy, and 
	\item \emph{(potential replacement of deleted nodes)} for each of the deleted nodes $\mathsf{nd}$, if available, uses a suitable (non-redundant) node $\mathsf{nd}'$ (constructed) from the elements of $\Queue_{dup}$ to replace $\mathsf{nd}$ by $\mathsf{nd}'$.
\end{itemize}
A node $\mathsf{nd}'$ qualifies as a \emph{replacement node for $\mathsf{nd}$} iff $\mathsf{nd}'$ is non-redundant and has the same \emph{set} of edge labels along its branch from the root node as $\mathsf{nd}$, 
i.e.,  
iff $\mathsf{nd}$ is set-equal (not equal in terms of the list of edge labels) to $\mathsf{nd}'$. This node replacement is necessary from the point of view of completeness (cf. \cite{greiner1989correction}).
%
Importantly, $\Queue_{dup}$ must be pruned before the other collections of nodes ($\Queue$, $\mD_{\supset}$, $\mD_{\times}$, $\mD_{\checkmark}$, $\mD_{calc}$) are pruned, to guarantee that all nodes in $\Queue_{dup}$ represent possible \emph{non-redundant} replacement nodes when it comes to pruning nodes from these collections. 

Additionally, the argument $X$ is used to update the conflicts stored for reuse (set $\mC_{calc}$). The principle is to run through $\mC_{calc}$ once, while deleting all proper supersets of $X$, and to finally add $X$ to $\mC_{calc}$.

\begin{figure*}
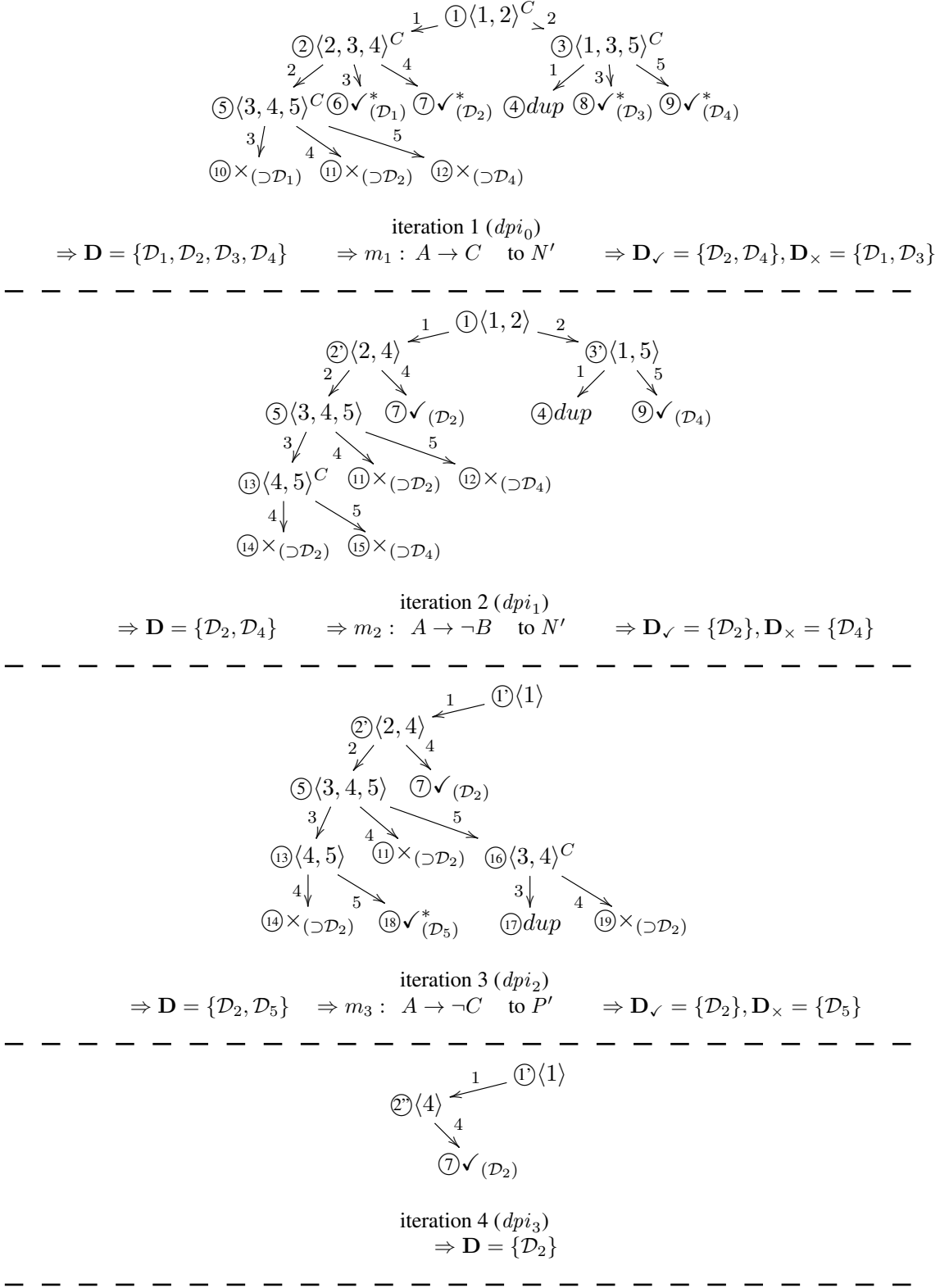

	\centering
		\[\xygraph{
			!{<0cm,0cm>;<1cm,0cm>:<0cm,-1cm>::}
			!{(0.1,0) }*+{ \textcircled{\scriptsize 1}\langle1,2\rangle^{C} }="n1"
			!{(-2.2,.5) }*+{ \textcircled{\scriptsize 2}\langle2,3,4\rangle^{C} }="n2"
			"n1":"n2"_{1}
			!{(2,.5) }*+{ \textcircled{\scriptsize 3}\langle1,3,5\rangle^{C} }="n3"
			"n1":"n3"^{2}
			!{(-3.5,1.5) }*+{ \textcircled{\scriptsize 5}\langle3,4,5\rangle^{C} }="n4"
			"n2":"n4"_(0.6){2}
			!{(-1.9,1.5) }*+{ \textcircled{\scriptsize 6}\checkmark_{(\md_1)}^* }="n5"
			"n2":"n5"_{3}
			!{(-0.5,1.5) }*+{ \textcircled{\scriptsize 7}\checkmark_{(\md_2)}^* }="n6"
			"n2":"n6"^{4}
			!{(0.8,1.5) }*+{ \textcircled{\scriptsize 4} dup }="n7"
			"n3":"n7"_(0.6){1}
			!{(2.1,1.5) }*+{ \textcircled{\scriptsize 8}\checkmark_{(\md_3)}^* }="n8"
			"n3":"n8"_{3}
			!{(3.5,1.5) }*+{ \textcircled{\scriptsize 9}\checkmark_{(\md_4)}^* }="n9"
			"n3":"n9"^{5}
			!{(-3.7,2.6) }*+{ \textcircled{\tiny 10} \XSup{\md_1} }="n10"
			"n4":"n10"_{3}
			!{(-1.9,2.6) }*+{ \textcircled{\tiny 11} \XSup{\md_2} }="n11"
			"n4":"n11"_{4}
			!{(-0.1,2.6) }*+{ \textcircled{\tiny 12} \XSup{\md_4} }="n12"
			"n4":"n12"^(0.6){5}
		}\]
		\begin{center}
			\small 
			\hspace{-10pt}
			iteration 1 ($\dpi_0$)\\
			$\qquad\Rightarrow \mD = \{\md_1,\md_2,\md_3,\md_4\}$ $\qquad\Rightarrow m_1:\, A\to C\phantom{\lnot}\;$ to $\Tn'$
			$\qquad\Rightarrow \mD_{\checkmark} = \{\md_2,\md_4\}, \mD_{\times} = \{\md_1,\md_3\}$
		\end{center}
		\vspace{-5pt}
		
		\hdashrule[0.5ex]{\columnwidth}{1pt}{3mm}
		\vspace{-38pt} 
	
		\[\xygraph{
			!{<0cm,0cm>;<1cm,0cm>:<0cm,-1cm>::}
			!{(-2.1,0) }*+{ \textcircled{\scriptsize 1}\langle 1,2\rangle }="n13"
			!{(-4.2,0.5) }*+{ \textcircled{\scriptsize 2'}\langle 2,4\rangle }="n14"
			"n13":"n14"_{1}
			!{(0,0.5) }*+{ \textcircled{\scriptsize 3'}\langle 1,5\rangle }="n15"
			"n13":"n15"^{2}
			!{(-5,1.5) }*+{ \textcircled{\scriptsize 5}\langle 3,4,5\rangle }="n16"
			"n14":"n16"_{2}
			!{(-3.2,1.5) }*+{ \textcircled{\scriptsize 7}\checkmark_{(\md_2)} }="n17"
			"n14":"n17"^{4}
			!{(-1,1.5) }*+{ \textcircled{\scriptsize 4} dup }="n18"
			"n15":"n18"_{1}
			!{(0.8,1.5) }*+{ \textcircled{\scriptsize 9} \checkmark_{(\md_4)} }="n19"
			"n15":"n19"^{5}
			!{(-5.5,2.6) }*+{ \textcircled{\tiny 13} \langle 4,5\rangle^C }="n19"
			"n16":"n19"_{3}
			!{(-3.7,2.6) }*+{ \textcircled{\tiny 11} \XSup{\md_2} }="n20"
			"n16":"n20"_(0.4){4}
			!{(-1.9,2.6) }*+{ \textcircled{\tiny 12} \XSup{\md_4} }="n21"
			"n16":"n21"^(0.6){5}
			!{(-5.5,3.7) }*+{ \textcircled{\tiny 14} \XSup{\md_2} }="n22"
			"n19":"n22"_{4}
			!{(-3.7,3.7) }*+{ \textcircled{\tiny 15} \XSup{\md_4} }="n23"
			"n19":"n23"^(0.6){5}
		}\]
		\begin{center}
			\small 
			iteration 2 ($\dpi_1$)\\
			$\qquad\Rightarrow \mD = \{\md_2,\md_4\}$ $\qquad\Rightarrow m_2:\;\, A\to\lnot B\phantom{\lnot}\;$ to $\Tn'$
			$\qquad\Rightarrow \mD_{\checkmark} = \{\md_2\}, \mD_{\times} = \{\md_4\}$
		\end{center}
		\vspace{-5pt}
		
		\hdashrule[0.5ex]{\columnwidth}{1pt}{3mm}
		\vspace{-38pt} 

		\[\xygraph{
			!{<0cm,0cm>;<1cm,0cm>:<0cm,-1cm>::}
			!{(-2.1,0) }*+{ \textcircled{\scriptsize 1'}\langle1\rangle }="n24"
			!{(-4.2,.5) }*+{ \textcircled{\scriptsize 2'}\langle2,4\rangle }="n25"
			"n24":"n25"_{1}
			!{(-5,1.5) }*+{ \textcircled{\scriptsize 5}\langle3,4,5\rangle }="n26"
			"n25":"n26"_{2}
			!{(-3.2,1.5) }*+{ \textcircled{\scriptsize 7}\checkmark_{(\md_2)} }="n27"
			"n25":"n27"^{4}
			!{(-5.5,2.6) }*+{ \textcircled{\tiny 13}\langle4,5\rangle }="n28"
			"n26":"n28"_{3}
			!{(-3.7,2.6) }*+{ \textcircled{\tiny 11}\XSup{\md_2} }="n29"
			"n26":"n29"_{4}
			!{(-1.9,2.6) }*+{ \textcircled{\tiny 16}\langle3,4\rangle^C }="n30"
			"n26":"n30"^(0.6){5}
			!{(-5.5,3.7) }*+{ \textcircled{\tiny 14}\XSup{\md_2} }="n31"
			"n28":"n31"_{4}
			!{(-3.7,3.7) }*+{ \textcircled{\tiny 18}\checkmark_{(\md_5)}^* }="n32"
			"n28":"n32"_{5}
			!{(-1.9,3.7) }*+{ \textcircled{\tiny 17} dup }="n33"
			"n30":"n33"_{3}
			!{(-0.1,3.7) }*+{ \textcircled{\tiny 19}\XSup{\md_2} }="n34"
			"n30":"n34"_{4}
		}\]
		\begin{center}
			\small 
		iteration 3 ($\dpi_2$)\\
		$\qquad\Rightarrow \mD = \{\md_2,\md_5\}$ $\quad\Rightarrow m_3:\;\, A\to\lnot C\phantom{\lnot}\;$ to $\Tp'$
		$\qquad\Rightarrow \mD_{\checkmark} = \{\md_2\}, \mD_{\times} = \{\md_5\}$
		\end{center}
		\vspace{-5pt}
		
		\hdashrule[0.5ex]{\columnwidth}{1pt}{3mm}
		\vspace{-38pt}
		
	
		\[\xygraph{
			!{<0cm,0cm>;<1cm,0cm>:<0cm,-1cm>::}
			!{(-2,0) }*+{ \textcircled{\scriptsize 1'}\langle1\rangle }="n35"
			!{(-4,.5) }*+{ \textcircled{\scriptsize 2''}\langle4\rangle }="n36"
			"n35":"n36"_{1}
			!{(-3,1.5) }*+{ \textcircled{\scriptsize 7}\checkmark_{(\md_2)} }="n37"
			"n36":"n37"^{4}
			!{(-4.3,1.5) }*+{ \hphantom{3pt}}="n38"
		}\]
		\begin{center}
			\small 
			iteration 4 ($\dpi_3$)\\
			$\qquad\Rightarrow \mD = \{\md_2\}$
		\end{center}
		\vspace{-5pt} 
		
		\hdashrule[0.5ex]{\columnwidth}{1pt}{3mm} \vspace{-18pt}
	
	\caption{DynamicHS executed on example DPI given in Tab.~\ref{tab:example_DPI}.}
	\label{fig:dynhs_example}
\end{figure*}




\begin{figure}[!htbp]
	\centering
		\[\xygraph{
			!{<0cm,0cm>;<1cm,0cm>:<0cm,-1cm>::}
			!{(-2.1,0) }*+{ \textcircled{\scriptsize 1}\langle1,2\rangle^{C} }="n1"
			!{(-4.2,.5) }*+{ \textcircled{\scriptsize 2}\langle2,3,4\rangle^{C} }="n2"
			"n1":"n2"_{1}
			!{(0,.5) }*+{ \textcircled{\scriptsize 3}\langle1,3,5\rangle^{C} }="n3"
			"n1":"n3"^{2}
			!{(-5.5,1.5) }*+{ \textcircled{\scriptsize 4}\langle3,4,5\rangle^{C} }="n4"
			"n2":"n4"_(0.6){2}
			!{(-3.9,1.5) }*+{ \textcircled{\scriptsize 5}\checkmark_{(\md_1)}^* }="n5"
			"n2":"n5"_{3}
			!{(-2.5,1.5) }*+{ \textcircled{\scriptsize 6}\checkmark_{(\md_2)}^* }="n6"
			"n2":"n6"^{4}
			!{(-1.2,1.5) }*+{ \textcircled{\scriptsize 7} \times }="n7"
			"n3":"n7"_(0.6){1}
			!{(0.1,1.5) }*+{ \textcircled{\scriptsize 8}\checkmark_{(\md_3)}^* }="n8"
			"n3":"n8"_{3}
			!{(1.5,1.5) }*+{ \textcircled{\scriptsize 9}\checkmark_{(\md_4)}^* }="n9"
			"n3":"n9"^{5}
			!{(-5.7,2.6) }*+{ \textcircled{\tiny 10} \XSup{\md_1} }="n10"
			"n4":"n10"_{3}
			!{(-3.9,2.6) }*+{ \textcircled{\tiny 11} \XSup{\md_2} }="n11"
			"n4":"n11"_{4}
			!{(-2.1,2.6) }*+{ \textcircled{\tiny 12} \XSup{\md_4} }="n12"
			"n4":"n12"^(0.6){5}
		}\]
		\begin{center}
			\small 
			iteration 1 ($\dpi_0$)\\
			$\qquad\Rightarrow \mD = \{\md_1,\md_2,\md_3,\md_4\}$ $\qquad\Rightarrow m_1:\;\, A\to C\phantom{\lnot}\;$ to $\Tn'$
		\end{center}
		\vspace{-5pt}
		
		\hdashrule[0.5ex]{\columnwidth}{1pt}{3mm} 
		\vspace{-38pt}
	
		\[\xygraph{
			!{<0cm,0cm>;<1cm,0cm>:<0cm,-1cm>::}
			!{(-2.1,0) }*+{ \textcircled{\tiny 13}\langle1,2\rangle^{C} }="n1"
			!{(-4.2,.5) }*+{ \textcircled{\tiny 14}\langle2,4\rangle^{C} }="n2"
			"n1":"n2"_{1}
			!{(0,.5) }*+{ \textcircled{\tiny 15}\langle1,5\rangle^{C} }="n3"
			"n1":"n3"^{2}
			!{(-5.5,1.5) }*+{ \textcircled{\tiny 16}\langle4,5\rangle^{C} }="n4"
			"n2":"n4"_(0.6){2}
			!{(-3,1.5) }*+{ \textcircled{\tiny 17}\checkmark_{(\md_2)}^* }="n6"
			"n2":"n6"^{4}
			!{(-0.7,1.5) }*+{ \textcircled{\tiny 18} \times }="n7"
			"n3":"n7"_(0.6){1}
			!{(1.5,1.5) }*+{ \textcircled{\tiny 19}\checkmark_{(\md_4)}^* }="n9"
			"n3":"n9"^{5}
			!{(-5.7,2.6) }*+{ \textcircled{\tiny 20} \XSup{\md_2} }="n11"
			"n4":"n11"_{4}
			!{(-3.5,2.6) }*+{ \textcircled{\tiny 21} \XSup{\md_4} }="n12"
			"n4":"n12"^(0.6){5}
		}\]
		\begin{center}
			\small 
			iteration 2 ($\dpi_1$)\\
			$\qquad\Rightarrow \mD = \{\md_2,\md_4\}$ $\qquad\Rightarrow m_2:\;\, A\to\lnot B\phantom{\lnot}\;$ to $\Tn'$
		\end{center}
		\vspace{-5pt}
		
		\hdashrule[0.5ex]{\columnwidth}{1pt}{3mm}
		\vspace{-38pt} 
	
		\[\xygraph{
			!{<0cm,0cm>;<1cm,0cm>:<0cm,-1cm>::}
			!{(-4.1,0) }*+{ \textcircled{\tiny 22}\langle1\rangle^{C} }="n1"
			!{(-6.2,.5) }*+{ \textcircled{\tiny 23}\langle2,4\rangle^{C} }="n2"
			"n1":"n2"_{1}
			!{(-7.5,1.5) }*+{ \textcircled{\tiny 24}\langle4,5\rangle^{C} }="n4"
			"n2":"n4"_(0.6){2}
			!{(-5,1.5) }*+{ \textcircled{\tiny 25}\checkmark_{(\md_2)}^* }="n6"
			"n2":"n6"^{4}
			!{(-7.7,2.6) }*+{ \textcircled{\tiny 26} \XSup{\md_2} }="n11"
			"n4":"n11"_{4}
			!{(-5.5,2.6) }*+{ \textcircled{\tiny 27} \langle3,4\rangle^{C} }="n12"
			"n4":"n12"^(0.6){5}
			!{(-6.3,3.7) }*+{ \textcircled{\tiny 28} \checkmark_{(\md_5)}^* }="n13"
			"n12":"n13"_{3}
			!{(-4.1,3.7) }*+{ \textcircled{\tiny 29} \XSup{\md_2} }="n14"
			"n12":"n14"^{4}
		}\]
		\begin{center}
			\small 
			iteration 3 ($\dpi_2$)\\
			$\qquad\Rightarrow \mD = \{\md_2,\md_5\}$ $\quad\Rightarrow m_3:\;\, A\to\lnot C\phantom{\lnot}\;$ to $\Tp'$
		\end{center}
		\vspace{-5pt}
		
		\hdashrule[0.5ex]{\columnwidth}{1pt}{3mm}
		\vspace{-38pt} 
	
		\[\xygraph{
			!{<0cm,0cm>;<1cm,0cm>:<0cm,-1cm>::}
			!{(-2,0) }*+{ \textcircled{\tiny 30}\langle1\rangle^C }="n35"
			!{(-4,.5) }*+{ \textcircled{\tiny 31}\langle4\rangle^C }="n36"
			"n35":"n36"_{1}
			!{(-3,1.5) }*+{ \textcircled{\tiny 32}\checkmark_{(\md_2)}^* }="n37"
			"n36":"n37"^{4}
			!{(-4.3,1.5) }*+{ \hphantom{3pt}}="n38"
		}\]
		\begin{center}
			\small 
			iteration 4 ($\dpi_3$)\\
			$\qquad\Rightarrow \mD = \{\md_2\}$
		\end{center}
		\vspace{-5pt} 
		
		\hdashrule[0.5ex]{\columnwidth}{1pt}{3mm}
		\vspace{-18pt}

	\caption{HS-Tree executed on example DPI given in Tab.~\ref{tab:example_DPI}.}
	\label{fig:hs_example}
\end{figure}

\begin{example}\label{ex:algo_description}
Consider our example DPI $\dpi_0$ given in Tab.~\ref{tab:example_DPI}. The goal is the localization of the faulty axioms, i.e., of those elements of $\mo$ 
that prevent the satisfaction of all given measurements (in this case, there is only one negative measurement, i.e., $\lnot A$ must not be entailed by the correct knowledge base). We now illustrate the workings of both DynamicHS (Fig.~\ref{fig:dynhs_example}) and HS-Tree (Fig.~\ref{fig:hs_example}) in that we play through a complete sequential diagnosis session 
on this particular example, under the assumption that $\dt = [\tax_1,\tax_4]$ is the actual diagnosis.\vspace{5pt} 

\noindent\emph{Inputs (Sequential Diagnosis):} Concerning the inputs to Alg.~\ref{algo:sequential_diagnosis}, we assume that $\ld := 5$ (i.e., in each iteration, if existent, five leading diagnoses are computed), $\mathsf{heur}$ is a heuristic that prefers measurements the more, the more leading diagnoses they eliminate in the worst case (cf.\ the split-in-half heuristic in \cite{Shchekotykhin2012}), and $\pr$ is chosen in a way all elements of $\mo$ have an equal fault probability 
which leads to a breadth-first construction of the hitting set tree(s).\vspace{5pt} 

\noindent\emph{Inputs (HS-Tree and DynamicHS):} 
The inputs to Alg.~\ref{algo:sequential_diagnosis} are passed to DynamicHS in line~\ref{algoline:inter_onto_debug:call_DynHS}.
In addition to these, 
DynamicHS (see Alg.~\ref{algo:dynamic_hs}) accepts the arguments $\Tp'$, $\Tn'$ (current sets of positive and negative measurements gathered throughout the execution of Alg.~\ref{algo:sequential_diagnosis} so far; both initially empty), $\mD_{\checkmark}$ and $\mD_{\times}$ (the minimal diagnoses found in the previous iteration which are consistent and inconsistent, respectively, with the last added measurement; both initially empty), and $\mathsf{state}$ (stores the state of the current DynamicHS-Tree within and between all its iterations) which consists of the collections $\Queue$ (list of open nodes; initially the list containing an empty node $[\,]$), $\Queue_{dup}$ (list of duplicate nodes; initially empty), $\mD_{\supset}$ (set of non-minimal diagnoses; initially empty), as well as $\mC_{calc}$ (set of already computed conflicts; initially empty). These variable initializations take place in lines~\ref{algoline:inter_onto_debug:var_init_start}--\ref{algoline:inter_onto_debug:var_init_end} of Alg.~\ref{algo:sequential_diagnosis}. Unlike DynamicHS, HS-Tree takes only the arguments $\Tp'$ and $\Tn'$ beyond the general sequential diagnosis inputs $\ld$, $\pr$ and $\mathsf{heur}$. That is, the queue $\Queue$, the computed diagnoses $\mD_{calc}$, and the computed conflicts $\mC_{calc}$ are instance variables in case of HS-Tree due to its statelessness.\vspace{5pt} 

\noindent\emph{Notation in Figures:} Axioms $\tax_i$ in $\mo$ are referred to by $i$ (in node and edge labels) for simplicity. 
Circled numbers indicate the chronological order in which nodes are labeled, i.e., node $\textcircled{\scriptsize i}$ is labeled at point in time $i$ (starting from $1$). We use $\tuple{\dots}$ to denote the conflicts that label the internal hitting set tree nodes, and tag these conflicts by $^C$ if they are freshly computed by a \textsc{findMinConflict} call (line~\ref{algoline:dlabel:qx_2} in \textsc{dLabel}), and leave them without a tag if they are the product of a redundancy check and a subsequent relabeling (lines~\ref{algoline:update:call_redundant_function}--\ref{algoline:update:prune} in \textsc{updateTree}).\footnote{The rationale behind this distinction between the individual conflicts used as node labels is that their computation time tends to differ significantly (as we will discuss in more detail in Sec.~\ref{sec:eval_advanced_tehniques}). For that reason, we will later also denote the operations used to compute the conflicts tagged by $^C$ as ``hard'', and those that yield conflicts without any tag as ``easy'' (see Sec.~\ref{sec:avoidance_of_exp_reasoning}).}\footnote{Note that there are no reused conflicts (found in lines~\ref{algoline:dlabel:reuse_start}--\ref{algoline:dlabel:reuse_end} in \textsc{dLabel}) in this particular example. Therefore, we do not introduce a distinct conflict tag for this case.} For the leaf nodes, we use the following labels:
$\checkmark_{(\md_i)}$ to denote a minimal diagnosis, named $\md_i$, stored in $\mD_{calc}$,
$\times$ to denote a duplicate in HS-Tree (see \ref{enum:hstree:label:duplicate} criterion above),
$\times_{(\supset \md_i)}$ to denote a non-minimal diagnosis (stored in $\mD_{\supset}$ by DynamicHS), where $\md_i$ is some minimal diagnosis that proves the non-minimality, and
$\mathit{dup}$ to denote a duplicate in DynamicHS (stored in $\Queue_{dup}$). 
Branches representing minimal diagnoses are additionally tagged by a $^*$ if logical reasoning 
(\textsc{findMinConflict} call, line~\ref{algoline:dlabel:qx_2} in DynamicHS, and line~\ref{algoline:hs:qx} in HS-Tree)
is necessary to prove it is a diagnosis, and untagged otherwise (i.e., the branch is known to be a diagnosis from a previous iteration, and stored in $\mD_{\checkmark}$; only pertinent to DynamicHS). Below each graph showing the respective hitting set tree, the figures show the current iteration $i$ of Alg.~\ref{algo:sequential_diagnosis}, the leading diagnoses $\mD$ output (Alg.~\ref{algo:sequential_diagnosis}, line~\ref{algoline:inter_onto_debug:call_DynHS} for DynamicHS and line~\ref{algoline:inter_onto_debug:call_HS-Tree} for HS-Tree, respectively), the measurement $m_i$ added to one of the measurement sets ($\Tp'$ or $\Tn'$) of the the DPI (Alg.~\ref{algo:sequential_diagnosis}, lines~\ref{algoline:inter_onto_debug:calc_query}--\ref{algoline:inter_onto_debug:add_meas}), as well as the sets $\mD_{\checkmark}$ and $\mD_{\times}$ (Alg.~\ref{algo:sequential_diagnosis}, line~\ref{algoline:inter_onto_debug:partition_D_into_Dcheckmark_and_Dtimes}) in case of DynamicHS.\vspace{5pt} 
     

\noindent\emph{Iteration~1:} In their first iteration, HS-Tree and DynamicHS essentially do the same, i.e., they build the same tree (compare 
Figs.~\ref{fig:dynhs_example} and \ref{fig:hs_example}). The only difference is that DynamicHS stores the duplicate branches (labeled by $\mathit{dup}$) and the ones corresponding to non-minimal diagnoses (labeled by $\times_{(\supset \md_i)}$),
whereas HS-Tree simply discards these branches (those labeled by $\times$ or $\times_{(\supset \md_i)}$). Note that duplicates are stored by DynamicHS at \emph{generation} time (line~\ref{algoline:dyn:add_to_Qdup}), hence the duplicate ($\mathit{dup}$) found has number \textcircled{\scriptsize 4} (and not \textcircled{\scriptsize 7}, as the respective branch for HS-Tree). The leading diagnoses computed by both algorithms are $\{\md_1,\md_2,\md_3,\md_4\} = \{[1,3],[1,4],[2,3],[2,5]\}$.
The soundness and completeness of this result can also be read from Tab.~\ref{tab:diag_conf_evolution_example} which illustrates the evolution of the minimal diagnoses ($\sol$) and minimal conflicts ($\conf$) for the varying DPI (successively extended $\Tp'$ and $\Tn'$ sets) in this example. We emphasize that the leading diagnoses returned by both algorithms \emph{must} be equal in any iteration (given the same parameters $\ld$ and $\pr$) since both methods are sound, complete and best-first minimal diagnosis computers. As a consequence, when using the same measurement computation technique (and heuristic $\mathsf{heur}$), both algorithms also \emph{must} give rise to the same proposed next measurement $m_i$ in each iteration. \vspace{5pt} 

\begin{table*}
	\footnotesize
	\centering
	\caption{\small Evolution of minimal diagnoses and minimal conflicts after successive extension of the example DPI $\dpi_0$ (Tab.~\ref{tab:example_DPI}) by positive ($\Tp'$) or negative ($\Tn'$) measurements $m_i$ shown in Figures~\ref{fig:dynhs_example} and \ref{fig:hs_example}, respectively. Newly arisen minimal conflicts (which are neither subsets nor supersets of prior minimal conflicts) are underlined.}
	\label{tab:diag_conf_evolution_example}
	\begin{tabular}{c|c|c|c|c}
		iteration $j$ & $\Tp'$ & $\Tn'$ & $\sol(\dpi_{j-1})$ & $\conf(\dpi_{j-1})$ \\
		\midrule
		1         & --   & --   &  $[1,3],[1,4],[2,3],[2,5]$      & $\tuple{1,2},\tuple{2,3,4},\tuple{1,3,5},\tuple{3,4,5}$           \\
		2  		  & --   &  $A\to C$  &  $[1,4],[2,5]$         &   $\tuple{1,2},\tuple{2,4},\tuple{1,5},\tuple{4,5}$        \\
		3         & --   &  $A\to C, \; A\to \lnot B$  &   $[1,4],[1,2,3,5]$        &   $\tuple{1},\tuple{2,4},\underline{\tuple{3,4}},\tuple{4,5}$        \\
		4         &  $A\to \lnot C$  &  $A\to C, \; A\to \lnot B$  &  $[1,4]$         &  $\tuple{1},\tuple{4}$        
	\end{tabular}
\end{table*}

\noindent\emph{First Measurement:} Accordingly, both algorithms lead to the first measurement $m_1: A \to C$, 
which corresponds to the question ``must it be true that $A$ implies $C$?''. This measurement
turns out to be negative, e.g., by consulting a knowledgeable expert for the domain described by $\mo$, and is therefore added to $\Tn'$. This effectuates a transition from the initial DPI $\dpi_0$ to a new DPI $\dpi_1$ (which includes the additional element $A \to C$ in $\Tn'$), and thus a change of the relevant minimal diagnoses and conflicts (see Tab.~\ref{tab:diag_conf_evolution_example}). \vspace{5pt}  

\noindent\emph{Tree Update:} From the second iteration (where $\dpi_1$ is considered) on, the behaviors of DynamicHS and HS-Tree start to differ. Whereas HS-Tree simply constructs a new hitting set tree from scratch for $\dpi_1$, DynamicHS runs a tree update (function \textsc{updateTree}) to make the existing tree built for $\dpi_0$ reusable for $\dpi_1$. In the course of this tree update, two witnesses of redundancy (the new minimal conflicts $\tuple{2,4}$ and $\tuple{1,5}$) are found while analyzing the (conflicts along the) branches of the two invalidated diagnoses $[1,3]$ and $[2,3]$ (\textcircled{\scriptsize 6} and \textcircled{\scriptsize 8}). 
For instance, $\mathsf{nd}=[1,3]$ is redundant since the (\emph{former}, for $\dpi_0$, \emph{minimal}) conflict $\mathsf{nd.cs}[2] = \tuple{2,3,4}$ is a proper superset of the \emph{now minimal} conflict $X = \tuple{2,4}$ \emph{and} $\mathsf{nd}$'s outgoing edge of $\mathsf{nd.cs}[2]$ is $\mathsf{nd}[2] = 3$ which is an element of $\mathsf{nd.cs}[2]\setminus X = \setof{3}$.  
Since there are no appropriate duplicates that allow the construction of a replacement node for any of the two redundant branches $[1,3]$ and $[2,3]$, both of these branches are removed from the tree. Further, the by now
non-minimal conflicts at 
\textcircled{\scriptsize 2} and \textcircled{\scriptsize 3} are replaced, each by the respective witness of redundancy that is a subset of it (e.g., $\tuple{2,3,4}$ is replaced by $\tuple{2,4}$). This relabeling is signified by the prime (') in the new node numbers \textcircled{\scriptsize 2'} and \textcircled{\scriptsize 3'}, respectively. 

Other than that, only a single further change is induced by \textsc{updateTree}. Namely, the branch $[1,2,3]$, a non-minimal diagnosis for $\dpi_0$, is returned to the queue of unlabeled nodes $\Queue$ because there is no longer a diagnosis in the tree witnessing its non-minimality (both such witnesses $[1,3]$ and $[2,3]$ have been discarded). Note that, first, $[1,2,3]$ is indeed no longer a diagnosis, i.e., hitting set of all minimal conflicts for $\dpi_1$ (cf.\ Tab.~\ref{tab:diag_conf_evolution_example}) and, second, there is still a non-minimality witness for all other branches (\textcircled{\tiny 12} and \textcircled{\tiny 13}) representing non-minimal diagnoses for $\dpi_0$, which is why they remain labeled by $\times_{(\supset \md_i)}$. \vspace{5pt} 

\noindent\emph{Iteration~2:} Observe the crucial differences between HS-Tree and DynamicHS in the second iteration (compare Figs.~\ref{fig:dynhs_example} and \ref{fig:hs_example}). 

First, while HS-Tree has to compute all the conflicts labeling internal nodes by potentially expensive \textsc{findMinConflict} calls (see the $^C$ tags), DynamicHS performs a (generally) cheaper reduction of the existing conflicts in the course of the pruning actions we discussed above. However, note that not all conflicts are necessarily always kept up-to-date after a DPI-transition. This is part of the \emph{lazy updating policy} pursued by DynamicHS, detailed in Sec.~\ref{sec:lazy_updating_policy}. For instance, node \textcircled{\scriptsize 5} is still labeled by the now non-minimal conflict $\tuple{3,4,5}$ at the time \textsc{updateTree} terminates. Hence, the subtree comprising nodes \textcircled{\tiny 13}, \textcircled{\tiny 14} and \textcircled{\tiny 15} is not present in case HS-Tree is used. 
Importantly, this lazy updating strategy has no negative impact on the soundness or completeness of DynamicHS (see Sec.~\ref{sec:lazy_updating_policy} and Theorem~\ref{thm:dynHS_correctness}).
%
%

Second, the verification of the minimal diagnoses ($\md_2$, $\md_4$) found in iteration~2 requires logical reasoning in HS-Tree (see $^*$ tags of $\checkmark$ nodes) whereas it comes for free in DynamicHS due to the storage and reuse of $\mD_{\checkmark}$ (the minimal diagnoses returned by the previous iteration which are consistent with the lastly added measurement).\vspace{5pt}

\noindent\emph{Remaining Execution:} After the second measurement $m_2$ is added to $\Tn'$, causing a DPI-transition once again, DynamicHS reduces the conflict that labels the root node. This leads to the pruning of the complete right subtree. The left subtree is then further expanded in iteration~3 (see generated nodes \textcircled{\tiny 16}, \textcircled{\tiny 17}, \textcircled{\tiny 18} and \textcircled{\tiny 19}) until the two leading diagnoses $\md_2 = [1,4]$ and $\md_5 = [1,2,3,5]$ are located and the queue $\Queue$ of unlabeled nodes becomes empty (which proves that no further minimal diagnoses exist). Eventually, the addition of the third measurement $m_3$ to $\Tp'$ brings sufficient information to isolate the actual diagnosis. This is reflected by a pruning of all branches except for the one representing the actual diagnosis $[1,4]$.\vspace{5pt}

\noindent\emph{Performance Comparison:}
As Figs.~\ref{fig:dynhs_example} and \ref{fig:hs_example} show, DynamicHS generates 19 nodes and requires 6 conflict computations, as opposed to 32 nodes and 14 computations in case of HS-Tree. These reductions in terms of tree rebuilding and conflict computation costs represent two main factors responsible for runtime improvements of DynamicHS over HS-Tree.
\qed	
\end{example}

\subsection{Advanced Techniques in DynamicHS}
\label{sec:advanced_techniques_in_DynHS}
DynamicHS embraces several sophisticated techniques specialized in improving its (time or space) performance,
which we discuss next. 

\subsubsection{Efficient Redundancy Checking}
\label{sec:efficient_redundancy_checking}
We now detail the workings of the function \textsc{redundant} (called in line~\ref{algoline:update:call_redundant_function} of Alg.~\ref{algo:dynamic_hs}):

The definition of node redundancy given in Sec.~\ref{sec:algo_walkthrough} directly suggests a method for checking whether or not a node is redundant, which we call \emph{complete redundancy check (CRC)}.
It runs through all conflicts $\mathsf{nd.cs}[i]$ used as labels along the branch to node $\mathsf{nd}$ (i.e., $i \in \setof{1,\dots,|\mathsf{nd.cs}|}$) and calls 
\textsc{findMinConflict} with arguments $(\tuple{\mathsf{nd.cs}[i] \setminus \setof{\mathsf{nd}[i]}, \mb, \Tp\cup\Tp', \Tn\cup\Tn'})$ to test if there is a witness of redundancy (see Sec.~\ref{sec:algo_walkthrough}) for $\mathsf{nd}$. A witness of redundancy exists for $\mathsf{nd}$ iff, for some $i$, this call to \textsc{findMinConflict} returns a conflict $X$. Because, this conflict then must be a subset of $\mathsf{nd.cs}[i] \setminus \setof{\mathsf{nd}[i]}$, meaning that $\mathsf{nd}$ is redundant. 
Hence, if such an $X$ is found, then CRC is successful (returns $\true$) and the witness $X$ is used as an argument passed to the subsequent call of the \textsc{prune} method (Alg.~\ref{algo:dynamic_hs}, line~\ref{algoline:update:prune}). Otherwise, i.e., if all calls of \textsc{findMinConflict} made by the CRC return 'no conflict', it is proven that the node is not redundant and CRC returns $\false$.

The CRC enables sound (if CRC true, then node redundant) and complete (if node redundant, then CRC true) redundancy checking. However, a drawback of the CRC is that it requires $|\mathsf{nd}|$ calls to the (expensive) method \textsc{findMinConflict} in the worst case, where $|\mathsf{nd}|$ is in $O(|\conf(\dpi)|)$ since a node cannot hit any more than each minimal conflict for the current DPI $\dpi$.
As a remedy to that, we devised a more efficient, sound but incomplete, so-called \emph{quick redundancy check (QRC)}, which is executed previous to the CRC and requires only a single call of \textsc{findMinConflict}. The concept is that a positive QRC makes the more expensive CRC obsolete; and, in case of a negative outcome, CRC must be executed, but the overhead amounts to only a single \textsc{findMinConflict} call. 

To check the redundancy of $\mathsf{nd}$, QRC executes \textsc{findMinConflict} with arguments $(\langle U_{\mathsf{nd.cs}} \setminus \mathsf{nd}, \mb, \Tp\cup\Tp', \Tn\cup\Tn'\rangle)$.\footnote{For a collection of sets $Z$, we denote by $U_{Z}$ the union of all sets in $Z$.} If 'no conflict' is returned, the QRC terminates negatively, which prompts the execution of the CRC. Otherwise, if a conflict $X$ is returned, QRC checks whether $X$ is a proper subset of some conflict in $\mathsf{nd.cs}$, i.e., whether $X \subset \mc$ for $\mc = \mathsf{nd.cs}[k]$ for some $k$.
In case of a positive subset-check, the QRC returns positively and it follows that $\mathsf{nd}$ is redundant, regardless of the particular $k$. 
The reason is that the argument $U_{\mathsf{nd.cs}} \setminus \mathsf{nd}$ passed to \textsc{findMinConflict} does not include any element of $\mathsf{nd}$, and hence the output conflict cannot include such elements either.
Thus, if $X \subset \mathsf{nd.cs}[k]$ holds, then $X \subset \mathsf{nd.cs}[k]\setminus\mathsf{nd}$, and therefore $X \subset \mathsf{nd.cs}[k]\setminus\{\mathsf{nd}[i]\}$ for all $i$, and in particular for $i = k$.
So, $X \subset \mathsf{nd.cs}[k]\setminus\{\mathsf{nd}[k]\}$, which is equivalent to the definition of redundancy (see page~\pageref{def:redundancy}).

To see why the QRC is incomplete, i.e., that $\mathsf{nd}$ \emph{can} be redundant even if the outcome of the QRC is negative, consider the following example: 

\begin{example}
	Let $\mathsf{nd} = [1,2]$ and $\mathsf{nd.cs} = [\tuple{1,2},\tuple{2,3}]$. Assume that $X := \tuple{2}$ is a new minimal conflict and that $\setof{3}$ is not a conflict. Clearly, this implies that $\mathsf{nd}$ is redundant because $\mathsf{nd.cs}[1]\setminus X = \setof{1}$ and $\mathsf{nd}[1] = 1$. However, $U_{\mathsf{nd.cs}}\setminus \mathsf{nd} = \setof{3}$, which is not a conflict, which is why \textsc{findMinConflict} given the DPI $(\langle U_{\mathsf{nd.cs}} \setminus \mathsf{nd}, \mb, \Tp\cup\Tp', \Tn\cup\Tn'\rangle)$ as argument returns 'no conflict'. Hence, the QRC returns negatively although $\mathsf{nd}$ is in fact redundant.\qed  	
\end{example}
The crucial aspect which makes this incompleteness possible is the potential overlapping of conflicts. Exactly this overlapping effectuates in the above example that more than one element (actually even all elements) of the outdated non-minimal conflict $\tuple{1,2}$ are eliminated from 
$U_{\mathsf{nd.cs}} = \setof{1,2,3}$ 
by deleting $\mathsf{nd} = [1,2]$. As a consequence, the new reduced conflict $\tuple{2}$ is not contained any longer in the set tested by \textsc{findMinConflict}.

In fact, we can conclude that the QRC is sound \emph{and complete} in the special cases where all minimal conflicts are pairwise disjoint or, more generally, where $\mathsf{nd}$ does not include any element that occurs in multiple conflicts in $\mathsf{nd.cs}$. 

\subsubsection{Lazy Updating Policy}
\label{sec:lazy_updating_policy}
Updating DynamicHS's hitting set tree after the detection of some witness of redundancy $X$ 
involves going through all nodes of the tree and checking their redundancy wrt.\ $X$ (cf.\ Sec.~\ref{sec:algo_walkthrough}). To avoid these costs as much as possible, DynamicHS aims at minimizing the number of performed updates under preservation of its correctness. 
This can be seen from line~\ref{algoline:update:process_Dtimes_start}, where only the set $\mD_{\times}$ including the most ``suspicious'' nodes (i.e., the diagnoses invalidated by the latest added measurement) is checked for redundancy.
In general, this means that 
we allow some differences between the tree used by DynamicHS to compute diagnoses and the one that would be obtained when building Reiter's HS-Tree for the current DPI from scratch.
More specifically, we allow the presence of non-minimal conflicts that are used as node labels as well as---conditioned by these non-minimal conflicts---the presence of unnecessary nodes (while, of course, seeking to minimize the number of such occurrences; cf.\ \textsc{prune} function).
Still, every time a conflict is used to label a (newly processed) node, the algorithm guarantees that it is a minimal conflict for the current DPI (cf.\ the conflict-minimality test in the course of the conflict reuse check in \textsc{dLabel}, lines~\ref{algoline:dlabel:qx_1}--\ref{algoline:dlabel:return_X}).

This \emph{lazy updating policy} takes effect, e.g., in iteration~2 of the example execution of DynamicHS shown in Fig.~\ref{fig:dynhs_example} (see Example~\ref{ex:algo_description}). Here the, at this point, already non-minimal conflict $\mc_{\mathit{\lnot min}} := \tuple{3,4,5}$ still appears as a node label, while $\mc_{\mathit{min}}:=\tuple{4,5}$ is now a minimal conflict for the current DPI. 
%

Notwithstanding the correctness proof of DynamicHS given in Sec.~\ref{sec:correctness_proof}, we next argue briefly why the presence of such non-minimal conflicts $\mc_{\mathit{\lnot min}}$ does neither counteract the soundness nor the completeness of DynamicHS. To this end, we first point out that (*) only nodes can be labeled $\mathit{valid}$ by DynamicHS which are 
diagnoses for the current DPI (see line~\ref{algoline:dlabel:qx_2} in \textsc{dLabel}; and line~\ref{algoline:dyn:node_in_Dcheckmark} along with the definition of $\mD_{\checkmark}$):
\begin{itemize}
	\item Assume the \emph{completeness} is compromised. Then 
	the processing of some minimal diagnosis must be prevented from reaching line~\ref{algoline:dlabel:qx_2} in \textsc{dLabel} (note that it is unavoidable that processed nodes are recognized as diagnoses in line~\ref{algoline:dyn:node_in_Dcheckmark}). Since the presence of all still relevant (i.e., non-redundant) nodes is not harmed by the presence of redundant nodes, which just constitute \emph{additional} branches, each node corresponding to a minimal diagnosis $\md$ will sooner or later be processed by \textsc{dLabel}. As $\md$ does not reach line~\ref{algoline:dlabel:qx_2} and since a minimal diagnosis does hit all (minimal) conflicts, a labeling of $\md$ by $\mathit{nonmin}$ (line~\ref{algoline:dlabel:non-min_crit_end}) is the only possible case.
	That is, 
	there must be a node labeled $\mathit{valid}$ (and thus stored in $\mD_{calc}$) which is a proper subset of $\md$. This is a contradiction to (*).
	\item Assume the \emph{soundness} is compromised. Then some node $\mathsf{nd}$ is labeled $\mathit{valid}$ which is not a minimal diagnosis. Due to (*), it must be the case that $\mathsf{nd}$ is a diagnosis, but a non-minimal one. First, a non-minimal diagnosis can never be identified in line~\ref{algoline:dyn:node_in_Dcheckmark} because of Property~\ref{property:impact_of_DPI_transition}.\ref{property:impact_of_DPI_transition:enum:diags_can_only_grow}. Second, since the queue $\Queue$ is always sorted such that node $\mathsf{n}$ is ranked prior to node $\mathsf{n}'$ whenever $\mathsf{n}\subset\mathsf{n}'$ (cf.\ 
	Sec.~\ref{sec:Reiters_HS-Tree}),
	the non-minimal diagnosis $\mathsf{nd}$ can only be labeled $\mathit{valid}$ if some minimal diagnosis $\md$ ($\subset \mathsf{nd}$)
	---processed prior to $\mathsf{nd}$---  
	is not found to be a diagnosis (and thus not labeled $\mathit{valid}$ and not stored in $\mD_{calc}$). Hence, the completeness must be compromised. This is a contradiction to the argumentation in the above bullet.
\end{itemize}

\subsubsection{Avoidance of Expensive Reasoning}
\label{sec:avoidance_of_exp_reasoning}
As explained in Sec.~\ref{sec:basics}, DynamicHS aims at reducing the reaction \emph{time} of a sequential diagnosis system. One crucial time-consuming and recurring operation in hitting-set-based diagnosis computation algorithms
is the reasoning in terms of logical consistency checks. To optimize computation time, DynamicHS is therefore equipped with strategies that minimize the amount of and time spent for reasoning by exploiting its statefulness in terms of the hitting set tree maintained throughout its iterations. We discuss the concept behind these strategies next. \vspace{5pt}

\noindent\emph{Ways of Reducing the Cost for Reasoning.} In DynamicHS, the logical inference engine is called (solely) by the \textsc{findMinConflict} function, which is involved in the determination of node labels in the tree expansion phase (\textsc{dLabel} function) and in the evaluation of node redundancy in the tree update phase (\textsc{redundant} function). 
As discussed in Sec.~\ref{sec:conflicts}, a single execution of \textsc{findMinConflict} given the DPI $\langle\mo,\mb,\Tp,\Tn\rangle$ generally requires multiple reasoner calls
and their number depends critically on the size of the universe $\mo$ 
from which a minimal conflict should be computed.
Note, what we, for simplicity, refer to as a \emph{reasoner call} actually
corresponds to a check if some $\mc \subseteq \mo$ is a 
conflict for a DPI $\langle\mo,\mb,\Tp,\Tn\rangle$. By the definition of a conflict (see Sec.~\ref{sec:conflicts}), this means checking whether some $x \in \Tn \cup \{\bot\}$ exists such that $\mc \cup \mb \cup \Tp \models x$. Consequently, a reasoner call corresponds to a maximum of $|\Tn|+1$ logical consistency checks.
E.g., if QuickXPlain \cite{junker04,rodler2020qx} is used to implement the \textsc{findMinConflict} function, as in our evaluations (cf.\ Sec.~\ref{sec:eval}), 
then the worst-case number of consistency checks executed by a single call of \textsc{findMinConflict} on 
the DPI $\langle\mo,\mb,\Tp,\Tn\rangle$ 
is in $O(|\mo|(|\Tn|+1))$ \cite{marques2013minimal}. 
%
%
The hardness of consistency checking tends to increase 
with the size of the knowledge base on which the check is performed (cf., e.g., \cite{gonccalves2012performance}).\footnote{Further evidence that larger knowledge bases lead to worse reasoning performance is given in \cite{kang2012predicting,karlsson2014does}.} In other words, the smaller the size of $\mo\cup\mb\cup\Tp$ is, the more efficient consistency checking will tend to be in the course of \textsc{findMinConflict} operating on the DPI $\langle \mo,\mb,\Tp,\Tn\rangle$. 
%

In summary, 
the lower the cardinality of the first entry $\mo$ (number of system components) of the tuple $\langle \mo,\mb,\Tp,\Tn\rangle$ provided as an input to \textsc{findMinConflict} is, the lower the hardness and the number of executed consistency checks will tend to be, and thus the faster \textsc{findMinConflict} will tend to execute. 
Hence, there are basically three different ways of scaling down the necessary reasoning throughout DynamicHS: 
\begin{enumerate}[label=\textit{(\roman*)},noitemsep]
	\item \label{enum:reduce_reasoning:minimize_num_of_CCs} Reducing the number and hardness of consistency checks made while \textsc{findMinConflict} executes,
	\item \label{enum:reduce_reasoning:minimize_num_of_findMinConf_calls} reducing the number of \textsc{findMinConflict} calls, or
	\item \label{enum:reduce_reasoning:avoid_findMinConf_calls} entirely avoiding \textsc{findMinConflict} calls (and replacing them by 
	equivalents that do not involve reasoning).
\end{enumerate}
In its various stages, DynamicHS embraces all these three approaches, as we explain next.\vspace{5pt} 

\noindent\emph{In the tree expansion phase}, any \textsc{findMinConflict} call in the course of the conflict reuse check (lines~\ref{algoline:dlabel:reuse_start}--\ref{algoline:dlabel:reuse_end}) starts from an \emph{already computed} conflict---and \emph{not} from the entire set of system components $\mo$---trying to verify its minimality or, alternatively, extracting a subset which constitutes a minimal conflict (for the current DPI). That is, \textsc{findMinConflict} is given a set of at most $|\mc_{\max}|$ elements as an input,\footnote{When we say that \textsc{findMinConflict} gets a set $S$ as an input, we always mean by $S$ the \emph{first} argument of the 4-tuple (DPI) passed to the function as an argument.} where $\mc_{\max}$ is the conflict of maximal size 
(for the original\footnote{Recall from Property~\ref{property:impact_of_DPI_transition}.\ref{property:impact_of_DPI_transition:enum:conflicts_can_only_shrink} that minimal conflicts can only become smaller throughout a sequential diagnosis session, i.e., there cannot be any minimal conflict whose size exceeds $|\mc_{\max}|$.\label{footnote:C_max_explanation}} DPI, i.e., the one given as an input to Alg.~\ref{algo:sequential_diagnosis}).
Note that, in many practical applications involving systems of non-negligible size, $|\mc_{\max}|$ is significantly (if not orders of magnitude) smaller than the number of components of the diagnosed system (cf., e.g., \cite{Horridge2012b,Shchekotykhin2012,shchekotykhin2015mergexplain}). Thus, DynamicHS applies strategy \ref{enum:reduce_reasoning:minimize_num_of_CCs} in this stage.\vspace{5pt} 

\noindent\emph{In the tree update phase}, and particularly during redundancy detection (line~\ref{algoline:update:call_redundant_function}), the quick redundancy check (QRC) is employed 
to potentially replace the complete redundancy check (CRC) by making only a single \textsc{findMinConflict} call instead of multiple ones (cf.\ Sec.~\ref{sec:efficient_redundancy_checking}). Moreover, any call of \textsc{findMinConflict} made in the course of the redundancy detection in general involves a significantly reduced input set, as compared to the overall number of components $|\mo|$ of the diagnosed system. 
The cardinality of this input set is bounded by the cardinality of the union of all minimal conflicts (for the original\footnote{Cf.\ Footnote~\ref{footnote:C_max_explanation}} DPI). 
So, during redundancy checking, both strategies \ref{enum:reduce_reasoning:minimize_num_of_CCs} and \ref{enum:reduce_reasoning:minimize_num_of_findMinConf_calls} are pursued.\vspace{5pt}

\noindent\emph{In the tree pruning phase} (\textsc{prune} function), which constitutes a part of the tree update phase, no reasoning is required at all (i.e., no 
\textsc{findMinConflict} calls).
%
This is accomplished by leveraging the stored hitting set tree as well as adequate instructions operating on sets and lists (cf.\ Example~\ref{ex:advanced_techniques} below).
At this point, note that a stateless algorithm, in contrast, \emph{has to} draw on logical reasoning to reconstruct (the still relevant) parts of the tree, and thus to achieve essentially the same as DynamicHS's pruning actions. 
Importantly, operations relying on a logical inference engine 
can be expected to have a (much)
higher time complexity than the list concatenations, set-equality or subset checks performed by DynamicHS in lieu of these operations. For instance, reasoning with propositional logic is already NP-complete \cite{cook1971complexity}, not to mention more expressive languages such as Description logics \cite{DLHandbook}, whereas set- and list-operations are (mostly linear) polynomial time operations. So, as far as the tree pruning is concerned, DynamicHS can be viewed as trading cheaper (reasoner-free) operations for expensive reasoner calls. It thus makes use of strategy~\ref{enum:reduce_reasoning:avoid_findMinConf_calls}.\vspace{5pt}

\noindent\emph{Different Types of Reasoning Operations.} Given the preceding discussion, we can at the core distinguish the following 
three categories of 
\textsc{findMinConflict} function calls
based on their input DPI $\langle X,\mb,\Tp,\Tn\rangle$ (where $\langle \mo,\mb,\Tp,\Tn\rangle$ is the DPI relevant to the current iteration of DynamicHS):\footnote{It is important to note that the used \emph{intuitive} terminology ``hard'', ``medium'' and ``easy'' is to be understood \emph{by tendency}, but does not allow \emph{general} conclusions about the relative or absolute hardness of the respective \textsc{findMinConflict} calls. That is, e.g., an ``easy'' call might not be fast or easy at all. Or a ``medium'' call might be faster than an ``easy'' one, e.g., because the latter operates on a set of logical sentences that is particularly hard to reason with (cf.\ \cite{gonccalves2012performance}).
However, as we verified in our experimental evaluation (see Sec.~\ref{sec:eval}), the used terminology does largely reflect the actual relative computation times of the \textsc{findMinConflict} calls in our considered dataset. More precisely, on average ``hard'' calls turned out to be \emph{always} (and at least four times and up to more than 100 times) more time-intensive than ``medium'' and ``easy'' calls; and, ``easy'' calls terminated faster than ``medium'' ones in 77\,\% of the studied cases.} 

\begin{itemize}[noitemsep,leftmargin=*]
	\item \emph{''hard''}: Size of $X$ in the order of number of system components, i.e., $|X| \approx |\mo|$, \emph{and} a conflict is returned. (\emph{multiple ``hard'' reasoner calls})
	\item \emph{''medium''}: Size of $X$ in the order of number of system components, i.e., $|X| \approx |\mo|$, \emph{and} 'no conflict' is returned. (\emph{single ``hard'' reasoner call})
	\item \emph{''easy''}: Size of $X$ low compared to the number of system components, i.e., $|X| \ll |\mo|$. (\emph{few ``easy'' reasoner calls})
\end{itemize}
``Hard'' \textsc{findMinConflict} calls are those executions of line~\ref{algoline:dlabel:qx_2} (in \textsc{dLabel}) that compute a fresh conflict, and ``medium'' ones those which lead to the finding of a diagnosis (output 'no conflict').
In contrast, ``easy'' \textsc{findMinConflict} invocations are those geared towards redundancy checking (line~\ref{algoline:update:call_redundant_function}, \textsc{updateTree}) and minimality testing for reused conflicts (line~\ref{algoline:dlabel:qx_1}, \textsc{dLabel}).
In terms of this characterization, compared against 
HS-Tree, DynamicHS tries to substantially reduce the ``hard'' (and ``medium'') \textsc{findMinConflict} operations at the cost of performing an as small as possible number of ``easy'' ones.

\renewcommand{\arraystretch}{1.1}
\begin{table}[]
	\footnotesize
	\centering
	\caption{\small Stats wrt.\ the number of different kinds of reasoning operations (\textsc{findMinConflicts} calls) throughout the execution of DynamicHS (DHS) and HS-Tree (HST), respectively, on the example DPI from Tab.~\ref{tab:example_DPI}. ``U$i$'' in the first column refers to the tree update performed by DynamicHS subsequent to iteration $i$. Note, HS-Tree does not (need to) perform any tree updates.}
	\label{tab:stats_reasoning_example}
	\begin{tabular}{c|cc|cc|cc}
		& \multicolumn{2}{c|}{\# ``hard''} & \multicolumn{2}{c|}{\# ``medium''} & \multicolumn{2}{c}{\# ``easy''} \\
		iteration & DHS                & HST               & DHS                & HST               & DHS         & HST         \\
		\midrule
		1         & 4                  & 4                 & 4                  & 4                 & 0           & 0           \\
		U1  	  & 0                  & --                & 0                  & --                & 2           & --          \\
		2         & 1                  & 4                 & 0                  & 2                 & 0           & 0           \\
		U2        & 0                  & --                & 0                  & --                & 1           & --          \\
		3         & 1                  & 4                 & 1                  & 2                 & 0           & 0           \\
		U3        & 0                  & --                & 0                  & --                & 1           & --          \\
		4         & 0                  & 2                 & 0                  & 1                 & 0           & 0           \\
		\midrule
		total     & 6                  & 14                & 5                  & 9                 & 4           & 0          
	\end{tabular}
\end{table}

\begin{example}\label{ex:hard_med_easy_consistency_checks_in_example_DPI}
	Reconsider our example DPI in Tab.~\ref{tab:example_DPI} and the evolution of the hitting set computation throughout a sequential diagnosis session for DynamicHS and HS-Tree discussed in Example~\ref{ex:algo_description}. 
	Tab.~\ref{tab:stats_reasoning_example} shows the number of ``hard'', ``medium'' and ``easy'' \textsc{findMinConflict} calls throughout the sequential session executed by both algorithms. The ``hard'' and ``medium'' calls are denoted by $^C$ and $^*$, respectively, in the hitting set trees depicted by Figs.~\ref{fig:dynhs_example} and \ref{fig:hs_example} (``easy'' calls are not indicated as tree updates are not displayed in the figures). We can see that DynamicHS trades a significant reduction of ``hard'' (57\,\%) and ``medium'' (44\,\%) reasoner operations for some ``easy'' ones.\qed  	
\end{example}

\subsubsection{Space-Saving Duplicate Storage and On-Demand Reconstruction}
\label{sec:space-saving_dup_storage}
Basically, there are several options how to 
organize the storage of
duplicate nodes.
These options range from storing all of them \emph{explicitly} to storing only a minimal set of (stubs of) duplicate nodes that \emph{implicitly} allow all duplicates to be reconstructed on demand.
Since DynamicHS performs the duplicate check at node generation time (cf.\ line~\ref{algoline:dyn:check_node_already_in_Q}), it uses the more natural way of handling duplicate storage given by the latter strategy. This means directly adding detected duplicate nodes---generated tree branches whose \emph{set} of edge labels equals the \emph{set} of edge labels of an active branch (node in $\Queue$ or $\mD_{\supset}$)---to the collection $\Queue_{dup}$ without further extending them as the hitting set tree grows. Hence, each node stored in $\Queue_{dup}$ is potentially only a partial duplicate node and might need to be combined with some other (partial) duplicate node 
or some active node in the hitting set tree 
to explicitly generate (or: reconstruct) a duplicate that is only implicitly stored. For instance, assume two nodes $\mathsf{n}_1, \mathsf{n}_2$ that have been detected as duplicates and added to $\Queue_{dup}$, where $\mathsf{n}_1 = [3,2], \mathsf{n}_1.\mathsf{cs} = [\tuple{1,2,3},\tuple{2,4}]$ and $\mathsf{n}_2 = [2,3,1], \mathsf{n}_2.\mathsf{cs} = [\tuple{1,2,3},\tuple{3,4},\tuple{1,4}]$ (cf.\ nodes with numbers \textcircled{\scriptsize 5} and \textcircled{\scriptsize 9} in Fig.~\ref{fig:tree_pruning_example_HStree}, discussed in more detail in Example~\ref{ex:advanced_techniques}). Then, an implicit duplicate constructible from $\mathsf{n}_1$ and $\mathsf{n}_2$ is $\mathsf{n}_{1,2} = [3,2,1], \mathsf{n}_{1,2}.\mathsf{cs} = [\tuple{1,2,3},\tuple{2,4},\tuple{1,4}]$, where the last node label (conflict $\tuple{1,4}$) and edge (labeled by $1$) of $\mathsf{n}_2$ have been appended to $\mathsf{n}_1$. The rationale behind this node combination is as follows: $\mathsf{n}_1$ was recognized as duplicate first, while the first part (i.e., the first two node and edge labels) of $\mathsf{n}_2$ was still in the queue $\Queue$ of open nodes. This first part of $\mathsf{n}_2$ was then extended by the node label $\tuple{1,4}$ and the edge label $1$, but was subsequently itself spotted as a duplicate. Node $\mathsf{n}_1$, however, given it had still been in $\Queue$, would have undergone the same extension. This extension is so to say ``made good for'' by combining $\mathsf{n}_1$ with $\mathsf{n}_2$ to (re)construct $\mathsf{n}_{1,2}$. 

In general, to reconstruct a duplicate node $\mathsf{n}_{i,j}$ from a combination of two nodes $\mathsf{n}_i, \mathsf{n}_j$, 
the following criteria have to be met:
\begin{enumerate}[label=\textit{(RD\arabic*)}, wide, labelwidth=!, labelindent=0pt, noitemsep]
	\item \label{enum:dup:RD1} The first $|\mathsf{n}_i|$ elements of $\mathsf{n}_j$ interpreted as a set are equal to the elements of $\mathsf{n}_i$ interpeted as a set; there are no conditions on the conflict labels $\mathsf{n}_i.\mathsf{cs}$ and $\mathsf{n}_j.\mathsf{cs}$ of the combined nodes.
	\item \label{enum:dup:RD2} The reconstructed node $\mathsf{n}_{i,j}$ is built by setting\footnote{Notation: Given a list $\mathsf{n}$,
	$\mathsf{n}[k..l]$ refers to the sublist including all elements from the $k$-th (included) until the $l$-th (included).} $\mathsf{n}_{i,j}[1..|\mathsf{n}_i|] = \mathsf{n}_i$ and $\mathsf{n}_{i,j}[|\mathsf{n}_i|+1..|\mathsf{n}_{j}|] = \mathsf{n}_j[|\mathsf{n}_i|+1..|\mathsf{n}_{j}|]$ as well as $\mathsf{n}_{i,j}.\mathsf{cs}[1..|\mathsf{n}_i|] = \mathsf{n}_i.\mathsf{cs}$ and $\mathsf{n}_{i,j}.\mathsf{cs}[|\mathsf{n}_i|+1..|\mathsf{n}_{j}|] = \mathsf{n}_j.\mathsf{cs}[|\mathsf{n}_i|+1..|\mathsf{n}_{j}|]$, i.e., the first part of $\mathsf{n}_{i,j}$ and $\mathsf{n}_{i,j}.\mathsf{cs}$, respectively, is equal to $\mathsf{n}_{i}$ and $\mathsf{n}_{i}.\mathsf{cs}$, to which the last part of $\mathsf{n}_{j}$ and $\mathsf{n}_{j}.\mathsf{cs}$ is appended.
	\item \label{enum:dup:RD3} Either 
	\begin{enumerate*}[label=\textit{(\alph*)}]
	\item \label{enum:dup:RD3:a}
	$\mathsf{n}_i, \mathsf{n}_j$ are each (reconstructed\footnote{Note the recursive character of this definition. That is, all combinations of explicit and already reconstructed nodes are possible. E.g., a reconstructed node $\mathsf{n}_6$ can be the result of combining two explicit nodes $\mathsf{n}_1$ and $\mathsf{n}_2$ to reconstruct a node $\mathsf{n}_3$, which in turn is combined with some explicit node $\mathsf{n}_4$, which in turn is combined with a reconstructed node $\mathsf{n}_5$.} or explicit) nodes from $\Queue_{dup}$, or 
	\item \label{enum:dup:RD3:b}
	node $\mathsf{n}_i$ is from the node combination closure $\Queue^*_{dup}$ of $\Queue_{dup}$ which is the union of $\Queue_{dup}$ with the set of all nodes reconstructible\footnote{Formally, $\Queue^*_{dup}$ can be defined as the fixpoint $S^*$ of the sequence of sets $S_0, S_1,\dots$ resulting from the iterative application of the $\mathit{Comb}$ function starting from $S_0 := \Queue_{dup}$ where $S_{i+1} = \mathit{Comb}(S_i)$ and $\mathit{Comb}$ is defined as $\mathit{Comb}(S) = S \cup \{\mathsf{n}_{i,j}\,|\,\mathsf{n}_{i,j} \text{ is the result as per \ref{enum:dup:RD2} of combining two nodes } \mathsf{n}_i,\mathsf{n}_j \in S \text{ which meet \ref{enum:dup:RD1}}\}$.} 
	through 
	\ref{enum:dup:RD3:a},
	and node $\mathsf{n}_j$ is from a node collection including active nodes, i.e., from one of $\Queue$, $\mD_{\supset}$, $\mD_{\checkmark}$, $\mD_{\times}$, or $\mD_{calc}$.
	\end{enumerate*}
\end{enumerate}

In the example above, criterion 
\ref{enum:dup:RD3}\ref{enum:dup:RD3:a}
is met, where $\mathsf{n}_1$ and $\mathsf{n}_2$ correspond to $\mathsf{n}_i$ and $\mathsf{n}_j$, respectively, which are both (explicit) elements of $\Queue_{dup}$; criterion 
\ref{enum:dup:RD1}
is satisfied as well because the first $|\mathsf{n}_i| = |\mathsf{n}_1| = 2$ elements of $\mathsf{n}_j = \mathsf{n}_2$ correspond to the set $\{2,3\}$, which is equal to the set of elements of $\mathsf{n}_i = \mathsf{n}_1$; the validity of criterion 
\ref{enum:dup:RD2} 
can be easily verified by comparing $\mathsf{n}_{1,2}$ with $\mathsf{n}_1$ and $\mathsf{n}_2$.
Let us consider some important remarks:
\begin{enumerate}[noitemsep]
	\item \emph{Node reconstruction 
	is sound and complete:} 
	The set of all nodes constructible by means of
	\ref{enum:dup:RD1}, \ref{enum:dup:RD2} and \ref{enum:dup:RD3}\ref{enum:dup:RD3:b}
	is exactly the set of all duplicates of the currently active nodes in $\Queue\cup\mD_{\supset}\cup\mD_{\checkmark}\cup\mD_{\times}\cup\mD_{calc}$.
	\item \label{rem:dup:relationship_between_reconstructed_and_combined_nodes}\emph{Relationship between reconstructed node and combined source nodes:} 
	\ref{enum:dup:RD1}
	implies that $|\mathsf{n}_i| \leq |\mathsf{n}_j|$. By 
	\ref{enum:dup:RD2}
	node reconstruction means that the node $\mathsf{n}_i$ of lower (or equal) length replaces the first part of (or the complete) node $\mathsf{n}_j$; we can thus call $\mathsf{n}_i$ the \emph{modifying node} and $\mathsf{n}_j$ the \emph{modified node}. Moreover, the reconstructed node has the same length as and is \emph{set}-equal (wrt.\ \emph{edge} labels) to node $\mathsf{n}_j$. 
	As a consequence of this, node reconstructions can never lead to nodes that are new in terms of their sets of edge labels. Hence, as far as \emph{sets} of edge labels of nodes are concerned, $\Queue_{dup}$ is representative of $\Queue^*_{dup}$.
	\item \emph{Reconstruction of nodes only on demand:} Neither $\Queue^*_{dup}$ 
	(as per \ref{enum:dup:RD3}\ref{enum:dup:RD3:a})
	nor the set of all duplicates of active nodes 
	(as per \ref{enum:dup:RD3}\ref{enum:dup:RD3:b})
	is ever exlicitly generated by DynamicHS. Instead, only a minimal number of node reconstructions necessary for the proper-functioning (completeness) of DynamicHS are performed. More specifically, node reconstructions can only take place in case one node has been pruned and a 
	replacement node for it is sought (cf.\ \textsc{prune} function, Sec.~\ref{sec:algo_walkthrough}). 
	And, for each pruned node, either just one replacement node is reconstructed, or none at all if no suitable replacement node exists. 
	\item \emph{Principle of node reconstruction in the course of tree pruning:} Assume the \textsc{prune} function is called given the minimal conflict $X$ and finds some redundant node $\mathsf{nd}$, i.e., $X$ is a witness of redundancy for $\mathsf{nd}$ (cf.\ Sec.~\ref{sec:algo_walkthrough}). Since there might be multiple edge and conflict labels in $\mathsf{nd}$ and $\mathsf{nd.cs}$ 
	due to which $\mathsf{nd}$ is redundant given $X$, let $k$ be the maximal index such that $X \subset \mathsf{nd.cs}[k]$ and $\mathsf{nd}[k] \in \mathsf{nd.cs}[k] \setminus X$ (redundancy criterion, cf.\ page~\pageref{def:redundancy}).
	After deleting $\mathsf{nd}$, a replacement node for it is sought. 
	A replacement node $\mathsf{nd}'$ of $\mathsf{nd}$ needs to be \emph{(i)}~non-redundant (as per the current knowledge, i.e., $X$ must not be a witness of redundancy for $\mathsf{nd}'$) and \emph{(ii)}~set-equal to $\mathsf{nd}$ (cf.\ Sec.~\ref{sec:algo_walkthrough}). 
	
	Due to (ii) and
	Remark~\ref{rem:dup:relationship_between_reconstructed_and_combined_nodes} above, 
	the redundant node $\mathsf{nd}$ can be interpreted as $\mathsf{n}_j$ and the sought replacement node as $\mathsf{n}_{i,j}$. 
	With that said, the task of finding a replacement node is equivalent to finding a non-redundant (as per $X$) node $\mathsf{n}_i$ in $\Queue^*_{dup}$, see 
	\ref{enum:dup:RD3},
	such that $|\mathsf{n}_i| \geq k$ (i.e., at least the redundant part of $\mathsf{n}_j = \mathsf{nd}$ is replaced by $\mathsf{n}_i$), see \ref{enum:dup:RD2}, and the set of elements of $\mathsf{n}_i$ is equal to the set of the first $|\mathsf{n}_i|$ elements of $\mathsf{n}_j = \mathsf{nd}$, see 
	\ref{enum:dup:RD1}.
	Since $\Queue_{dup}$ is representative of $\Queue^*_{dup}$ in terms of the \emph{sets} of edge labels of nodes (see Remark~\ref{rem:dup:relationship_between_reconstructed_and_combined_nodes}) and because  $\mathsf{n}_i$ must only be suitable in terms of \emph{set}-equality, it is sufficient to search for $\mathsf{n}_i$ in $\Queue_{dup}$ as opposed to $\Queue^*_{dup}$. 
	
	Due to (i) and since $\mathsf{n}_i$ is a node from $\Queue_{dup}$, it must be provided that each node from $\Queue_{dup}$ that qualifies as $\mathsf{n}_i$ in the search for a replacement node is non-redundant. This imposes two requirements, as pointed out in Sec.~\ref{sec:algo_walkthrough}: $\Queue_{dup}$ must be pruned previous to all other node collections (to account for case 
	\ref{enum:dup:RD3}\ref{enum:dup:RD3:b}),
	and nodes of $\Queue_{dup}$ must be pruned in ascending order of their length (to account for case 
	\ref{enum:dup:RD3}\ref{enum:dup:RD3:a}\footnote{Recall from Remark~\ref{rem:dup:relationship_between_reconstructed_and_combined_nodes} that $|\mathsf{n}_i| \leq |\mathsf{n}_j|$. If $|\mathsf{n}_i| < |\mathsf{n}_j|$, then a processing of nodes in order of ascending node cardinality guarantees that 
	all still available (non-pruned) nodes $\mathsf{n}_i$ are already verified non-redundant when some $\mathsf{n}_j$ might need to be replaced. If, on the other hand, $|\mathsf{n}_i| = |\mathsf{n}_j|$, then there are two cases: $\mathsf{n}_i$ is processed prior to $\mathsf{n}_j$, or the opposite holds. In the former case, $\mathsf{n}_i$ (unless pruned) must already be verified non-redundant 
	when $\mathsf{n}_j$ is considered. In the latter case, $\mathsf{n}_i$ is not available as a replacement node at the time $\mathsf{n}_j$ is addressed, but $\mathsf{n}_i$ itself will be processed later. Thus, if $\mathsf{n}_j$ is pruned and $\mathsf{n}_i$ non-redundant, then the latter will remain in $\Queue_{dup}$ which 
	means that $\mathsf{n}_i$ has essentially replaced $\mathsf{n}_j$.}). 
\end{enumerate}

\setlength{\fboxsep}{1pt}
\begin{figure}[!tbp]
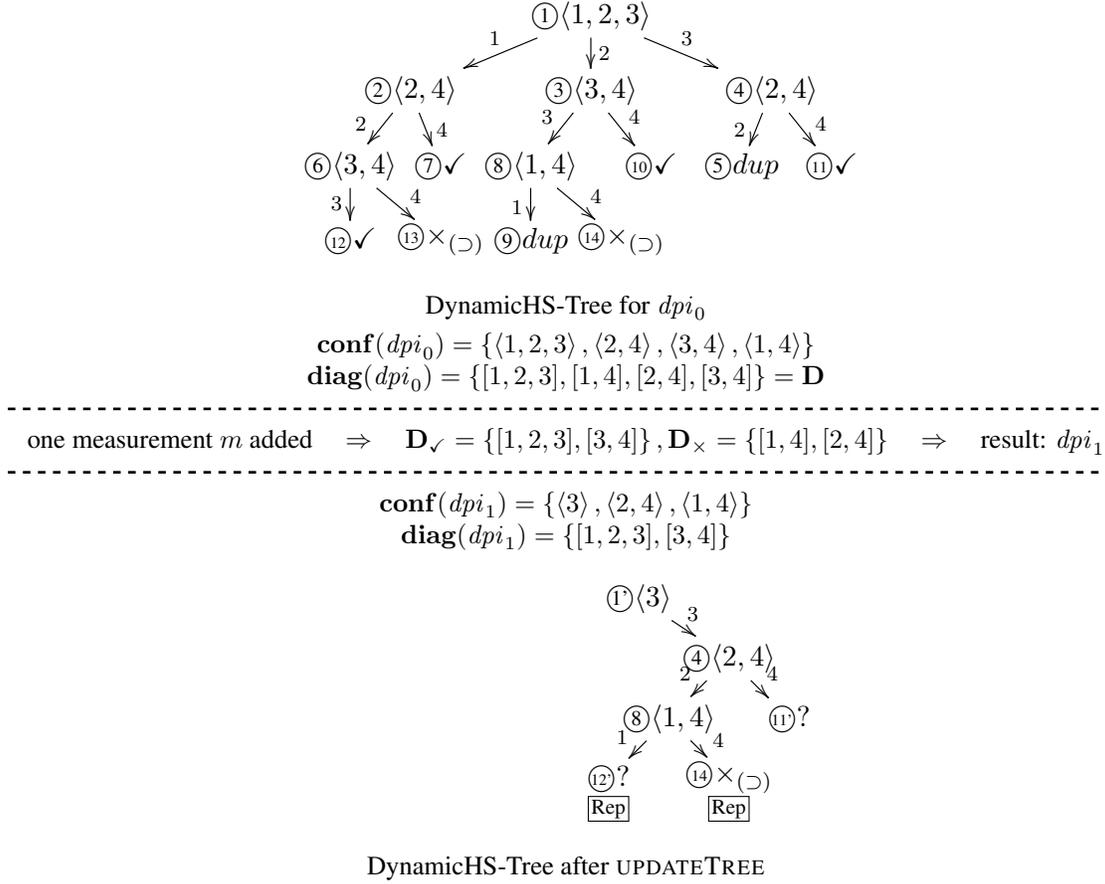

	\centering
	\[\xygraph{
		!{<0cm,0cm>;<0.8cm,0cm>:<0cm,-1cm>::}
		!{(-8,0) }*+{ \phantom{....} }="n1"
		!{(-3,0) }*+{ \textcircled{\scriptsize 1}\langle1,2,3\rangle }="n1"
		!{(-6,1) }*+{ \textcircled{\scriptsize 2}\langle2,4\rangle }="n2"
		"n1":"n2"_{1}
		!{(-3,1) }*+{ \textcircled{\scriptsize 3}\langle3,4\rangle }="n3"
		"n1":"n3"^{2}
		!{(0,1) }*+{ \textcircled{\scriptsize 4}\langle2,4\rangle }="n03"
		"n1":"n03"^{3}
		!{(-7,2) }*+{ \textcircled{\scriptsize 6}\langle3,4\rangle }="n4"
		"n2":"n4"_(0.6){2}
		!{(-5.5,2) }*+{ \textcircled{\scriptsize 7}\checkmark }="n04"
		"n2":"n04"^(0.6){4}
		!{(-4,2) }*+{ \textcircled{\scriptsize 8}\langle1,4\rangle }="n5"
		"n3":"n5"_{3}
		!{(-2,2) }*+{ \textcircled{\tiny 10}\checkmark }="n6"
		"n3":"n6"^{4}
		!{(-0.5,2) }*+{ \textcircled{\scriptsize 5} dup }="n7"
		"n03":"n7"_(0.6){2}
		!{(1,2) }*+{ \textcircled{\tiny 11}\checkmark }="n8"
		"n03":"n8"^(0.6){4}
		!{(-7,3) }*+{ \textcircled{\tiny 12} \checkmark }="n10"
		"n4":"n10"_{3}
		!{(-5.5,3) }*+{ \textcircled{\tiny 13} \XSup{} }="n11"
		"n4":"n11"^(0.6){4}
		!{(-4,3) }*+{ \textcircled{\scriptsize 9} dup }="n12"
		"n5":"n12"_(0.55){1}
		!{(-2.5,3) }*+{ \textcircled{\tiny 14} \XSup{} }="n13"
		"n5":"n13"^(0.6){4}
	}\]
	\begin{center}
		\small 
		DynamicHS-Tree for $\dpi_0$ \\ \vspace{3pt}
		$ \conf(\dpi_0) = \setof{\tuple{1,2,3},\tuple{2,4},\tuple{3,4},\tuple{1,4}}$ \\
		$ \sol(\dpi_0) = \setof{[1,2,3],[1,4],[2,4],[3,4]} = \mD$ \\
		\hdashrule[0.5ex]{0.99\columnwidth}{1pt}{1mm}  \\
		one measurement $m$ added 
		$\quad\Rightarrow\quad$ $\mD_{\checkmark} = \setof{[1,2,3],[3,4]}, \mD_{\times} = \setof{[1,4],[2,4]}$
		$\quad\Rightarrow\quad$ result: $\dpi_1$ \\
		\hdashrule[0.5ex]{0.99\columnwidth}{1pt}{1mm} 
		\\
		$ \conf(\dpi_1) = \setof{\tuple{3},\tuple{2,4},\tuple{1,4}}$\\
		$ \sol(\dpi_1) = \setof{[1,2,3],[3,4]}$
	\end{center}
	\[\xygraph{
		!{<0cm,0cm>;<0.8cm,0cm>:<0cm,-0.8cm>::}
		!{(-8,0) }*+{ \phantom{....} }="n1"
		!{(-3,0) }*+{ \textcircled{\scriptsize 1'}\langle3\rangle }="n1"
		!{(-1.5,1) }*+{ \textcircled{\scriptsize 4}\langle2,4\rangle }="n03"
		"n1":"n03"^{3}
		!{(-2.5,2) }*+{ \textcircled{\scriptsize 8}\langle1,4\rangle }="n5"
		"n03":"n5"_{2}
		!{(-0.5,2) }*+{ \textcircled{\tiny 11'} ? }="n6"
		"n03":"n6"^{4}
		!{(-3.5,3) }*+{ \textcircled{\tiny 12'} ? }="n12"
		"n5":"n12"_(0.55){1}
		!{(-1.5,3) }*+{ 
			\textcircled{\tiny 14} \XSup{} }="n13"
		"n5":"n13"^(0.6){4}
		!{(-1.5,3.5) }*+{ \framebox{\scriptsize Rep}}="nx1"
		!{(-3.5,3.5) }*+{ \framebox{\scriptsize Rep}}="nx2"
	}\]
	\begin{center}
		\small 
		DynamicHS-Tree after \textsc{updateTree} 
	\end{center}
	\caption{Tree pruning and redundancy checking example.}
	\label{fig:tree_pruning_example_HStree}
\end{figure}

\begin{example}
\label{ex:advanced_techniques}
We now showcase the workings of DynamicHS's tree update on a simple example, thereby also illustrating the discussed advanced techniques regarding the \emph{efficient redundancy checking} (Sec.~\ref{sec:efficient_redundancy_checking}), the \emph{avoidance of expensive reasoning} (Sec.~\ref{sec:avoidance_of_exp_reasoning}), as well as the \emph{space-saving duplicate storage and on-demand reconstruction} (Sec.~\ref{sec:space-saving_dup_storage}). Note that we already discussed the \emph{lazy updating policy} (Sec.~\ref{sec:lazy_updating_policy}) in terms of Example~\ref{ex:algo_description}.

Consider Fig.~\ref{fig:tree_pruning_example_HStree} (with a similar notation as used in Figs.~\ref{fig:dynhs_example} and \ref{fig:hs_example}) which depicts the hitting set tree produced by DynamicHS (iteration~1) for some DPI $\dpi_0$ in breadth-first order (see the node numbers \textcircled{\scriptsize t} signalizing that the respective node was generated at point in time $t$).
The sets of minimal conflicts and minimal diagnoses for $\dpi_0$ are given by $\conf(\dpi_0)$ and $\sol(\dpi_0)$ in the figure. 
We assume that DynamicHS uses the parameter $\ld := 5$, i.e., five minimal diagnoses (if existent) should be computed. Since there are only four minimal diagnoses (see $\sol(\dpi_0)$), DynamicHS executes until the queue is empty ($\Queue = [\,]$, see line~\ref{algoline:dyn:while} in Alg.~\ref{algo:dynamic_hs}), i.e., until the hitting set tree is complete. The resulting set of leading diagnoses $\mD$ corresponds to $\sol(\dpi_0)$ (nodes labeled by $\checkmark$ in Fig.~\ref{fig:tree_pruning_example_HStree}). Further, we assume that a (discriminating) measurement $m$ is added to $\dpi_0$, which leads to the new minimal conflict $\tuple{3}$ for the resulting DPI $\dpi_1$. As can be seen through a comparison of $\sol(\dpi_0)$ with $\sol(\dpi_1)$ in Fig,~\ref{fig:tree_pruning_example_HStree}, the leading diagnoses eliminated by the measurement $m$ are the nodes numbered \textcircled{\scriptsize 7} (corresponding to the node $[1,4]$) and \textcircled{\tiny 10} (node $[2,4]$). These two nodes are included in the set $\mD_{\times}$ given as an input argument to the second call of DynamicHS (iteration~2) in line~\ref{algoline:inter_onto_debug:call_DynHS} of Alg.~\ref{algo:sequential_diagnosis}. 

Now, when \textsc{updateTree} is invoked in iteration~2, the first step is the examination of the elements in $\mD_{\times}$ regarding their redundancy status (see description of the \textsc{updateTree} function in Sec.~\ref{sec:algo_walkthrough}). For node $\mathsf{nd}=[1,4]$, we have $\mathsf{nd.cs} = [\tuple{1,2,3},\tuple{2,4}]$. The QRC (see Sec.~\ref{sec:efficient_redundancy_checking}) executed on $\mathsf{nd}$ involves calling \textsc{findMinConflict} with argument\footnote{Note that we just mention the first element $\mo$ of the tuple $\tuple{\mo,\mb,\Tp,\Tn}$ passed to \textsc{findConflict} and write ``$\dots$'' for the remaining ones for simplicity and brevity. The reason is that we did not discuss the specific DPI underlying this example and that $\mo$ (the set from which a minimal conflict is to be computed) is sufficient to understand the discussed points.} $(\tuple{U_{\mathsf{nd.cs}}\setminus\mathsf{nd},\dots})$ which is equal to $(\tuple{\setof{1,2,3,4}\setminus\setof{1,4},\dots}) = (\tuple{\setof{2,3},\dots})$. Hence, the conflict $X=\tuple{3}$ is returned (cf.\ $\conf(\dpi_1)$), which is a subset of $\mathsf{nd.cs}[1]$ and thus must constitute a witness of redundancy of $\mathsf{nd} = [1,4]$.

As a next step, \textsc{prune} is called with argument $X$ (line~\ref{algoline:update:prune} in Alg.~\ref{algo:dynamic_hs}). At first, \textsc{prune} considers $\Queue_{dup} = [[3,2],[2,3,1]]$ (nodes numbered \textcircled{\scriptsize 5} and \textcircled{\scriptsize 9} in Fig.~\ref{fig:tree_pruning_example_HStree}) and cleans it up from 
redundant nodes.
At this, for each node $\mathsf{nd} \in \Queue_{dup}$, the algorithm runs through the conflicts $\mathsf{nd.cs}[i]$ from small to large $i$ and, if $X \subset \mathsf{nd.cs}[i]$,  \emph{(i)}~checks if $\mathsf{nd}[i] \in \mathsf{nd.cs}[i] \setminus X$ (is $X$ a witness of redundancy for $\mathsf{nd}$?) as well as \emph{(ii)}~replaces $\mathsf{nd.cs}[i]$ by $X$ (update of internal node labels). At the end of this traversal, the maximal index $i = k$, 
for which (i) is executed and true, is stored. 
Since, importantly, nodes are processed in ascending order of their size, the first processed node from $\Queue_{dup}$ in this concrete example is $\mathsf{nd} = [3,2]$ with $\mathsf{nd.cs} = [\tuple{1,2,3},\tuple{2,4}]$. For index $i=1$, we have $\tuple{3} \subset \tuple{1,2,3}$, which is why $\mathsf{nd.cs}[1]$ is replaced by $\tuple{3}$ in the course of step (ii). However, since $\mathsf{nd}[1] = 3$ is not an element of $\tuple{1,2,3}\setminus\tuple{3}$, check (i) is negative (no redundancy detected). 
For index $i=2$, it does not even hold that $\tuple{3} \subset \tuple{2,4}$, hence neither (i) nor (ii) are executed. The overall conclusion from this analysis is that $X$ is not a witness of redundancy for $\mathsf{nd}$. Consequently, there is no evidence up to this point that $\mathsf{nd}$ is redundant, which is why $\mathsf{nd}$ remains an element of $\Queue_{dup}$, however, with the modified conflict labels set $\mathsf{nd.cs} = [\tuple{3},\tuple{2,4}]$.

For the second node $[2,3,1]$ in $\Queue_{dup}$, a similar evaluation leads to the insight that this node is redundant and $k = 1$ (because $\tuple{3} \subset \tuple{1,2,3}$ and $2 \in \tuple{1,2,3}\setminus\tuple{3}$). However, instead of only discarding the node, the algorithm seeks
a replacement node that is non-redundant and \emph{set}-equal to $[2,3,1]$. To this end, it iterates through all already processed (and thus provenly non-redundant) nodes in $\Queue_{dup}$ (in this case only the node $[3,2]$) and tries to find some node of size $l$ which is set-equal to 
the first $l$ elements of the redundant node, for some $l \geq k$.
%

Indeed, since the first $2 \geq k = 1$ elements of nodes $[3,2]$ and $[2,3,1]$ are equal (when considered as a set), a replacement node for the latter can be constructed. This (re)constructed node is given by $\mathsf{ndnew} = [3,2,1]$ with $\mathsf{ndnew.cs} = [\tuple{3},\tuple{2,4},\tuple{1,4}]$ (i.e., figuratively spoken, the label $\mathit{dup}$ at node  \textcircled{\scriptsize 5} is replaced by the path including nodes \textcircled{\scriptsize 8} and \textcircled{\scriptsize 9}, cf.\ Fig.~\ref{fig:tree_pruning_example_HStree}).

Since all nodes of $\Queue_{dup}$ have at this point been processed, the \textsc{prune} method shifts its focus to the other collections $\Queue$, $\mD_{\supset}$, $\mD_{\times}$ and $\mD_{\checkmark}$. Essentially, the considerations made for these collections are equal to those explicated for $\Queue_{dup}$, with the only difference that solely elements of (the already cleaned up) $\Queue_{dup}$ are in line for being used
in the construction of replacement nodes
for redundant ones found in these collections. For instance, the node $\mathsf{n} = [1,2,3]$ with the number \textcircled{\tiny 12} is detected to be redundant ($k = 1$) when it comes to pruning $\mD_{\checkmark}$. Since, however, the first $3 \geq k = 1$ elements of $\mathsf{ndnew} = [3,2,1]$ are set-equal to the first $3$ elements of $\mathsf{n}$, the latter is replaced by $\mathsf{ndnew}$. 
After carrying out all pruning actions provoked by the witness of redundancy $X = \tuple{3}$ (detected by analysing the first node $[1,4]\in\mD_{\times}$), the reconstructed replacement node $\mathsf{ndnew}$ is now the only remaining node in the tree that corresponds to the set of edge labels $\setof{1,2,3}$. The fact that this set does constitute a minimal diagnosis wrt.\ $\dpi_1$ (cf.\ $\sol(\dpi_1)$) corroborates that the storage and adequate reconstruction of duplicates is pivotal for the completeness of DynamicHS. 
%

Note that the second node $[2,4]$ that was originally an element of $\mD_{\times}$ has meanwhile already been removed from $\mD_{\times}$ in the course of the pruning actions taken. The reason is that the conflict $X=\tuple{3}$ that was used as a basis for pruning is also a witness of redundancy for $[1,4]$. In fact, $\mD_{\times}=\emptyset$ holds after conflict $X=\tuple{3}$ has been processed. Hence, the for-loop in line~\ref{algoline:update:process_Dtimes_start} of Alg.~\ref{algo:dynamic_hs} terminates and no further pruning operations are conducted.
%

The pruned tree resulting from the execution of \textsc{updateTree} is shown at the bottom of Fig.~\ref{fig:tree_pruning_example_HStree}. Replacement nodes (i.e., those nodes which substitute deleted redundant nodes), are marked by \framebox{\scriptsize Rep}; their node number \textcircled{\scriptsize k} in the pruned tree is the number of the original deleted redundant node. 
Note that the second replacement node $[3,2,4]$ results from (the deleted) node \textcircled{\tiny 14} by a substitution of its first two edges $[2,3]$ by the duplicate $[3,2]$ (node \textcircled{\scriptsize 5}).
In addition, node labels changed during the pruning process are indicated by a prime (') symbol. In this concrete example, e.g., both relabeled nodes \textcircled{\tiny 11} and \textcircled{\tiny 12} originally represented minimal diagnoses for $\dpi_0$ and were returned by the first run of DynamicHS in terms of $\mD_{calc}$. Then, they were added to $\mD_{\checkmark}$ (line~\ref{algoline:inter_onto_debug:partition_D_into_Dcheckmark_and_Dtimes} in Alg.~\ref{algo:sequential_diagnosis}) since they are consistent with the added measurement $m$. Finally, at the end of \textsc{updateTree}, 
node \textcircled{\tiny 11} and the replacement node of node \textcircled{\tiny 12} were reinserted into the queue $\Queue$ of unlabeled nodes (preservation of best-first property; line~\ref{algoline:update:insert_sorted_1} in Alg.~\ref{algo:dynamic_hs}) because they ``survived'' the pruning.

As to the complexity of the tree pruning performed by DynamicHS in this example, 
the overall number of reasoner-invoking function calls amounts to merely a single ``easy'' call of \textsc{findMinConflict} (the executed QRC), whereas the (re)construction of a tree equivalent to DynamicHS's pruned tree,
carried out in iteration~2 when adopting (the stateless) HS-Tree, requires three ``hard'' \textsc{findMinConflict} calls 
(computations of 
conflicts $\tuple{3}$, $\tuple{2,4}$, $\tuple{1,4}$).\qed
\end{example}

\subsection{Correctness of DynamicHS}
\label{sec:correctness_proof}
\begin{theorem}\label{thm:dynHS_correctness}
	Let \textsc{findMinConflict} be a sound and complete method for conflict computation, i.e., given a DPI, it outputs a minimal conflict for this DPI if a minimal conflict exists, and 'no conflict' otherwise. Then, 
	DynamicHS (Alg.~\ref{algo:dynamic_hs}) is a sound, complete and best-first minimal diagnosis computation method. That is, given a DPI, DynamicHS computes all (completeness) and only (soundness) minimal diagnoses for this DPI in descending order (best-first property) of their probability as per a probability measure $\pr$ given as an input.
\end{theorem}
\begin{proof}\emph{(Sketch)}
	A fully detailed proof of this theorem can be found in \cite[Sec.~12.4]{Rodler2015phd}. Here, we explicate the main proof ideas.\vspace{5pt}
	

	\noindent\emph{Proof of Completeness and of the Best-First Property:} 
	We have to show that DynamicHS, given the parameter $\ld := k$, 
	outputs the $k$ best minimal diagnoses (according to $\pr$). First, we make two general observations, and then we prove the completeness by contradiction.

\emph{(Observation 1):} Only such nodes can be deleted by \textsc{prune} which are provably redundant (i.e., irrelevant), and, whenever existent, a suitable replacement node is extracted from $\Queue_{dup}$ (which stores all possible replacement nodes) to replace the deleted node (cf.\ Sec.~\ref{sec:algo_walkthrough}). 

\emph{(Observation 2):} During the execution of DynamicHS given the DPI $\dpi_k$, only diagnoses for $\dpi_k$ can be added to $\mD_{calc}$.

\emph{(Proof by Contradiction):} Assume that DynamicHS returns a set $\mD_{calc}$ with $|\mD_{calc}|=k$ and one of the $k$ best (according to $\pr$) minimal diagnoses is not in $\mD_{calc}$ (i.e., has not been computed by DynamicHS). We denote this non-found diagnosis by $\md'$. Since $\md'$ is one of the $k$ best minimal diagnoses, we have that some of the returned $k$ diagnoses in $\mD_{calc}$ must have a lower probability as per $\pr$ than $\md'$; let us call this diagnosis $\md''$. Below, we will show that, for any minimal diagnosis $\md$ for the relevant DPI  $\langle\mo,\mb,\Tp\cup\Tp',\Tn\cup\Tn'\rangle$
considered by DynamicHS, the following invariant (\emph{INV}) holds during the execution of the main while-loop (between lines~\ref{algoline:dyn:while}--\ref{algoline:dyn:add_to_Queue}) of DynamicHS: \emph{$\md \in \mD_{calc}$ or there is some node $\mathsf{n} \subseteq \md$ in $\Queue$.}
Due to the properties of $\pr$ (cf.\ Sec.~\ref{sec:Reiters_HS-Tree}) and because $\md'$ has a greater probability than $\md''$, each 
subset of $\md'$ has a greater probability than $\md''$. Since $\md' \not\in\mD_{calc}$ by assumption, we infer that $\mathsf{n}\in\Queue$ for some node $\mathsf{n}\subseteq \md'$ when DynamicHS terminates. This is a contradiction to $\md'' \in \mD_{calc}$ due to the sorting of $\Queue$ in descending order of probability and the fact that only elements of $\Queue$ can be added to $\mD_{calc}$ in DynamicHS.

We now demonstrate that the invariant INV holds, by induction over the number $n$ of times DynamicHS has already been called in Alg.~\ref{algo:sequential_diagnosis}.

\emph{(Induction Base):} Assume $n=1$, i.e., DynamicHS is called for the first time in Alg.~\ref{algo:sequential_diagnosis}. This has three implications: \emph{(1)}~At the time the while-loop is entered, the empty node $[]$ (cf.\ line~\ref{algoline:inter_onto_debug:var_init_end} in Alg.~\ref{algo:sequential_diagnosis}), which is a subset of any minimal diagnosis, is in $\Queue$. Hence, INV holds from the outset. 
\emph{(2)}~$\mD_{\checkmark} = \emptyset$ (cf.\ line~\ref{algoline:inter_onto_debug:var_init_mD_checkmark+mD_times} in Alg.~\ref{algo:sequential_diagnosis} and note that \textsc{updateTree} does not modify $\mD_{\checkmark}$). As a consequence, for each node from $\Queue$ that is processed throughout the execution of the while-loop of DynamicHS, the \textsc{dLabel} function is called (because line~\ref{algoline:dyn:set_L_to_valid} cannot be reached).
\emph{(3)}~The \textsc{prune} function (line~\ref{algoline:dlabel:prune}) cannot be called during the execution of the while-loop (as there cannot be any non-minimal conflicts due to the soundness of \textsc{findMinConflict}).  

Now, assume an arbitrary minimal diagnosis $\md$ for the DPI $\dpi_0$ 
relevant in the first call of DynamicHS. We next show that INV remains true for $\md$ after any node $\mathsf{node}$ is processed within DynamicHS's while-loop. There are two possible cases: \emph{(a)}~$\mathsf{node} \not\subseteq \md$ and \emph{(b)}~$\mathsf{node} \subseteq \md$. 

Let us consider case (a) first. 
Since no pruning is possible as argued above, i.e., no nodes can be deleted from any node collection stored by DynamicHS except $\Queue$ (cf.\ line~\ref{algoline:dyn:get_first}), the following holds. The processing (and deletion from $\Queue$) of $\mathsf{node}$, which is in no subset-relationship with $\md$, \emph{(i)}~cannot effectuate an elimination 
of any other node from $\Queue$ (in particular, this holds for all nodes being equal to or a subset of $\md$), and 
\emph{(ii)}~cannot modify $\mD_{calc}$. Hence, since INV held before $\mathsf{node}$ was processed, INV must still hold thereafter in case (a).

Now, assume case (b). Here, we have again two cases: \emph{(b1)}~$\mathsf{node} \subset \md$ and \emph{(b2)}~$\mathsf{node} = \md$. Suppose (b1) first, i.e., let $\mathsf{node} \subset \md$ be processed. 
Because \textsc{dLabel} is called for any processed node by the argumentation above, we have that line~\ref{algoline:dlabel:return_valid} is the only place where nodes can be assigned the label $\mathit{valid}$. 
Hence, if some node is assigned the label $\mathit{valid}$, this means that \textsc{findMinConflict} in line~\ref{algoline:dlabel:qx_2} must have returned 'no conflict' for it, which is why this node is a diagnosis by the completeness of \textsc{findMinConflict}.
Consequently, $\mathsf{node}$ cannot be labeled $\mathit{valid}$ because it is a proper subset of a minimal diagnosis. 
Moreover, $\mathsf{node}$ cannot be labeled $\mathit{nonmin}$ in line~\ref{algoline:dlabel:non-min_crit_end} as there cannot be a subset of $\mathsf{node}$ in $\mD_{calc}$ due to Observation~2 and the fact that $\mathsf{node}$ is a proper subset of a minimal diagnosis.
%
As a result, the \textsc{dLabel} function must return in either of the lines~\ref{algoline:dlabel:return_C}, \ref{algoline:dlabel:return_X} or \ref{algoline:dlabel:return_new_cs}, in each of which cases a minimal conflict set is returned. This conflict $L$ is then used to generate a new node $\mathsf{node}_e$ for each $e \in L$ (lines~\ref{algoline:dyn:for_e_in_L}--\ref{algoline:dyn:add_cs_to_node.cs}), where $|\mathsf{node}_e| = |\mathsf{node}| + 1$ and, for some $e$, $\mathsf{node}_e \subseteq \md$ must hold (if the latter was not the case then $\md$ would not hit the minimal conflict $L$, which is a contradiction to $\md$ being a minimal diagnosis).
For each of these nodes $\mathsf{node}_e$, either a set-equal node is already in $\Queue$ or $\mathsf{node}_e$ is added to $\Queue$ (cf.\ lines~\ref{algoline:dyn:check_node_already_in_Q}--\ref{algoline:dyn:add_to_Queue} and note that no node set-equal to $\mathsf{node}_e$ can be in $\mD_{\supset}$ due to Observation~2 and lines~\ref{algoline:dlabel:non-min_crit_start}--\ref{algoline:dlabel:non-min_crit_end}). 
Hence, in case (b1), INV remains true after $\mathsf{node}$ has been processed.

Finally, assume case (b2). Because \textsc{dLabel} is called for any processed node by the argumentation above, it must be called for $\mathsf{node}$. Since $\mathsf{node}$ however is equal to the minimal diagnosis $\md$, $\textsc{dLabel}$ will return $\mathit{valid}$ (this follows from Observation~2 and the fact that diagnoses are hitting sets of all conflicts). Due to lines~\ref{algoline:dyn:if_L_valid}--\ref{algoline:dyn:add_to_Dcalc}, this means that $\md \in \mD_{calc}$ will hold at the beginning of the next iteration of the while-loop. Consequently, also in case (b2), INV still holds after the processing of $\mathsf{node}$. 
This completes the proof of the Induction Base.

\emph{(Induction Assumption):} Assume INV holds for all 
$n\leq k$. 

\emph{(Induction Step):} Now, let $n=k+1$. That is, we consider the $(k+1)$-th call of DynamicHS in Alg.~\ref{algo:sequential_diagnosis}. We assume again an arbitrary minimal diagnosis $\md$ for the DPI $\dpi_{k}$ 
relevant in this call of DynamicHS.\footnote{Please note that the original DPI considered by the first call of DynamicHS during a sequential diagnosis session is referred to as $\dpi_0$ (cf.\ line~\ref{algoline:inter_onto_debug:call_DynHS} in Alg.~\ref{algo:sequential_diagnosis}). Hence, the DPI relevant to the $k$-th call of DynamicHS is denoted by $\dpi_{k-1}$.} By the fact that each minimal diagnosis for $\dpi_{k}$ is either equal to or a superset of some diagnosis for $\dpi_{k-1}$ (Property~\ref{property:impact_of_DPI_transition}.\ref{property:impact_of_DPI_transition:enum:diags_can_only_grow}), and since for each minimal diagnosis $\md'$ for $\dpi_{k-1}$ either $\md'$ was in $\mD_{calc}$ or some node $\mathsf{n} \subseteq \md'$ was in $\Queue$ when the $k$-th call of DynamicHS returned (Induction Assumption), we infer that some node corresponding to a subset of $\md$ is either in one of $\mD_{\checkmark}$ or $\mD_{\times}$ (cf.\ line~\ref{algoline:inter_onto_debug:partition_D_into_Dcheckmark_and_Dtimes} in Alg.~\ref{algo:sequential_diagnosis} where $\mD_{calc}$ is split into $\mD_{\checkmark}$ and $\mD_{\times}$) or in $\Queue$ at the beginning of the $(k+1)$-th execution of DynamicHS. The first steps in this execution are setting $\mD_{calc} = \emptyset$ and calling the function \textsc{updateTree}. Throughout \textsc{updateTree}, some nodes might be pruned (and potentially replaced by set-equal nodes), and all non-pruned nodes from $\mD_{\times}$ as well as all nodes from $\mD_{\checkmark}$ are finally reinserted into $\Queue$. Moreover, each non-pruned node from $\mD_{\supset}$ for which there is no known diagnosis that is a subset of it is added to $\Queue$ at the end of \textsc{updateTree}. By 
Observation~1
and since $\md$ is a minimal diagnosis and thus relevant, we have that there must be a node $\mathsf{n} \subseteq \md$ in $\Queue$ when the while-loop of the $(k+1)$-th DynamicHS call is entered. That is, INV holds at the beginning of the while-loop. 

That INV remains true for $\md$ after any node $\mathsf{node}$ is processed within the while-loop, is shown analogously (i.e., same case analysis and argumentation) as expounded for the Induction Base,
except for two aspects: $\mD_{\checkmark} \neq \emptyset$ and the \textsc{prune} function (line~\ref{algoline:dlabel:prune}) might be called. 
Consequences of these aspects are:
\emph{(1)}~By Observation~1, during the execution of the while-loop of DynamicHS, the last remaining node in $\Queue$ which is a subset of some minimal diagnosis cannot be pruned without being replaced by a set-equal node. Neither can a minimal diagnosis be removed from $\mD_{calc}$ without being substituted by a set-equal node. Therefore, for every execution of \textsc{prune} (line~\ref{algoline:dlabel:prune}), if INV holds prior to it, INV holds after it finishes. \emph{(2)}~Assume $\mathsf{node} \subset \md$ and $\mathsf{node} \in \mD_{\checkmark}$. Due to Observation~2 and line~\ref{algoline:inter_onto_debug:partition_D_into_Dcheckmark_and_Dtimes} in Alg.~\ref{algo:sequential_diagnosis}, $\mD_{\checkmark}$ includes only diagnoses for the DPI $\dpi_{k-1}$ relevant to the preceding ($k$-th) call of DynamicHS in Alg.~\ref{algo:sequential_diagnosis}, and 
each diagnosis in $\mD_{\checkmark}$ during the $(k+1)$-th call of DynamicHS throughout Alg.~\ref{algo:sequential_diagnosis} is a diagnosis for the DPI $\dpi_{k}$. Hence, the assumptions $\mathsf{node} \in \mD_{\checkmark}$ and $\mathsf{node} \subset \md$ are in contradiction to our assumption that $\md$ is a minimal diagnosis for $\dpi_k$. Equivalently: $\mathsf{node} \subset \md$ implies $\mathsf{node} \notin \mD_{\checkmark}$.

The impact of (1) and (2) on the case analysis (cf.\ Induction Base) is as follows: 
The argumentation for the case where $\mathsf{node} \not\subseteq \md$ is processed is analogous to case (a) for the Induction Base. The proof for the case $\mathsf{node} \subset \md$ is equal to case (b1) for the Induction Base since 
\textsc{dLabel} must be called for $\mathsf{node}$ (due to $\mathsf{node} \notin \mD_{\checkmark}$). Finally, the case $\mathsf{node} = \md$ is treated as demonstrated in case (b2) in the Induction Base because, if $\mathsf{node} \in \mD_{\checkmark}$, then it is simply directly labeled 
$\mathit{valid}$
in line~\ref{algoline:dyn:set_L_to_valid} (no call of \textsc{dLabel})---hence, whether or not the \textsc{dLabel} function is called, $\md \in \mD_{calc}$ will hold after $\mathsf{node}$ having been processed. This completes the proof of the Induction Step, and thus the entire proof.\vspace{5pt}

\noindent\emph{Proof of Soundness:} We have to show that DynamicHS outputs only minimal diagnoses. That is, we need to demonstrate that every element in $\mD_{calc}$ satisfies the diagnosis property and the minimality property.
%
Since each call of DynamicHS in the course of Alg.~\ref{algo:sequential_diagnosis} outputs one set $\mD_{calc}$, we prove the soundness by induction over the number $n$ of times DynamicHS has already been called in Alg.~\ref{algo:sequential_diagnosis}.

\emph{(Induction Base):} Assume $n=1$, i.e., DynamicHS is called for the first time in Alg.~\ref{algo:sequential_diagnosis} and returns $\mD_{calc}$. Let $\mathsf{node} \in \mD_{calc}$. A node is added to $\mD_{calc}$ iff it has been labeled $\mathit{valid}$. There are two ways a node may be labeled $\mathit{valid}$, i.e., \emph{(i)}~in line~\ref{algoline:dyn:set_L_to_valid} and \emph{(ii)}~in line~\ref{algoline:dyn:dlabel}. Note that case (i) is impossible since $n=1$ which means that $\mD_{\checkmark} = \emptyset$ (cf.\ Alg.~\ref{algo:sequential_diagnosis}) and thus line~\ref{algoline:dyn:set_L_to_valid} can never be reached. Therefore, case (ii) applies to $\mathsf{node}$. That is, its label $\mathit{valid}$ is assigned by the \textsc{dLabel} function. Hence, \textsc{dLabel} must return in line~\ref{algoline:dlabel:return_valid}. From this we conclude that the \textsc{findMinConflict} call in line~\ref{algoline:dlabel:qx_2} returns 'no conflict' which implies that $\mathsf{node}$ is a diagnosis (due to the completeness of \textsc{findMinConflict}).
Assume there is a diagnosis $\md$ such that $\md \subset \mathsf{node}$. Since case (ii) is true for $\mathsf{node}$, there cannot be any such diagnosis $\md$ in $\mD_{calc}$ due to lines~\ref{algoline:dlabel:non-min_crit_start}--\ref{algoline:dlabel:non-min_crit_end}, because otherwise 
$\mathsf{node}$ would have been labeled $\mathit{nonmin}$
and line~\ref{algoline:dlabel:return_valid} could not have been reached. However, due to the completeness of DynamicHS, and since $\md$ must be ranked higher as per $\pr$ than $\mathsf{node}$ (cf.\ the definition of $\pr$ in Sec.~\ref{sec:Reiters_HS-Tree}), and since $\md$ is a diagnosis, $\md$ must already be included in $\mD_{calc}$ when $\mathsf{node}$ is added. This is a contradiction. Therefore, for $n=1$ (i.e., for the first call of DynamicHS in Alg.~\ref{algo:sequential_diagnosis}), the output $\mD_{calc}$ contains only minimal diagnoses.

\emph{(Induction Assumption):} Assume $\mD_{calc}$ contains only minimal diagnoses for $n\leq k$. 

\emph{(Induction Step):} Now, let $n=k+1$ and $\mathsf{node} \in \mD_{calc}$. Analogously to the argumentation above, we again have the two possibilities (i) and (ii) of how $\mathsf{node}$ might have attained its label $\mathit{valid}$. Suppose case (i) first. That is, $\mathsf{node} \in \mD_{\checkmark}$. By line~\ref{algoline:inter_onto_debug:partition_D_into_Dcheckmark_and_Dtimes} of Alg.~\ref{algo:sequential_diagnosis} (\textsc{assignDiagsOkNok}, cf.\ Sec.~\ref{sec:embedding_in_seq_diag_process}), $\mD_{\checkmark} \subseteq \mD_{calc}$ where $\mD_{calc}$ is the output of the the previous, i.e., the $k$-th, call of DynamicHS. Due to the Induction Assumption, we have that $\mD_{\checkmark}$ includes only minimal diagnoses for the DPI $\dpi_{k-1}$ considered in the $k$-th iteration, i.e., the DPI $\dpi_{k}$ considered in the $(k+1)$-th iteration without the most recently added measurement. However, \textsc{assignDiagsOkNok} adds to $\mD_{\checkmark}$ exactly those diagnoses that are consistent with the new measurement. Consequently, the diagnoses in $\mD_{\checkmark}$ are consistent with all measurements included in $\dpi_{k}$, and thus are diagnoses for $\dpi_{k}$. Due to Property~\ref{property:impact_of_DPI_transition}.\ref{property:impact_of_DPI_transition:enum:diags_can_only_grow}, no diagnosis for $\dpi_{k}$ can be a proper subset of any diagnosis for $\dpi_{k-1}$. Thus, all elements of $\mD_{\checkmark}$ must be minimal diagnoses which is why $\mathsf{node}$ must be a minimal diagnosis. 
For the other case (ii), the argumentation is exactly as for the Induction Base.
\end{proof}

\section{Related Work}
\label{sec:related_work}
Algorithms for diagnosis computation can be categorized (at least) according to the following dimensions: 
\begin{enumerate}
	\item \label{related_work:dimension_completeness}\emph{Incomplete} vs.\ \emph{complete}:\footnote{This dimension is sometimes also characterized as \emph{approximate} vs.\ \emph{exhaustive}, cf., e.g., \cite{jannach2016parallel}. However, note that ``approximation'' in this categorization refers to the approximation of the set of all (minimal) diagnoses achieved by the algorithm, and not, e.g., that computed sets are only approximate hitting sets of all conflicts, i.e., hit only a predefined minimal fraction of all conflicts \cite{vinterbo2000minimal}.} Complete algorithms
	guarantee to output all minimal diagnoses (given arbitrary time and memory). Examples of complete algorithms are HS-Tree \cite{Reiter87}, HS-DAG \cite{greiner1989correction}, GDE \cite{dekleer1987}, HST-Tree \cite{wotawa2001variant}, StaticHS \cite{rodler2018statichs}, Inv-HS-Tree \cite{Shchekotykhin2014}, BHS-Tree and Boolean algorithm \cite{lin2003computation}, HSSE-Tree \cite{xiangfu2006method}, SDE \cite{stern2012exploring}, RBF-HS and HBF-HS \cite{rodler2020_rbfhs,rodler2020dx_rbfhs}, the suite of algorithms presented in \cite{Rodler2015phd}, the parallel hitting set algorithms suggested by \cite{jannach2016parallel}, and a hypergraph inversion method proposed in \cite{haenni1998generating}.  
	Incomplete approaches, 
%
	in contrast, are usually geared towards computational efficiency, at the cost of not giving a completeness guarantee. 
	Examples of (generally) imcomplete algorithms are Genetic Algorithm \cite{li2002computing}, SAFARI \cite{feldman2008computing}, STACCATO \cite{abreu2009low}, CDA$^*$ \cite{williams2007conflict}, HDiag \cite{siddiqi2007hierarchical}, and NGDE \cite{dekleer2009mininimum}.  
	Also, incompleteness might be a consequence of a special focus of the diagnosis search, e.g., if the goal is to determine only (the) minimal cardinality (of) diagnoses \cite{siddiqi2007hierarchical,shi2010exact,dekleer2011hitting}. 
	In critical diagnosis applications, such as in medical applications \cite{schulz2010pitfalls,rector2011getting}, completeness is an important criterion as incompleteness might prevent the finding of the actual fault in the diagnosed system.\footnote{Note that soundness, i.e., that the algorithm outputs \emph{only} minimal diagnoses, is a very important property as well (and probably even more crucial than completeness). However, since soundness is usually a minimum presupposed property of diagnosis computation algorithms (and in fact satisfied by almost all algorithms in literature), we do not include a separate dimension that distinguishes between sound and unsound methods. 
	Nevertheless, it is noteworthy that there are two ways an algorithm may violate soundness: it may violate the diagnosis property or the minimality property. Usually, the former is a more severe issue (components that are not related in any way with the system's abnormality might be output), whereas the latter (unnecessarily many components, which however do explain the system's abnormality, may be output) can be fixed by performing a post-processing of the output diagnoses, 
	see, e.g., the $\mu$ function used to clean the output of BHS-Tree from non-minimal diagnoses \cite{haenni1998generating,lin2003computation}.}
	\item \emph{Best-first} vs.\ \emph{any-first}: Best-first approaches
	generate diagnoses in a specific order, usually guided by some preference criterion or heuristic, e.g., a probability measure or 
	minimum-cardinality-first.
	Among those, we further distinguish between \emph{general} best-first and \emph{focused} algorithms. The former, e.g., \cite{dekleer1987,Reiter87,greiner1989correction,dekleer1991focusing_prob_diag,Rodler2015phd,rodler2018statichs,rodler2020_rbfhs,rodler2020dx_rbfhs}, can output diagnoses in best-first order according to any (set-)monotonic function\footnote{A \emph{(set-)monotonic function} $f$ is one for which $f(X) \leq f(Y)$ whenever $X \subseteq Y$. Note that component fault probabilities satisfying the condition that each component's probability is below $0.5$ yield a diagnosis probability measure that is monotone (under the common framework for deriving diagnosis probabilities from component probabilities proposed in \cite{dekleer1987}). See \cite[Sec.~4.6.2]{Rodler2015phd} for more details.} that is used as a preference criterion, whereas the latter, e.g., \cite{darwiche2001decomposable,torasso2006model,siddiqi2007hierarchical,abreu2009low,dekleer2009mininimum,dekleer2011hitting}, consider only one particular preference criterion, usually 
	minimum-cardinality,
	or cannot 
	handle arbitrary (monotonic) preference functions.
	Any-first methods, e.g., \cite{lin2003computation,pill2012optimizations,Shchekotykhin2014}, on the other hand, do not ensure any particular order in which diagnoses are output, but still might attempt to generate preferred diagnoses first in a heuristic way \cite{rodler2020mbd_sampling}. One particular advantage of best-first methods is that they allow for a more reliable early termination\footnote{We refer by \emph{early termination} to the stopping of a sequential diagnosis session while there are \emph{multiple} minimal diagnoses left for the considered diagnosis problem, cf.\ \cite{dekleer1987}.} of a sequential diagnosis process (cf.\ Alg.~\ref{algo:sequential_diagnosis}). That is, if one of $k$ computed diagnoses has a very high probability compared to each of the other $k-1$ computed diagnoses, then it has a very high probability compared to each unknown (not yet computed) diagnosis as well (because it is guaranteed that all unknown diagnoses have a lower probability than the $k$ known ones). 
	%
	%
	A pro of any-first algorithms is that they
	are more flexible in terms of the search strategy pursued in the diagnosis search. For example, the Inv-HS-Tree method proposed by \cite{Shchekotykhin2014} does not need to stick to a (worst-case exponential-space) breadth-first or uniform-cost search (as comparable general best-first complete algorithms usually have to), but can use a more space-efficient depth-first search strategy. Hence, there are problems that can be solved by means of any-first methods, where similar best-first approaches run out of memory \cite{Shchekotykhin2014}. Unfortunately, any-first methods are not always appropriate and useful, e.g., for diagnosis tasks where (only) minimum-cardinality diagnoses are of interest.\footnote{Any-first methods would need to compute \emph{all} minimal diagnoses to find the minimum cardinality of diagnoses and thus to return a set of (proven) minimum-cardinality diagnoses.}
	%
	\item \label{enum:conflict-based_vs_direct} \emph{Conflict-based} vs.\ \emph{direct}: Conflict-based algorithms\footnote{
		These
	are sometimes referred to as \emph{conflict-to-diagnoses algorithms}, see, e.g., \cite{dekleer2011hitting}.}, e.g., \cite{dekleer1987,Reiter87,greiner1989correction,wotawa2001variant,lin2003computation,xiangfu2006method,dekleer2011hitting,stern2012exploring,Rodler2015phd,jannach2016parallel,rodler2018statichs,rodler2020_rbfhs,rodler2020dx_rbfhs},
	rely on the \mbox{(pre-)}computation of (a set of) conflicts and diagnoses are determined as hitting sets of all conflicts for a diagnosis problem. On the contrary, direct algorithms compute diagnoses without the indirection via conflicts. 
	Direct algorithms can be divided into approaches based on \emph{(i)~divide-and-conquer computation} \cite{felfernig2012efficient,marques2013computingMCSs,mencia2014relaxations,Shchekotykhin2014} 
	and 
	\emph{(ii)~compilation techniques} that translate the problem to a target language such as prime implicates \cite{dekleer1990prime_implicates}, decomposable negation normal form (DNNF) \cite{darwiche2001decomposable}, ordered binary decision diagrams (OBDD) \cite{torasso2006model} 
	or SAT \cite{metodi2014novel,marques2015MaxSATMBD}. 
	Methods of class (i), which are realized by means of (adapted) algorithms \cite{junker04,marques2013minimal} for solving the MSMP\footnote{MSMP refers to the problem of finding a \emph{minimal set over a monotone predicate} \cite{marques2013minimal,marques2017MSMP}. A \emph{monotone predicate} $p$ is a function that maps a set to a value in $\setof{0,1}$ with the property that $p(Y)=1$ whenever $p(X)=1$ and $Y \supseteq X$. Instances of MSMP problems are manifold, among them the problems of computing prime implicates \cite{marquis2000consequence}, minimal conflicts \cite{junker04}, minimal hitting sets \cite{felfernig2012efficient}, or minimal equally discriminating measurements for sequential diagnosis \cite{DBLP:journals/corr/Rodler2017}.} problem, are attractive, e.g., in diagnosis domains\footnote{One such domain is given by \emph{ontology alignment / matching} \cite{euzenat2011ontology} applications, where automated systems propose a set of (logical) correspondences between terms of two input knowledge bases describing similar domains, with the aim to integrate the expressed knowledge in a meaningful way. An example of such a correspondence would be that \emph{human} (term of the first ontology) is equivalent to \emph{person} (term of the second ontology). Since many correspondences are added to the union of both ontologies at once, it is often the case that a large set of conflicts arises \cite{meilicke2011thesis}.} 
	involving systems where a high number of components tend to be simultaneously faulty. In such a case, they mitigate memory issues arising for (complete) conflict-based algorithms by their ability to utilize space-efficient search techniques while preserving soundness of diagnosis computation \cite{shchekotykhin2012direct}. A drawback of these algorithms is their usual inability to enumerate diagnoses best-first as per a predefined preference order.
	An advantage of the problem-rewriting techniques of class (ii) is their ability to leverage highly optimized solving techniques from other domains (e.g., SAT), or the usually achieved polynomial complexity of (diagnostic) inference once the target representation has been compiled. On the downside, there is no guarantee that the representation is not exponential in size, and the published theories for compiling system descriptions mainly focus on propositional logic, which makes the (efficient) applicability of such techniques questionable for systems modeled in languages whose expressivity goes beyond propositional logic, e.g., ontology debugging problems \cite{kalyanpur2006thesis,Shchekotykhin2012}.

	Conflict-based techniques can be further categorized as follows:
	\begin{enumerate}
		\item \label{enum:on-the-fly_vs_preliminary}\emph{On-the-fly} vs.\ \emph{preliminary} (conflict computation): Since the precomputation of (all) minimal conflicts can be very expensive and is generally intractable\footnote{The computation of conflicts is dual to the computation of diagnoses \cite{stern2012exploring}, and hence it is NP-hard to decide if there is an additional minimal conflict given a set of minimal conflicts \cite{Bylander1991}.}, on-the-fly algorithms, e.g., \cite{Reiter87,greiner1989correction,wotawa2001variant,Rodler2015phd,jannach2016parallel,rodler2018statichs,rodler2020_rbfhs,rodler2020dx_rbfhs}, compute conflicts on demand in the diagnosis computation process. The main rationale behind this on-demand strategy is that diagnoses (hitting sets of all conflicts) can be computed even if not all conflicts are explicitly known\footnote{Note that to verify that a set is a diagnosis (i.e., hits all, possibly unknown, conflicts) a consistency check is sufficient. The hit conflicts thus do not need to be explicitly given or computed.}, and that conflict computation is usually (by far) the most expensive operation in the course of diagnosis computation (cf., e.g., \cite{pill2011eval_hitting_set_algos}).
		Especially in a sequential diagnosis setting (which is also the focus of this work), where a sample of (in principle only two \cite{Rodler2015phd}) computed diagnoses suffices and the diagnoses have the primary function to serve as a guideline and give evidence for proper measurement selection, normally not nearly all conflicts need to be precomputed to obtain this sample. Preliminary algorithms, on the other hand, assume the collection of conflicts to be given as an input to the diagnosis computation procedure. Instances of preliminary algorithms are \cite{li2002computing,lin2003computation,xiangfu2006method,abreu2009low,dekleer2011hitting,pill2012optimizations}. A way to view preliminary approaches is that they decouple the (model-based) reasoning, or, equivalently, the conflict computation, from the hitting set calculations. Sometimes the conflicts also do not result from reasoning, but from simulation, e.g., in spectrum-based diagnosis approaches \cite{abreu2007accuracy}, where ``conflicts'' correspond to so-called (program) spectra.\footnote{Roughly, a (program execution) spectrum is a set (sequence) of the system components (often: lines of program code) involved in one particular execution of the system (often: program). If an execution exhibits faulty behavior, the associated spectrum can be viewed as a conflict since at least one component in it must be faulty.}
		There are also systems for which the categorization is not clear-cut, such as GDE \cite{dekleer1987}, which interleaves conflict computation phases with diagnosis finding phases. 
		We finally note that any on-the-fly method can be used in a preliminary fashion (by precomputing, if possible, all conflicts and then using the collection of these conflicts instead of the on-the-fly reasoning); the inverse does not hold in general.  
		\item \emph{Centralized} vs.\ \emph{distributed}: Centralized algorithms (of which \emph{on-the-fly} and \emph{preliminary} instances exist)
		consider the diagnosis computation from conflicts as \emph{one} problem. Note that all of the works referenced when discussing conflict-based algorithms in \ref{enum:conflict-based_vs_direct} above are centralized.
		In contrast, distributed approaches (which are collectively \emph{preliminary}) 
		attempt to leverage the structure inherent in the collection of conflicts, e.g., linearity \cite{zhao2016deriving}, or equivalence classes based on conflict overlapping and disjointness \cite{zhao2013distributed}, to speed up the hitting set computation. More precisely, they seek to decompose the hitting set computation problem into suitable sub-problems, consider these sub-problems separately and finally merge the obtained partial results to determine the overall result in terms of the hitting sets. Obviously, the (potentially large) performance gains achieved by distributed approaches depend strongly on the ``richness'' of the exploitable structure in the conflict data. A drawback of such approaches is that a reasonable a-priori analysis for proper problem decomposition requires the precomputation of (a substantial set of) conflicts.
	\end{enumerate}  
	\item \emph{Black-box} vs.\ \emph{glass-box}: Black-box techniques, e.g., \cite{Reiter87,greiner1989correction,wotawa2001variant,Shchekotykhin2014,Rodler2015phd,jannach2016parallel,rodler2018statichs,rodler2020_rbfhs,rodler2020dx_rbfhs}, use the logical inference engine as is, as a pure oracle that answers queries (i.e., performs consistency checks) throughout the diagnosis computation process.\footnote{The distinction between black-box and glass-box techniques in this context was first suggested by Parsia et al.\ \cite{parsia2005debugging}.} Hence, black-box techniques are completely \emph{reasoner-independent} and able to use \emph{any logic} along with \emph{any reasoner} that can perform sound and complete consistency checks over this logic. Glass-box techniques \cite{kalyanpur2006thesis,schlobach2007debugging,baader2008axiom_pinpointing,Horridge2011a} try to make the determination of diagnoses more efficient by exploiting non-trivial modifications of the internals of the reasoner. These can be algorithmic modifications---e.g., the extraction of a conflict or conflict-related information as a byproduct of a negative consistency check \cite{parsia2005debugging}---as well as strategies based on gainful memory utilization---e.g., the TMS-based\footnote{TMS is a shortcut for \emph{truth maintenance system}.} bookkeeping of sets of logical axioms that are sufficient for particular inferred entailments to hold \cite{dekleer1986assumption,dekleer1987,dekleer1991focusing_prob_diag}. Glass-box approaches \emph{depend on} the (suitable adaptation of) \emph{one particular reasoner} that can deal with \emph{one particular (class of) logic(s)}.\footnote{Note that Baader and Penaloza \cite{baader2008axiom_pinpointing} provide a more general approach for extending different types of reasoners to extract debugging-relevant information throughout the reasoning process. The suitable adaptations however still need to be done on each particular reasoner in order to use it in a glass-box scenario.}
	While glass-box approaches are highly optimized for particular logics and can bring noticeable performance gains over black-box approaches in certain cases \cite{kalyanpur2006thesis,Horridge2011a}, advantages of black-box methods in the context of model-based diagnosis
	are their
	\emph{robustness} (no sophisticated, and potentially error-prone,
	modifications of complex reasoning algorithms), \emph{simplicity} (internals of reasoner irrelevant), \emph{generality} (applicable to diagnosis problem instances formulated over any knowledge representation formalism for which a sound and complete reasoner exists) and \emph{flexibility} (e.g., black-box methods can use a portfolio reasoning approach where the most efficient reasoner is used depending on the expressivity of the knowledge base that provides the reasoning context\footnote{Recall from Sec.~\ref{sec:avoidance_of_exp_reasoning} that DynamicHS performs its consistency checks 
	often with regard to significantly different parts of the knowledge base (diagnosed system),
	where the respective reasoning context is determined by (the labels along) one particular path in the hitting set tree. Although the logical expressivity 
	of the complete knowledge base (system description) might necessitate the use of a reasoner with unfavorable worst-case runtime, parts of the knowledge base might well fall into a lower logical expressivity class, 
	thus enabling the use of more performant reasoning approaches 
	that need not be complete wrt.\ the expressivity of the full knowledge base. E.g., if the goal is to diagnose a system described by a Description Logic \cite{DLHandbook} knowledge base, the full system descrition might be stated, say, in the logic $\mathcal{SROIQ}$ \cite{grau2008owl}, whereas parts thereof might be formulated, say, in $\mathcal{EL}$ \cite{baader2005EL}, which would allow to employ reasoners, say \cite{baader2005_CEL_reasoner} or \cite{kazakov2014}, that are particularly efficient for this more restricted language.}). 
	In-depth comparisons \cite{kalyanpur2006thesis,Horridge2011a} between black-box and glass-box strategies in the domain of knowledge-base debugging conclude 
	that, in terms of performance, black-box methods overall
	compete fairly reasonably with
	glass-box methods while offering a higher generality and
	being more easily and robustly implementable. 
	

	\item \emph{Holistic} vs.\ \emph{abstraction-based} vs.\ \emph{alteration-based}: Holistic approaches, to which most of the above-mentioned algorithms belong, 
	compute diagnoses by considering (the description of) the system to be diagnosed as is.
	In contrast, abstraction-based methods abstract from the original system by reformulating the problem in a more concise way, and alteration-based methods exploit modifications of the system description with well-known properties. The goal of these steps is to achieve better computation efficiency or to be able to solve problems that are otherwise too large or complex. Examples of abstraction-based techniques are \cite{mozetivc1991hierarchical,out1994construction,siddiqi2007hierarchical}, which pursue a stepwise hierarchical approach where, starting from a maximally abstract system description, successively more refined (in the sense of: more similar to the original system description) abstractions of those system parts 
	revealed to be abnormal while analyzing the abstract model are considered to obtain the diagnoses for the original system. The core principle behind these approaches is to reduce the search space by first using the abstract model and to reduce the search at the detailed level by excluding those system parts that have already been exonerated at the abstract level. An alteration-based technique is proposed in \cite{Siddiqi2011} where components in a system model are cloned in order to achieve a reduction of the generated system abstraction, while not losing relevant diagnostic information. Other alteration-based techniques safely discard considerable portions of the system description while maintaining full soundness and completeness wrt.\ the diagnostic task, e.g., in the knowledge-base debugging domain this can be achieved by using modules based on syntactic locality \cite{grau2008modular,sattler2009module}.



	
	%
	%
	%

	\item \label{related_work:dimension_statefulness} \emph{Stateful} vs.\ \emph{stateless} (during sequential diagnosis): Stateful approaches, e.g., \cite{dekleer1987,Siddiqi2011,Rodler2015phd,rodler2018statichs,penaloza2019}, maintain state (e.g., in terms of stored and later reused data structures) throughout a sequential diagnosis session. That is, data produced during one iteration constitutes an input to the execution of the next iteration, where the transition from one to the next iteration is defined by the incorporation of new information (resulting from a performed measurement) to the diagnosis problem. 
	Stateless algorithms, e.g., \cite{Reiter87,greiner1989correction,wotawa2001variant,xiangfu2006method,dekleer2011hitting,jannach2016parallel}, in contrast, do not propagate any information between two iterations. They can be thought of as getting a full reset after each iteration and starting completely anew, albeit with a modified input (previously considered diagnosis problem \emph{including additional measurements}), in the next iteration. 
\end{enumerate}
%
%
The bottom line of this categorization and discussion of diagnosis computation algorithms is:
\begin{itemize}
	\item Although in part highly different as regards algorithm type and properties, in principle, all algorithms have their right to exist; they are well-motivated by theoretical or practical problems or shortcomings of pre-existing algorithms.
	\item For most pairs of algorithms with a mutually different characterization in terms of the discussed dimensions \ref{related_work:dimension_completeness}--\ref{related_work:dimension_statefulness}, there will be problem classes that can be solved (more efficiently) with the one algorithm than with the other, and other problem classes for which the inverse holds. Simply put, the appropriateness of diagnosis computation algorithms is domain-dependent and different domains require different algorithmic techniques and algorithm properties. Thus, it is reasonable to expect that the one and only diagnosis algorithm that suits all possible model-based diagnosis scenarios best does not exist. \\
	Let us consider two examples: \emph{(1)}~Inv-HS-Tree appears to be not that well-suited for (spectrum-based) debugging problems focusing on \emph{procedural} software because it can be difficult (or even impossible) to enforce the execution of exactly the program traces required by the splitting strategy of the algorithm; this, however, is trivial for debugging problems involving \emph{declarative} knowledge bases. 
	STACCATO, on the other hand, does not ideally cater for knowledge-base debugging problems where a pre-computation of conflicts is often impractical and unnecessary (especially in a sequential diagnosis setting), but is very powerful in the (spectrum-based) software domain. \emph{(2)}~While the GDE and its successors are highly optimized and particularly well-performing for physical devices or digital circuits, they have not found widespread adoption in some other diagnosis domains such as knowledge-base debugging. A likely reason for that is the typically different justification\footnote{A \emph{justification} (for $\alpha$) is a minimal set of assumptions (axioms) that are sufficient for a particular entailment $\alpha$ to hold \cite{Horridge2008}. A justification for an unwanted entailment, such as $\bot$, is a minimal conflict. Depending on the context, justifications are sometimes referred to as \emph{environments} \cite{dekleer1987}.} structure inhering physical devices on the one hand and knowledge bases on the other. Specifically, the principle of \emph{locality} \cite{dekleer1990prime_implicates} (constraints for a particular system component only interact with constraints for physically adjacent components) appears to benefit justification bookkeeping strategies as used in GDE, whereas the potential strong non-locality and the related richer justification structure present in knowledge bases\footnote{As studied by Horridge et al.\ \cite{Horridge2012b}, the number of justifications for a single entailment $\alpha$ (and thus the number of conflicts) might well reach high three-digit numbers in knowledge-based systems.} can make such bookkeeping methods inefficient.\footnote{This issue was discussed in plenum during the ``Workshop on Principles of Diagnosis 2017'' (DX'17) in Brescia, Italy.} It seems that other techniques, such as HS-Tree, that abstain from protocoling computed justifications, work better in this particular domain. 
	\item If one algorithm only improves on another one's performance on all (reported) problems, then this often comes at the cost of sacrificing a (potentially useful) property of the improved algorithm. For instance, Inv-HS-Tree reduces the space-complexity of the diagnosis search from exponential to linear as compared to HS-Tree, but diagnoses cannot be enumerated best-first any longer. Other examples are distributed algorithms that can meliorate the diagnosis \emph{search} time significantly, but typically require a (usually expensive) precomputation of all conflicts, or glass-box algorithms that can notably improve the cost of diagnosis determination while, however, requiring systems expressible in one particular (class of) logic(s).\\
	Exceptions to this are particularly notable, i.e., when one algorithm improves on another one while preserving all its (desirable) properties.
	For instance, a compilation of the diagnosis problem to DNNF can be used whenever a compilation to OBDD is possible, and does not imply any disadvantages as far as ``standard'' diagnostic tasks are considered.\footnote{However, in other domains, such as system verification, tractable queries regarding OBDDs become intractable when using DNNF \cite{darwiche2001decomposable}.} This is also what DynamicHS aims to achieve: Enhancing HS-Tree for adoption in a sequential diagnosis setting while retaining the same guarantees (soundness, completeness, best-first property, general applicability). 
\end{itemize}

In terms of the described properties \ref{related_work:dimension_completeness}--\ref{related_work:dimension_statefulness} above, DynamicHS is (sound and) complete, (general) best-first, conflict-based (on-the-fly, centralized), black-box, holistic, and stateful. The algorithms most similar to DynamicHS as per these features, which have been reportedly used for sequential diagnosis according to the literature, are:\footnote{Differences in properties as compared to DynamicHS are italicized.} 
\begin{itemize}
	\item StaticHS \cite{rodler2018statichs}, which is sound, complete, (general) best-first, conflict-based (on-the-fly, centralized), black-box, holistic, and stateful: 
	The aim of StaticHS is orthogonal to the one of DynamicHS, as the former focuses on reducing the time spent (by the user) for measurement conduction 
	whereas the latter is geared to reduce the computation time. Moreover, DynamicHS admits more powerful tree pruning techniques than StaticHS \cite{Rodler2015phd}.
	\item HS-Tree \cite{Reiter87}, which is sound, complete, (general) best-first, conflict-based (on-the-fly, centralized), black-box, holistic, and \emph{stateless}: As extensively discussed, the only difference between HS-Tree and DynamicHS is that the latter maintains its state throughout sequential diagnosis whereas the former does not. As we show in our evaluations, the transition from statelessness (HS-Tree) to statefulness (DynamicHS) has a significant positive impact on the algorithm's runtime. And there is still potential for further improvement of DynamicHS, e.g., by exploiting the parallelization techniques proposed by \cite{jannach2016parallel}. 
	Besides the adoption of such techniques in DynamicHS's tree expansion phase (as shown by \shortcite{jannach2016parallel} for HS-Tree), they appear particularly appropriate for tree pruning (where nodes can be processed independently of one another) and for redundancy checking (where multiple independent reasoning operations can be necessary).
%
%
	\item RBF-HS and HBF-HS \cite{rodler2020_rbfhs,rodler2020dx_rbfhs}, which are sound, complete, (general) best-first, conflict-based (on-the-fly, centralized), black-box, holistic, and \emph{stateless}: 
	These two algorithms are specifically geared towards scenarios where the soundness, completeness as well as the best-first property are important and where memory is an issue, such as on low-end (e.g., IoT) devices or when facing particularly challenging diagnostic problems (e.g., with high-cardinality minimal diagnoses). In case the diagnosis problem at hand is solvable using HS-Tree given the available memory, the latter---and even more so DynamicHS, as we shall demonstrate in Sec.~\ref{sec:eval}---can be expected to be more time-efficient than RBF-HS and HBF-HS on average \cite{rodler2020_rbfhs}.
	\item GDE \cite{dekleer1987}, which is sound, complete, (general) best-first, conflict-based (centralized), \emph{glass-box}, holistic, and stateful: GDE has mainly been developed, optimized and used for the diagnosis of physical devices like circuits. The main difference to DynamicHS is that reasoning strategies drawing on a truth maintenance system (TMS) are employed. These store all justifications (environments) for already computed entailments and appear\footnote{See discussion above.} to be not the ideal approach for systems with a highly non-local structure and up to thousands of justifications per entailment \cite{Horridge2011a}.
	Moreover, when there is a large variety of logics used to conceptualize the (system) knowledge \cite{DLHandbook}, flexible black-box techniques, like HS-Tree and DynamicHS, that are (ad-hoc) usable for all these logics, appear particularly attractive.
	\item Inv-HS-Tree \cite{Shchekotykhin2014}, which is sound, complete, \emph{any-first}, \emph{direct}, black-box, holistic, and \emph{stateless}: 
	Due to its ability to use linear-search techniques, Inv-HS-Tree can solve problems for which HS-Tree runs out of memory \cite{Shchekotykhin2014}. This result remains valid for DynamicHS as its worst-case memory consumption is, in general, not lower than HS-Tree's. On the other hand, Inv-HS-Tree is any-first and can thus not be used for diagnosis tasks where a particular class of diagnoses (e.g., minimum-cardinality) or the most preferred diagnoses need to be found (first). Specifically, in the sequential diagnosis setting considered in this work, Inv-HS-Tree prevents 
	the reasonable use of early termination strategies (e.g., by stopping the sequential diagnosis session if some diagnosis exceeds a given probability threshold), which can be practical and save significant effort, yet with a high probability of finding the actual diagnosis.    
	\item SDA \cite{Siddiqi2011}, which is sound, \emph{direct}, \emph{glass-box}, \emph{abstraction-based}, stateful:\footnote{Whether SDA is complete and best-first (in the sense discussed in this work) cannot be unambiguously assessed as it does not compute \emph{a set of} diagnoses.} 
	SDA is a
	probabilistic hierarchical compilation-based method that eventually (at the end of sequential diagnosis) computes \emph{one} (best) diagnosis. In particular, no set of diagnoses is computed at intermediate steps. This method is crucially different to DynamicHS in the following regards: First, while the suggested abstractions have been shown to be highly beneficial as far as circuits are concerned, it remains unclear and seems highly non-trivial to apply similar techniques to logics more expressive than propositional logic and systems that are structurally different from typical circuit topologies (e.g., with non-local interactions between components). Second, SDA assumes the set of observables (possible measurement points) 
	to be explicitly given a-priori. However, this assumption is generally violated in diagnosis domains different from physical systems, e.g., in general knowledge bases, where measurement points (logical sentences that an oracle classifies as entailed or non-entailed) are mostly given in implicit form (i.e., derivable through reasoning) from the diagnosis problem \cite{DBLP:journals/corr/Rodler2017}, and a sample of \emph{multiple} precomputed diagnoses is often the only means for extracting reasonable measurement points \cite{Rodler2015phd}.  
	%
	%
\end{itemize}

\section{Evaluation}
\label{sec:eval}

\subsection{Objective}
The goals of our evaluation are
\begin{itemize}[noitemsep,topsep=0pt]
	\item to compare the performances of DynamicHS and HS-Tree in real-world sequential diagnosis scenarios under varying diagnostic settings,
	\item to understand the reasons of performance differences between both algorithms under particular consideration of DynamicHS's advanced techniques discussed in Sec.~\ref{sec:advanced_techniques_in_DynHS},
	\item to demonstrate the out-of-the-box
	applicability of 
	DynamicHS to diagnosis problems over
	different and highly expressive 
	knowledge representation languages.
\end{itemize}
Importantly, the goal is \emph{not} to show that DynamicHS is better than all or most diagnosis computation algorithms in literature, or that it is the best method in most or all diagnosis application domains, which is pointless (cf. Sec.~\ref{sec:related_work}). Rather, the intention is to show the advantage of preferring DynamicHS to HS-Tree in sequential diagnosis scenarios where HS-Tree is a state-of-the-art technique for diagnosis computation. In fact, HS-Tree is among the best available diagnosis computation techniques when a sound, complete, best-first, and reasoner-independent approach for arbitrary (and potentially highly expressive) logics is required.

One important domain where these requirements are essential is ontology and knowledge base\footnote{We will use the terms \emph{ontology} and \emph{knowledge base} interchangeably throughout this section. For the purposes of this paper, we consider both to be finite sets of axioms expressed in some monotonic logic, cf.\ $\mo$ in Sec.~\ref{sec:basics}.} debugging. In this field, practitioners and experts usually\footnote{The discussed requirements emerged from discussions with ontologists 
whom we interviewed (e.g., in the course of a tool demo at the International Conference on Biological Ontology 2018)
in order to develop and customize our ontology debugging tool \emph{OntoDebug} \cite{DBLP:conf/foiks/SchekotihinRS18,schekotihin2018protege} to match its users' needs.}, and especially in critical applications of ontologies such as medicine \cite{rector2011getting}, want a debugger to output 
exactly the faulty axioms that really explain the observed faults in the ontology (\emph{soundness} and \emph{completeness}) at the end of a debugging session.
%
%
%
%
In addition, experts often wish to perpetually monitor the most promising fault explanation throughout the debugging process (\emph{best-first property}) with the intention to stop the session early if they recognize the fault. As was recently studied by \cite{rodler2020mbd_sampling}, the use of best-first algorithms often also involves efficiency gains in sequential diagnosis as opposed to strategies that enumerate diagnoses in a different order. Apart from that, it is a big advantage for users of knowledge-based or semantic systems to have a debugging solution that works out of the box for different logical languages and with different logical reasoners (\emph{general applicability}, cf. \emph{black-box} property in Sec.~\ref{sec:related_work}). 
This is because
\emph{(i)}~ontologies are formulated in a myriad of different (Description) logics \cite{DLHandbook} with the aim to achieve the required expressivity for each ontology domain of interest at the least cost for inference, and \emph{(ii)}~highly specialized reasoners exist for different logics (cf., e.g., \cite{baader2005_CEL_reasoner,kazakov2014}), and being able to flexibly switch to the most efficient reasoner for a particular debugging problem 
can bring significant performance improvements.
%

Hence, we ran our experiments to compare DynamicHS with HS-Tree on a dataset of faulty real-world knowledge bases, for which HS-Tree is the method of choice according to literature, cf., e.g., \cite{kalyanpur2006thesis,schlobach2007debugging,Horridge2011a,meilicke2011thesis,Shchekotykhin2012,Rodler2015phd,fu2016graph,baader2018weakening}.

\setlength{\tabcolsep}{28pt}
\begin{table}[tbp]
	\renewcommand\arraystretch{1}
	\footnotesize
	\centering
	\caption{\small Dataset used in the experiments (sorted by the number of components/axioms of the diagnosis problem, 2nd column).}
	\label{tab:dataset}
	\begin{minipage}{0.98\linewidth}
		\centering
		\begin{tabular}{@{}lrlr@{}} 
			\toprule
			KB $\mo$				& $|\mo|$& expressivity \textsuperscript{\textbf{1)}} 		& \#D/min/max \textsuperscript{\textbf{2)}} \\ \midrule
			Koala (K) 
			& 42 		& $\mathcal{ALCON}^{(D)}$& 10/1/3     \\
			University (U) 
			& 50 		& $\mathcal{SOIN}^{(D)}$& 90/3/4      \\
			IT  
			& 
			140 		& $\mathcal{SROIQ}$& 1045/3/7	  \\
			UNI  
			\phantom{\textsuperscript{\textbf{4)}}}		
			& 
			142 		& $\mathcal{SROIQ}$& 1296/5/6	  \\
			Chemical (Ch) 
			& 144 		& $\mathcal{ALCHF}^{(D)}$& 6/1/3     \\
			MiniTambis (M) 
			\phantom{\textsuperscript{\textbf{4)}}}		
			& 173 		& $\mathcal{ALCN}$ 		& 48/3/3	  \\
			falcon-crs-sigkdd (fal-cs) 
			& 
			195 		& $\mathcal{ALCIF}^{(D)}$&  5/1/2 \\
			coma-conftool-sigkdd (com-cs) 
			& 
			338 		& $\mathcal{SIN}^{(D)}$&  570/3/10 \\
			ctxmatch-sigkdd-ekaw (ctx-se) 
			& 
			367 		& $\mathcal{SHIN}^{(D)}$&  26/1/5 \\
			hmatch-cmt-conftool (hma-cmc) 
			& 
			434 		& $\mathcal{SIN}^{(D)}$&  43/2/7 \\
			hmatch-conftool-cmt (hma-coc) 
			& 
			436 		& $\mathcal{SIN}^{(D)}$&  41/3/8 \\
			coma-cmt-conftool (com-cc) 
			& 
			442 		& $\mathcal{SIN}^{(D)}$&  905/2/8 \\
			ctxmatch-cmt-conftool (ctx-cc) 
			& 
			458 		& $\mathcal{SIN}^{(D)}$&  934*/2/16* \\
			falcon-conftool-ekaw (fal-ce) 
			& 
			465 		& $\mathcal{SHIN}^{(D)}$&  201/2/7 \\
			Transportation (T) 
			& 
			1300 		& $\mathcal{ALCH}^{(D)}$& 1782/6/9	  \\
			Economy (E) 
			& 1781 		& $\mathcal{ALCH}^{(D)}$& 864/4/8     \\
			DBpedia (D) 
			& 
			7228 		& $\mathcal{ALCHF}^{(D)}$& 7/1/1     \\
			Opengalen (O) 
			& 
			9664		& $\mathcal{ALEHIF}^{(D)}$& 110/2/6     \\
			CigaretteSmokeExposure (Cig) 
			& 
			26548 		& $\mathcal{SI}^{(D)}$& 1566*/4/7*	  \\
			Cton (C) 
			& 
			33203		& $\mathcal{SHF}$& 15/1/5     \\
			\bottomrule
		\end{tabular}
	\end{minipage}
	\renewcommand\arraystretch{1}
	\begin{minipage}{0.98\linewidth}
		\centering
		\setlength{\tabcolsep}{2pt}
		\begin{tabular}{@{}lp{13.2cm}@{}}
			\textbf{1):} & Description Logic expressivity: each calligraphic letter stands for a (set of) logical operator(s) that are allowed in the respective language, e.g., $\mathcal{C}$ denotes negation (``complement'') of concepts, for details see \cite{DLHandbook,DL_complexity};
			roughly, the more letters, the higher the expressivity of a logic and the complexity of reasoning for this logic tends to be.
			\\
			\textbf{2):} & \#D/min/max denotes the number/the minimal size/the maximal size of minimal diagnoses for the initial DPI $\dpi_0$ resulting from each input KB $\mo$. If tagged with a $^*$, a value signifies the number or size determined within 1200sec using HS-Tree (for problems where the finding of \emph{all} minimal diagnoses was impossible within reasonable time). 
		\end{tabular}
	\end{minipage}
\end{table}

\subsection{Dataset}
\label{sec:dataset}
The benchmark of 20 inconsistent or incoherent\footnote{A knowledge base $\mo$ is called \emph{incoherent} iff it entails that some predicate $p$ must always be \emph{false}; formally: $\mo \models \forall \mathbf{X} \lnot p(\mathbf{X})$ where $p$ is a predicate with arity $k$ and $\mathbf{X}$ a tuple of $k$ variables. With regard to ontologies, incoherence means that some class (unary predicate) must not have any instances, and if so, the ontology becomes inconsistent. 
} real-world ontologies we used for our experiments 
is given in Tab.~\ref{tab:dataset}.\footnote{The benchmark problems can be downloaded from \url{http://isbi.aau.at/ontodebug/evaluation}.}
Parts of this dataset have been investigated i.a.\ in \cite{kalyanpur2006thesis,qi2007measuring,Horridge2008,Stuckenschmidt2008,del2010modular,Shchekotykhin2012,Rodler2013,jannach2016parallel,Rodler2019userstudy}. For all ontologies whose name (Tab.~\ref{tab:dataset}, column~1) starts with an upper case letter, their faultiness is a result of (only) modeling errors of humans developing the ontology. The problems in ontologies termed by lowercase names, in contrast, result from automatically merging two human-modeled ontologies which describe the same domain in different ways. This merging process, which involves a matching system integrating two source ontologies by adding additional (potentially faulty) axioms stating relationships between concepts from both source ontologies, 
is called \emph{ontology alignment} \cite{euzenat2011ontology} and can lead to numerous independent faults in the resulting merged ontology. In our dataset, the domain described by the merged ontologies is ``conference organization'' (describing concepts such a papers, authors, program committee members, reviewers, etc.) and the faulty merged ontologies are results produced by four different matching systems \cite{oaei2006}. The names of the respective ontologies in Tab.~\ref{tab:dataset} follow the scheme $\mathsf{name1}-\mathsf{name2}-\mathsf{name3}$ where $\mathsf{name1}$ corresponds to the name of the used matching system, and $\mathsf{name2}$ as well as $\mathsf{name3}$, respectively, to the name of the first and second source ontology to be merged. 

As Tab.~\ref{tab:dataset} shows, the ontologies in the dataset cover a spectrum of different problem sizes (number of axioms or components; column 2), logical expressivities (which determine the complexity of consistency checking; column 3),  
as well as diagnostic structures (number and size of minimal diagnoses; column 4). Note that the complexity of consistency checks
over the logics in Tab.~\ref{tab:dataset} ranges from EXPTIME-complete to 2-NEXPTIME-complete \cite{grau2008owl,DL_complexity}.
Hence, from the point of view of model-based reasoning, ontology debugging problems represent a particularly challenging diagnosis domain as they usually deal with harder logics
than more traditional diagnosis problems (which often use propositional knowledge representation languages that are not beyond {NP-complete}).

\subsection{Experiment Settings}
\label{sec:exp_settings}
To study the performance and robustness of DynamicHS under varying circumstances, we considered a range of different \emph{diagnosis scenarios} in our experiments. 
A diagnosis scenario is defined by the set of inputs given to Alg.~\ref{algo:sequential_diagnosis}, i.e., by an initial DPI $\dpi_0$, a number $\ld$ of minimal diagnoses to be computed per sequential diagnosis iteration, a probability measure $\pr$ assigning fault probabilities to system components (axioms in the knowledge base), as well as a heuristic $\qqm$ used for measurement selection.
The DPIs $\dpi_0$ for our tests were defined as $\tuple{\mo,\emptyset,\emptyset,\emptyset}$, one for each $\mo$ in Tab.~\ref{tab:dataset}. That is, the task was to find a minimal set of axioms (faulty components) responsible for the inconsistency or incoherency of $\mo$, without any background knowledge or measurements initially given.
For the parameter $\ld$ we used the values $\{2,4,6,10,20,30\}$.
The fault probability $\pr(\tax)$ of each axiom (component) $\tax \in \mo$ was chosen uniformly at random from $(0,1)$. 
As measurement selection heuristics, we used $\{\text{ENT},\text{SPL},\text{MPS}\}$, where ENT is the well-known entropy function proposed by \cite{dekleer1987}, SPL is the ``split-in-half'' strategy \cite{moret1982decision,Shchekotykhin2012}, and MPS (``most probable singleton'') \cite{DBLP:journals/corr/Rodler16a,rodler17dx_activelearning} is the selection measure that overall performed most favorably in the evaluations conducted by \cite{DBLP:conf/ruleml/RodlerS18}. Roughly speaking, given a set of minimal diagnoses, ENT / SPL / MPS selects a discriminating measurement point (cf.\ Sec.~\ref{sec:computational_issues_in_SD}) which maximizes the expected information gain / maximizes the worst-case diagnosis elimination rate / maximizes the probability of the elimination of a maximal number of diagnoses.
For each diagnosis scenario, we randomly selected 20 different minimal diagnoses, each of which was used as the target solution (actual diagnosis; cf.\ Sec.~\ref{sec:diagnoses}) in one sequential diagnosis session (execution of Alg.~\ref{algo:sequential_diagnosis}). 
Finally, we executed DynamicHS (Alg.~\ref{algo:sequential_diagnosis} with setting $\mathit{dynamic} = \true$) and HS-Tree (setting $\mathit{dynamic} = \false$) to find the same 20 target solutions for each diagnosis scenario.
%
As a Description Logic reasoner, we adopted Pellet \cite{Sirin2007}, and we implemented the \textsc{findMinConflict} function by means of the QuickXPlain algorithm \cite{junker04,rodler2020qx}.

\subsection{Experiment Results}
\label{sec:exp_results}
We first analyze the performance of DynamicHS compared to HS-Tree in Sec.~\ref{sec:performance_analysis}, and then, in Sec.~\ref{sec:eval_advanced_tehniques}, explain the findings and give additional insights by examining the impact of the advanced techniques incorporated into DynamicHS discussed in Sec.~\ref{sec:advanced_techniques_in_DynHS}.\footnote{The raw data obtained from our experiments can be downloaded from \url{http://isbi.aau.at/ontodebug/evaluation}. Moreover, our implementation of DynamicHS can be accessed under \url{https://bit.ly/348Ny8f}.}

\begin{figure}[tbp]
	\includegraphics[width=\columnwidth]{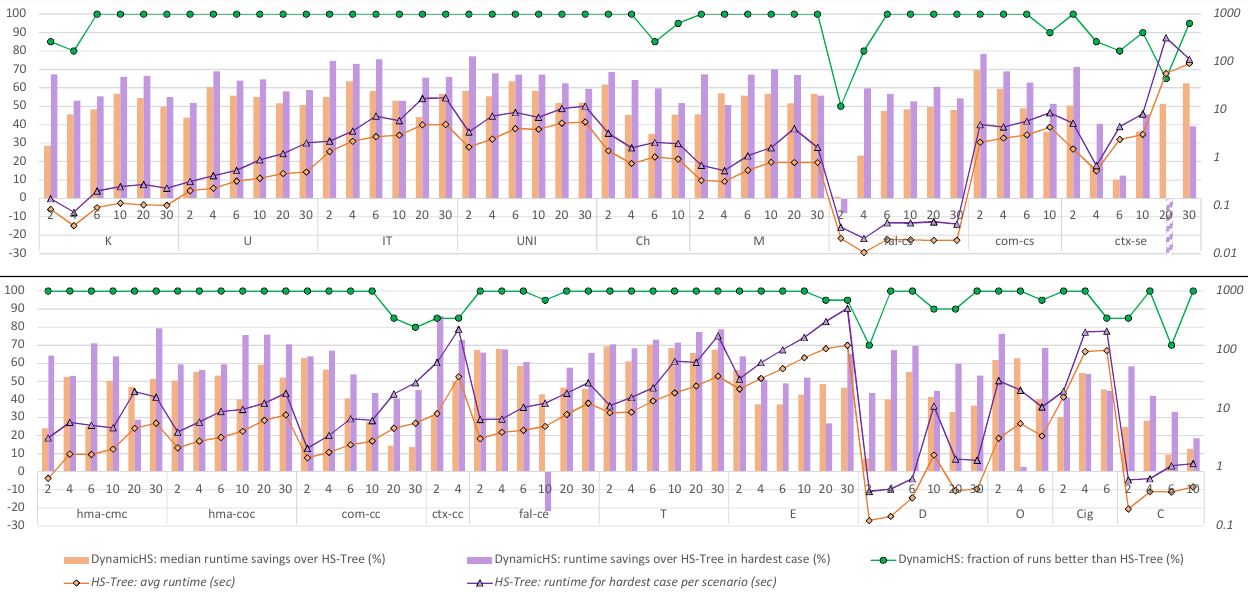}
	\caption{Experiment results (DynamicHS vs.\ HS-Tree) for measurement selection heuristic ENT: x-axis shows faulty ontologies $\mo$ from Table~\ref{tab:dataset} and number $\ld$ of leading diagnoses computed per call of hitting set algorithm (iteration of while-loop in Alg.~\ref{algo:sequential_diagnosis}). 
		For each ($\mo$,$\ld$) scenario, \emph{(left y-axis)} the orange bar shows the median \% runtime savings (over 20 sequential diagnosis sessions) achieved by DynamicHS, the violet bar indicates the \% runtime savings exhibited by DynamicHS for the hardest case of the scenario, and the green dot shows the fraction (in \%) of the 20 runs for the scenario where DynamicHS manifested lower runtime than HS-Tree; \emph{(right y-axis)} the orange diamond / violet triangle depicts the average / maximal runtime (in sec) of HS-Tree over the 20 sequential sessions. In the legend, features written in normal / italic font are plotted wrt.\ the left / right y-axis. The striped violet bar is not fully shown for clarity of the figure: In this case, HS-Tree required 86\,\% less time than DynamicHS.}
	\label{fig:results_median_savings_ENT}
\end{figure}

\subsubsection{Performance Analysis}
\label{sec:performance_analysis}
The comparison of runtime performance between HS-Tree and DynamicHS over the 20 sequential diagnosis sessions for each diagnosis scenario is shown in Fig.~\ref{fig:results_median_savings_ENT} for the heuristic ENT, in Fig.~\ref{fig:results_median_savings_SPL} for the heuristic SPL, and in Fig.~\ref{fig:results_median_savings_MPS} for the heuristic MPS.\footnote{For some knowledge bases in Tab.~\ref{tab:dataset}, the figures do not include data for all settings of $\ld$. The reason 
is that the experiment runs for these knowledge bases terminated with errors (for both algorithms). More specifically, the cause of the error was either an internal reasoner exception 
(recall that both algorithms use the reasoner as a black-box oracle, which is why this error is outside the range of influence of the algorithms) or an out-of-memory error. Regarding the latter, it is important to note that both algorithms give strong guarantees in terms of soundness, completeness and the best-first property, which requires them to keep track of all possible branches of the search tree, which in turn can become problematic if minimal diagnoses of high cardinality exist and tree depth is high \cite{Shchekotykhin2014}. Please note that the diagnostic structure (fourth column of Tab.~\ref{tab:dataset}) indicates only the sizes of minimal diagnoses for the \emph{initial} DPI $\dpi_0$ and that minimal diagnoses computed throughout a sequential diagnosis session grow in size (cf.\ Property~\ref{property:impact_of_DPI_transition}.\ref{property:impact_of_DPI_transition:enum:diags_can_only_grow}) and can reach cardinalities that exceed the quoted initial values substantially.} As the orange bars in these three figures testify, DynamicHS exhibits substantial median
runtime savings
over HS-Tree in almost all considered scenarios and for all three heuristics ENT, SPL and MPS.\footnote{The computation time reduction of DynamicHS over HS-Tree was statistically significant wrt.\ the level $\alpha := 0.05$ in almost all (97\,\%) of the diagnosis scenarios, as verified by Wilcoxon Signed Rank Tests. More specifically, statistical significance could be ascertained in all scenarios except for the following: ctx-se for ENT 20 as well as SPL 10 and 20; fal-cs for ENT 2; com-cc for MPS 10 and 20; and D for MPS 10, 20, 30 as well as for SPL 20.} Only in two (of the overall 312) scenarios slight losses against HS-Tree in terms of median runtime could be registered, namely $-2$\,\% for (SPL, D, 30) and $-3$\,\% for (MPS, fal-ce, 30). However, in terms of absolute runtime, DynamicHS never required more than 0.8 sec more time than HS-Tree in any of the 20 runs in the former scenario, and never more that 12 sec more in the latter.     

\begin{figure}[tbp]
	\includegraphics[width=\columnwidth]{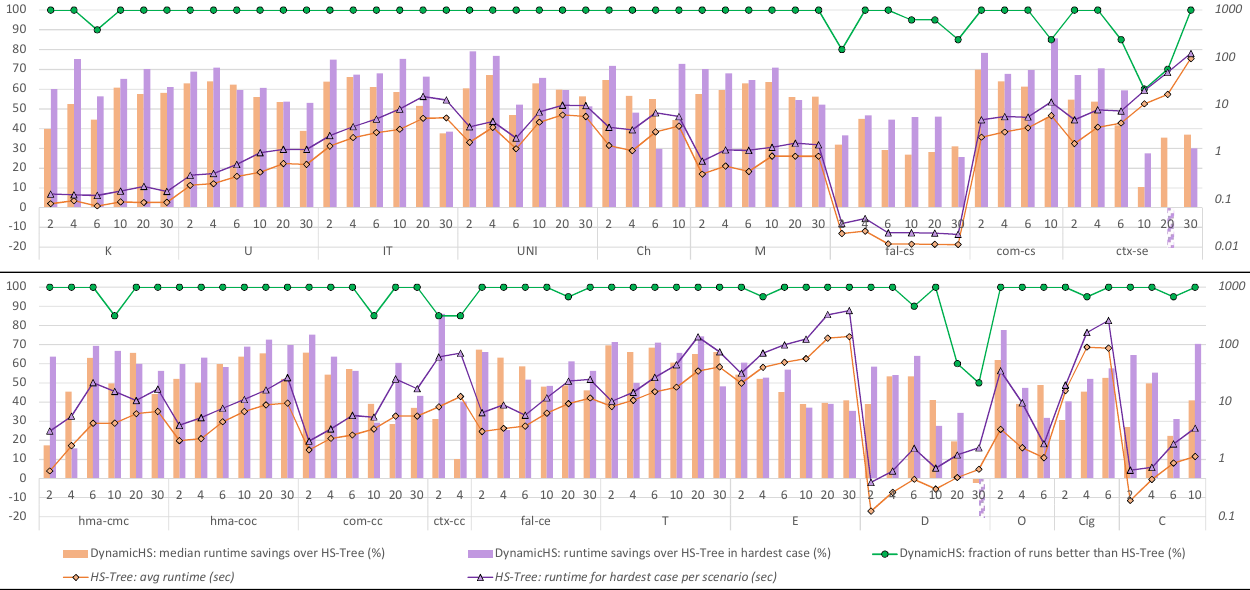}
	\caption{Experiment results (DynamicHS vs.\ HS-Tree) for measurement selection heuristic SPL: x-axis shows faulty ontologies $\mo$ from Table~\ref{tab:dataset} and number $\ld$ of leading diagnoses computed per call of hitting set algorithm (iteration of while-loop in Alg.~\ref{algo:sequential_diagnosis}). 
		For each ($\mo$,$\ld$) scenario, \emph{(left y-axis)} the orange bar shows the median \% runtime savings (over 20 sequential diagnosis sessions) achieved by DynamicHS, the violet bar indicates the \% runtime savings exhibited by DynamicHS for the hardest case of the scenario, and the green dot shows the fraction (in \%) of the 20 runs for the scenario where DynamicHS manifested lower runtime than HS-Tree; \emph{(right y-axis)} the orange diamond / violet triangle depicts the average / maximal runtime (in sec) of HS-Tree over the 20 sequential sessions. In the legend, features written in normal / italic font are plotted wrt.\ the left / right y-axis. The striped violet bars are not fully shown for clarity of the figure: In these cases, HS-Tree required 76\,\% (ctx-se, 20) and 32\,\% (D, 30) less time than DynamicHS.}
	\label{fig:results_median_savings_SPL}
\end{figure}

The left boxplot in Fig.~\ref{fig:boxplots_all_stats} depicts the overall distribution of the median savings (orange bars in Figs.~\ref{fig:results_median_savings_ENT}--\ref{fig:results_median_savings_MPS}) for the three heuristics.
As we can see from the boxplot, median savings reach maximal values of 
70\,\% / 70\,\% / 75\,\% 
for ENT / SPL / MPS, and amount to median values of 50\,\% / 52\,\% / 52\,\%. 
Moreover, the observed median savings do not differ in a statistically significant way between the three heuristics.\footnote{We verified this by an ANOVA (p-value: 0.26).}  
Also, the data does not show a general tendency that the percentage of runtime reduction depends on the number of minimal diagnoses $\ld$ computed per sequential diagnosis iteration;\footnote{We checked this by a Kruskal-Wallis test, one for each of the three heuristics.
The p-values are 0.81 (ENT), 0.22 (SPL), and 0.81 (MPS).} see the boxplots in Fig.~\ref{fig:boxplots_savings_for_values_of_ld}. For the values 2 / 4 / 6 / 10 / 20 / 30 of $\ld$, over all scenarios, we observe median runtime savings of 49\,\% / 51\,\% / 49\,\% / 45\,\% / 46\,\% / 49\,\% for ENT, of 45\,\% / 50\,\% / 45\,\% / 45\,\% / 48\,\% / 49\,\%  for SPL, and of 49\,\% / 51\,\% / 49\,\% / 45\,\% / 46\,\% / 49\,\% for MPS.

Considering the fraction of the 20 runs per diagnosis scenario where DynamicHS outperformed HS-Tree (green circles in Figs.~\ref{fig:results_median_savings_ENT}--\ref{fig:results_median_savings_MPS} and right boxplot in Fig.~\ref{fig:boxplots_all_stats}), we observe values equal or very close to 100\,\% in almost all scenarios. That is, sequential diagnosis sessions where HS-Tree is preferable to DynamicHS in terms of runtime are very rare, regardless of the adopted measurement selection strategy. In terms of numbers, DynamicHS requires less runtime than HS-Tree in 96\,\% (all data), 96\,\% (ENT), 97\,\% (SPL), and 95\,\% (MPS) of the scenarios.  

Finally, since the hardest sequential diagnosis session in terms of the required time for diagnosis computation was up to one order of magnitude harder than the average session of one and the same diagnosis scenario (cf.\ orange diamonds and violet triangles in Figs.~\ref{fig:results_median_savings_ENT}--\ref{fig:results_median_savings_MPS}), we also included the savings of DynamicHS over HS-Tree for these hardest cases in the plots of Figs.~\ref{fig:results_median_savings_ENT}--\ref{fig:results_median_savings_MPS} (see the violet bars) and summarized their distribution in the middle boxplot of Fig.~\ref{fig:boxplots_all_stats}. As readable from the boxplot, savings in these hardest cases reach maximal values up to 86\,\% / 86\,\% / 89\,\% for ENT / SPL / MPS, and amount to median values of 60\,\% / 60\,\% / 70\,\%. 
In fact, the observed savings even exceed the median savings (orange bars in Figs.~\ref{fig:results_median_savings_ENT}--\ref{fig:results_median_savings_MPS}) in 83\,\% (all data), 86\,\% (ENT), 75\,\% (SPL), and 87\,\% (MPS) of the diagnosis scenarios (compare the left and middle boxplots in Fig.~\ref{fig:boxplots_all_stats}).This indicates that savings achieved by DynamicHS versus HS-Tree tend to be above-average when 
diagnosis computation
takes long. Tab.~\ref{tab:times+savings_for_some_hardest_cases} enumerates for some of the per-scenario hardest cases the runtimes (rounded to full seconds) of both algorithms as well as the percental savings achieved by DynamicHS over HS-Tree. For instance, we see from the table that a runtime reduction of more than five and a half minutes could be obtained by DynamicHS for (ENT, E, 30), or that HS-Tree with (more than) 1:44 minutes computation time consumed more than nine times as much time as DynamicHS (less than 12 seconds) for the hardest case of scenario (MPS, ctx-cc, 4).

\begin{figure}[tbp]
	\includegraphics[width=\columnwidth]{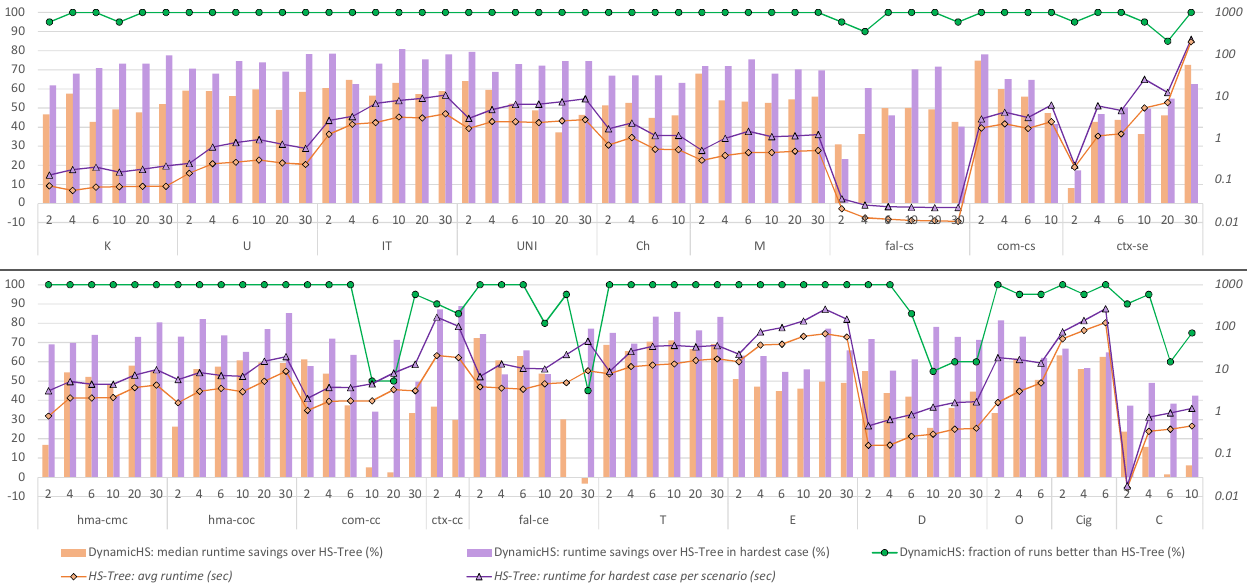}
	\caption{Experiment results (DynamicHS vs.\ HS-Tree) for measurement selection heuristic MPS: x-axis shows faulty ontologies $\mo$ from Table~\ref{tab:dataset} and number $\ld$ of leading diagnoses computed per call of hitting set algorithm (iteration of while-loop in Alg.~\ref{algo:sequential_diagnosis}). 
		For each ($\mo$,$\ld$) scenario, \emph{(left y-axis)} the orange bar shows the median \% runtime savings (over 20 sequential diagnosis sessions) achieved by DynamicHS, the violet bar indicates the \% runtime savings exhibited by DynamicHS for the hardest case of the scenario, and the green dot shows the fraction (in \%) of the 20 runs for the scenario where DynamicHS manifested lower runtime than HS-Tree; \emph{(right y-axis)} the orange diamond / violet triangle depicts the average / maximal runtime (in sec) of HS-Tree over the 20 sequential sessions. In the legend, features written in normal / italic font are plotted wrt.\ the left / right y-axis.}
	\label{fig:results_median_savings_MPS}
\end{figure}

\begin{figure}[tbp]
	\includegraphics[width=0.985\columnwidth]{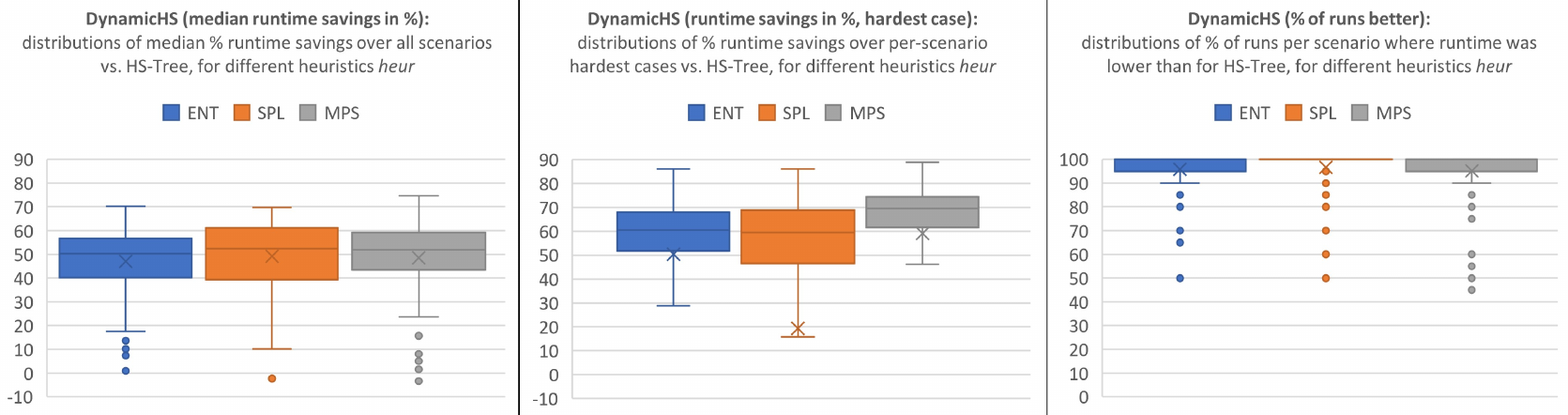}
	\caption{Boxplots showing the distribution of data 
		given in Figs.~\ref{fig:results_median_savings_ENT}--\ref{fig:results_median_savings_MPS}. The left / middle / right plot relates to the orange bars / violet bars / green circles in Figs.~\ref{fig:results_median_savings_ENT}--\ref{fig:results_median_savings_MPS}, and the blue / orange / gray boxplots refer to Fig.~\ref{fig:results_median_savings_ENT} / Fig.~\ref{fig:results_median_savings_SPL} / Fig.~\ref{fig:results_median_savings_MPS}.} 
	\label{fig:boxplots_all_stats}
\end{figure}

\begin{figure}[tbp]
	\includegraphics[width=0.985\columnwidth]{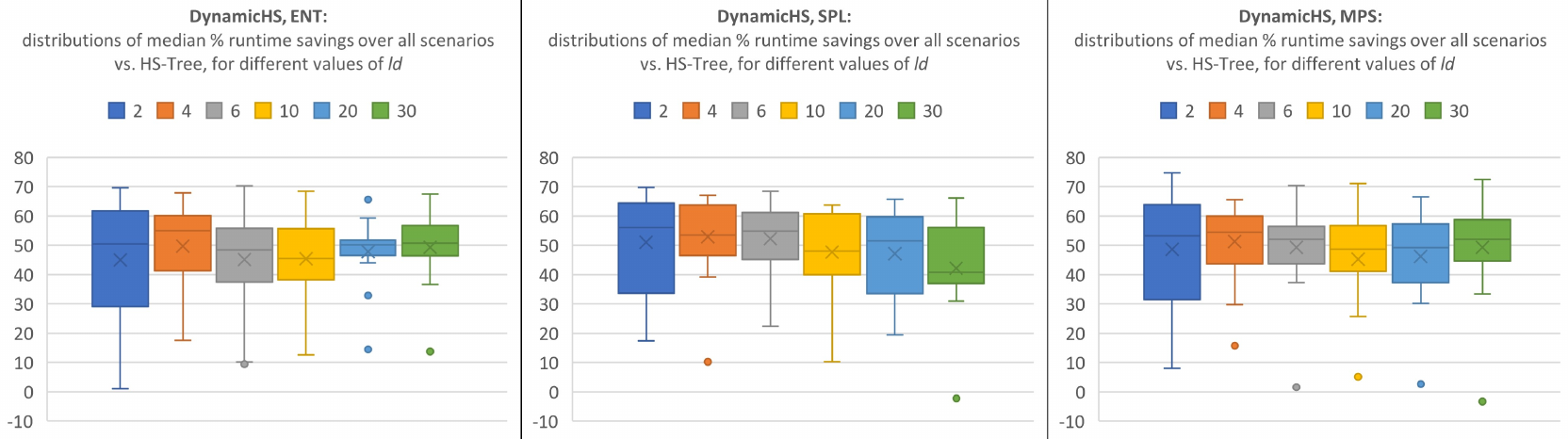}
	\caption{Boxplots showing the distributions of data (orange bars) displayed by Figs.~\ref{fig:results_median_savings_ENT}--\ref{fig:results_median_savings_MPS} for different values of $\ld$: Left / middle / right plot relates to Fig.~\ref{fig:results_median_savings_ENT} / Fig.~\ref{fig:results_median_savings_SPL} / Fig.~\ref{fig:results_median_savings_MPS}.}
	\label{fig:boxplots_savings_for_values_of_ld}
\end{figure}

\setlength{\tabcolsep}{10pt}
\begin{table}[tbp]
	\centering
	\footnotesize
	\caption{\small Runtimes of both algorithms and savings achieved by DynamicHS for five of the per-scenario hardest cases for each of the measurement selection heuristics ENT, SPL and MPS. The leftmost column lists scenarios in the form ``$\qqm$, $\mo$, $\ld$'' where $\qqm$ is the heuristic, $\mo$ the knowledge base from Tab.~\ref{tab:dataset}, and $\ld$ the number of leading diagnoses computed per sequential diagnosis iteration.}
	\label{tab:times+savings_for_some_hardest_cases}
	\begin{tabular}{@{}lccc@{}}
		\toprule
		scenario         & \multicolumn{1}{l}{runtime HS-Tree (min)} & \multicolumn{1}{l}{runtime DynamicHS (min)} & savings of DynamicHS (\%) \\
		\midrule
		ENT, E, 30       & 8:30                                      & 2:58                                        & 65                                          \\
		ENT, ctx-cc, 4   & 3:43                                      & 1:01                                        & 73                                          \\
		ENT, T, 30       & 2:52                                      & 0:36                                        & 79                                          \\
		ENT, ctx-cc, 2   & 1:00                                      & 0:08                                        & 86                                          \\
		ENT, O, 2        & 0:29                                      & 0:07                                        & 76                                          \\
		SPL, Cig, 6      & 4:22                                      & 1:51                                        & 58                                          \\
		SPL, T, 20       & 2:16                                      & 0:35                                        & 74                                          \\
		SPL, com-cc, 30  & 1:00                                      & 0:08                                        & 86                                          \\
		SPL, T, 10       & 0:44                                      & 0:15                                        & 66                                          \\
		SPL, O, 2        & 0:34                                      & 0:08                                        & 77                                          \\
		MPS, E, 20       & 4:23                                      & 1:00                                        & 77                                          \\
		MPS, ctx-cc, 2   & 2:51                                      & 0:22                                        & 87                                          \\
		MPS, ctx-cc, 4   & 1:44                                      & 0:12                                        & 89                                          \\
		MPS, T, 10       & 0:36                                      & 0:05                                        & 86                                          \\
		MPS, hma-coc, 30 & 0:20                                      & 0:03                                        & 85  										\\
		\bottomrule                                  
	\end{tabular}
\end{table}

\subsubsection{Analysis of Advanced Techniques}
\label{sec:eval_advanced_tehniques}
As the performance comparison between DynamicHS and HS-Tree at the end of Example~\ref{ex:algo_description} already suggested, there are two major sources for the runtime savings obtained by means of DynamicHS: \emph{(I)}~fewer expensive reasoning operations and \emph{(II)}~saved effort for tree (re)construction. The price to pay for these reductions is \emph{(III)}~the storage of the existing hitting set tree (including duplicate nodes) throughout the diagnosis session and \emph{(IV)}~the execution of regular tree pruning actions. We examine these four aspects more closely in the next three paragraphs based on our experiment results. The first paragraph addresses aspect (I), the second analyzes aspects (II) and (III), and the third focuses on aspect (IV).\vspace{5pt}

\begin{figure}[tbp]
	\includegraphics[width=\columnwidth]{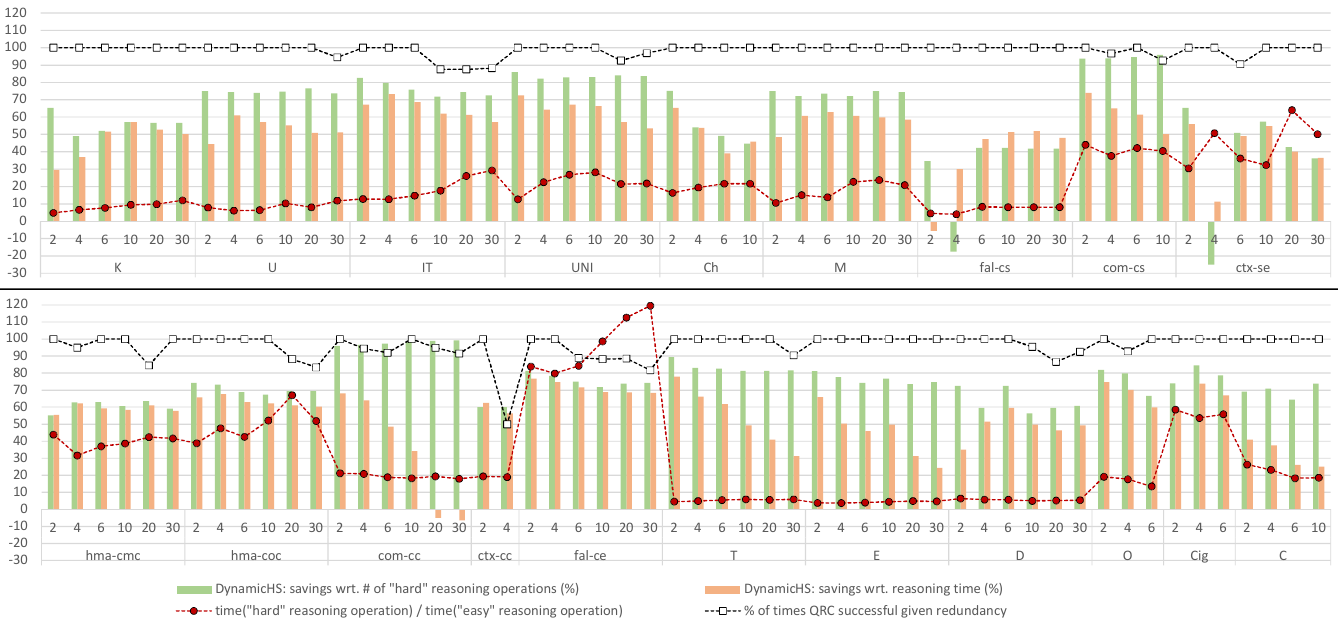}
	\caption{\small Analysis of DynamicHS wrt.\ avoidance of reasoning and efficient redundancy checking: x-axis shows faulty ontologies $\mo$ from Table~\ref{tab:dataset} and number $\ld$ of leading diagnoses computed per call of hitting set algorithm (iteration of while-loop in Alg.~\ref{algo:sequential_diagnosis}). The plot shows data for the heuristic ENT. 
		All values depicted are medians over the 20 sequential diagnosis sessions executed per diagnosis scenario. See Sec.~\ref{sec:avoidance_of_exp_reasoning} for a definition of ``hard'' and ``easy'' reasoning operations; ``QRC'' refers to quick redundancy check, see Sec.~\ref{sec:efficient_redundancy_checking}; ``savings'' refers to the savings over HS-Tree.}
	\label{fig:avoid_reasoning_effect_ENT}
\end{figure}

\noindent \emph{Avoidance of Expensive Reasoning and Efficient Redundancy Checking:} 
Fig.~\ref{fig:avoid_reasoning_effect_ENT} summarizes several statistics that illustrate aspect (I) in more detail. The reduction of costly reasoning operations (``hard'' \textsc{findMinConflict} calls, cf.\ Sec.~\ref{sec:avoidance_of_exp_reasoning}) is achieved by trading them against more efficient reasoning operations (``easy'' \textsc{findMinConflict} calls, cf.\ Sec.~\ref{sec:avoidance_of_exp_reasoning}).  
%
The red circles in the figure, which show how much harder an average ``hard'' 
call is than an average ``easy'' one in each diagnosis scenario, attest that this strategy of swapping ``hard'' for ``easy'' calls is indeed plausible, as the former are a median of 19 times and up to 120 times as time-consuming as the latter.

The green bars in Fig.~\ref{fig:avoid_reasoning_effect_ENT} reveal that significant savings in terms of ``hard'' calls are indeed consistently generated by DynamicHS. More specifically, in more than 98\,\% of the diagnosis scenarios, the median relative savings wrt.\ ``hard'' \textsc{findMinConflict} calls in comparison to HS-Tree are higher than 30\,\%; over all scenarios, median and maximal savings amount to 74\,\% and 99\,\%, respectively. 
%

Rather unsurprisingly, these savings wrt.\ the number of ``hard'' \textsc{findMinConflict} calls translate to a similarly substantial reduction of runtime spent for reasoning operations (orange bars in Fig.~\ref{fig:avoid_reasoning_effect_ENT}) manifested by DynamicHS. Savings in terms of the runtime dedicated to reasoning are achieved in 97\,\% of the diagnosis scenarios and amount to median and maximal values of 57\,\% and 78\,\%, respectively. The main reason why the diminution in terms reasoning time is most times lower than the decrease of ``hard'' reasoning operations is that DynamicHS, as opposed to HS-Tree, makes use of ``easy'' reasoner calls to compensate for these saved more expensive calls. That is, these ``easy'' calls account to a large extent for the difference between green and orange bars. As evident in some scenarios, e.g., for the knowledge base fal-cs, however, it is also possible that savings in reasoning time top savings in terms of ``hard'' \textsc{findMinConflict} calls.
This can happen, e.g., when DynamicHS saves many ``medium'' \textsc{findMinConflict} calls against HS-Tree (cf.\ Example~\ref{ex:hard_med_easy_consistency_checks_in_example_DPI} and Tab.~\ref{tab:stats_reasoning_example}) in addition to its savings wrt.\ ``hard'' ones.

Last but not least, also the used efficient redundancy checking technique 
incorporated in DynamicHS contributes to the achieved runtime savings. 
In fact, the adoption of the QRC (quick redundancy check, cf.\ Sec.~\ref{sec:efficient_redundancy_checking}) attempts to minimize the (``easy'') reasoning operations necessary to decide the redundancy of a particular node (wrt.\ a particular given minimal conflict). The white boxes in Fig.~\ref{fig:avoid_reasoning_effect_ENT} indicate that this strategy is very powerful in that the QRC detects redundancy almost always when redundancy is actually given, thus allowing the alorithm to skip the more expensive CRC (complete redundancy check, cf.\ Sec.~\ref{sec:efficient_redundancy_checking}). One could say that the QRC appears to be ``almost complete'' in our experiments. In numbers, we observe that the QRC detected \emph{all} redundancies in 75\,\%, at least nine of ten redundancies in 89\,\%, and at least eight of ten redundancies in 99\,\% of the scenarios.\vspace{5pt}

\noindent\emph{Statefulness (DynamicHS) vs.\ Statelessness (HS-Tree):} 
Let us now consider Fig.~\ref{fig:time_space_tradeoff_ENT}, which allows to investigate aspects (II) and (III). The figure depicts, per diagnosis scenario, the median factors how many more nodes HS-Tree had to process than DynamicHS (grey bars) and how many more nodes DynamicHS needed to 
store than HS-Tree (yellow bars). The fact that the majority of all bars attains values greater than 1 demonstrates that DynamicHS---expectedly---tends to trade less time for more space. It keeps the produced search tree in memory and utilizes the information contained in this tree to allow for more efficient diagnosis computation in the next sequential diagnosis iteration. We can identify from the figure that in two thirds of the diagnosis scenarios, the trade-off achieved by DynamicHS is favorable in the sense that the factor of more memory used by DynamicHS 
is less than the factor of more nodes processed by HS-Tree (yellow bars smaller than grey ones).  

Moreover, Fig.~\ref{fig:time_space_tradeoff_ENT} (yellow bars) reveals that DynamicHS required less memory than HS-Tree in 14\,\% of the scenarios, exhibited less than 25\,\% memory overhead in 35\,\% and less than 50\,\% overhead in 52\,\% of the scenarios, and consumed less than twice the memory of HS-Tree in 82\,\% of the diagnosis scenarios. Circumstances where DynamicHS can require even less memory than HS-Tree are when few or no duplicate nodes exist (e.g., when minimal conflicts are mostly disjoint), when DynamicHS's hitting set tree after tree updates is (largely) equal to the one produced by HS-Tree (cf.\ Sec.~\ref{sec:lazy_updating_policy}), 
and when DynamicHS happens to compute or select for reuse more ``favorable'' conflicts 
than HS-Tree in the course of node labeling.\footnote{The order in which conflicts are computed and selected for node labeling is neither controlled by HS-Tree nor by DynamicHS. If one of the algorithms happens to compute smaller minimal conflicts that are used to label nodes at the top of the hitting set tree and when the hitting set trees produced by both algorithms are not computed to their entirety (as is the case when $\ld$ diagnoses have been computed before the tree is complete), then the tree with the smaller conflicts at the top can be smaller than the one generated by the other algorithm.} 
In less than 10\,\% of the scenarios, however, the memory overhead shown by DynamicHS was substantial, reaching values of 4 or more times the amount of memory consumed by HS-Tree. Hence, although not observed in our evaluation where either both or none of the algorithms ran out of memory, there may be cases where DynamicHS is not applicable due to too little available memory while HS-Tree is. In general, this can be remedied by adding a mechanism to DynamicHS which simply discards the entire stored hitting set tree and starts over from scratch building a new tree whenever a certain fraction of the available memory has been exhausted. In other words, the stateful (DynamicHS) strategy can be flexibly and straightforwardly switched to a stateless (HS-Tree) one.

Fig.~\ref{fig:time_space_tradeoff_ENT} (grey bars) also shows the benefit of the stateful strategy pursued by DynamicHS. That is, the overhead in terms of the time expended for tree (re)construction incurred when using HS-Tree instead of DynamicHS is significant in the majority of scenarios. 
In numbers, HS-Tree had to process 
at least 1.5 times as many nodes as DynamicHS in 78\,\% of the scenarios, at least twice as many in 42\,\%, and at least three times as many in 5\,\% of the diagnosis scenarios.
%
%

Finally, note that aspects (I) and (II), i.e., DynamicHS's reduction of time for reasoning and its savings in terms of tree construction costs, are orthogonal in the following sense: 
Even if one of these aspects turns out to be unfavorable from the viewpoint of DynamicHS in comparison to HS-Tree, the other aspect is not affected by that in general. For instance, in our experiments we observed three scenarios (fal-cs, 2; com-cc, 20; com-cc, 30) where DynamicHS actually required slightly more time for reasoning than HS-Tree (cf.\ orange bars in Fig.~\ref{fig:avoid_reasoning_effect_ENT}), i.e., aspect (I) was unfavorable for DynamicHS. Since DynamicHS however did save node processing time in these cases, i.e., aspect (II) was favorable, the overall computation time was still lower for DynamicHS. Similarly, there are cases (cf., e.g., ctx-cc, 4) where aspect (II) is unfavorable while aspect (I) is favorable, again resulting in overall time savings of DynamicHS. \vspace{5pt}   

\begin{figure}[tbp]
	\includegraphics[width=\columnwidth]{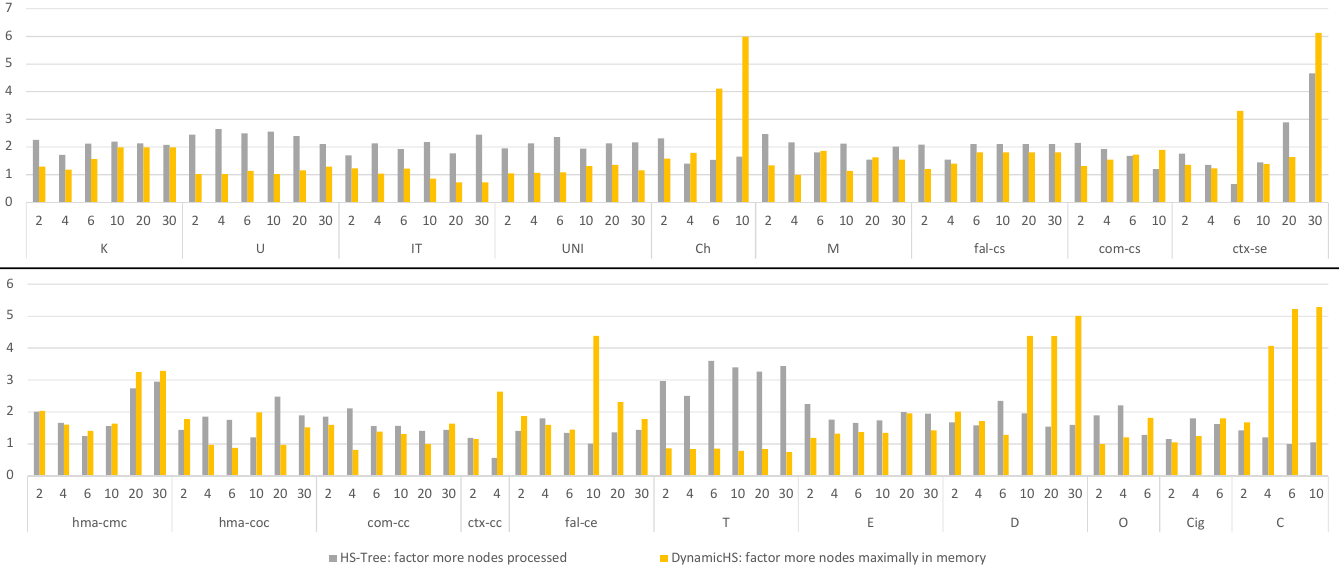}
	\caption{\small Comparison of stateful (DynamicHS) versus stateless (HS-Tree) hitting set tree management in terms of the nodes processed (time efficiency) and the nodes maximally in memory (space efficiency): x-axis shows faulty ontologies $\mo$ from Table~\ref{tab:dataset} and number $\ld$ of leading diagnoses computed per call of hitting set algorithm (iteration of while-loop in Alg.~\ref{algo:sequential_diagnosis}). The plot shows data for the heuristic ENT. All values depicted are medians over the 20 sequential diagnosis sessions executed per diagnosis scenario.}  
	\label{fig:time_space_tradeoff_ENT}
\end{figure}

\begin{figure}[tbp]
	\includegraphics[width=\columnwidth]{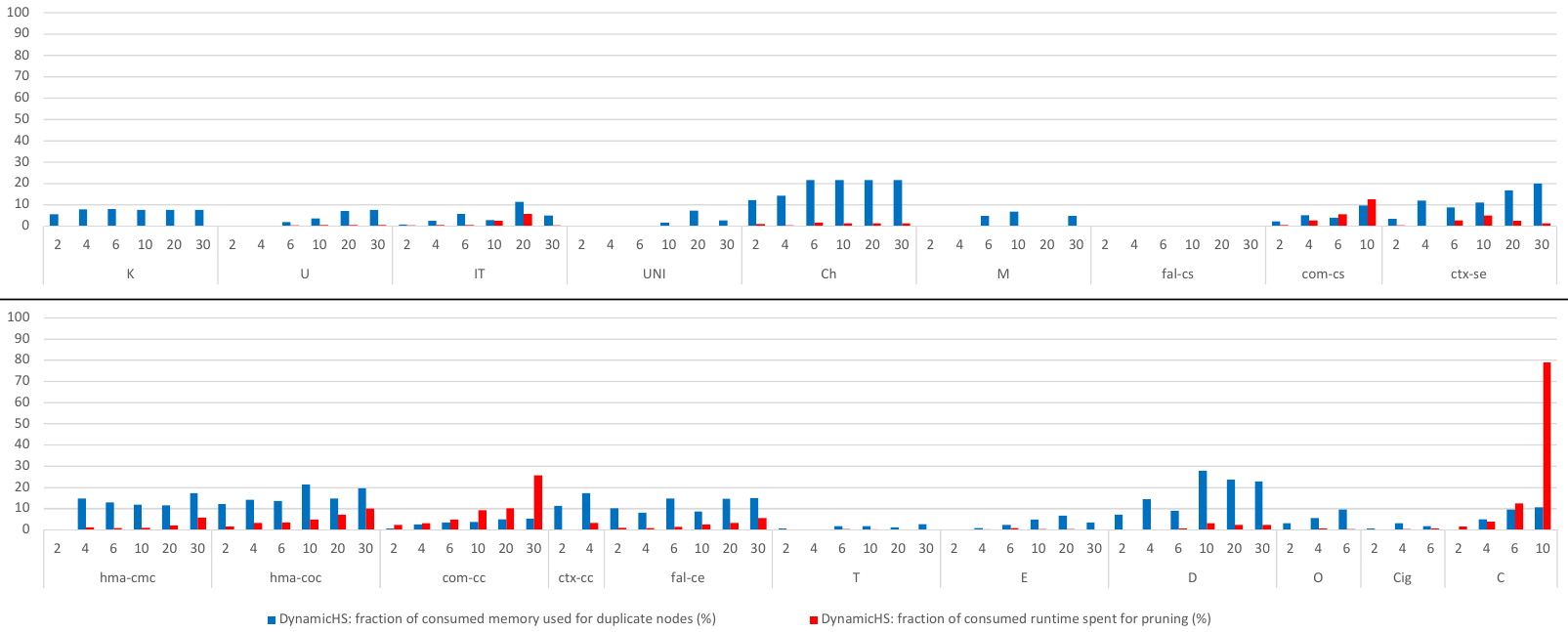}
	\caption{\small Analysis of DynamicHS wrt.\ space-efficiency of duplicate storage and time-efficiency of pruning actions: 
		x-axis shows faulty ontologies $\mo$ from Table~\ref{tab:dataset} and number $\ld$ of leading diagnoses computed per call of hitting set algorithm (iteration of while-loop in Alg.~\ref{algo:sequential_diagnosis}). The plot shows data for the heuristic ENT. All values depicted are medians over the 20 sequential diagnosis sessions executed per diagnosis scenario. }
	\label{fig:duplicates+pruning_ENT}
\end{figure}

\noindent\emph{Duplicates and Pruning:}
Fig.~\ref{fig:duplicates+pruning_ENT} illuminates aspect (IV). It displays the median fraction of the memory consumed by DynamicHS that is used for the storage of duplicate nodes per diagnosis scenario (blue bars). 
We can see that duplicate nodes account for less than 10\,\% of the used memory in 69\,\% of the diagnosis scenarios, for less than 20\,\% in 92\,\%, and for less than 30\,\% in 100\,\% of the scenarios. In other words, the memory overhead caused by duplicate nodes does not exceed 12\,\% / 25\,\% / 39\,\% in 69\,\% / 92\,\% / 100\,\% of the scenarios. 
%
Consequently, the additional memory consumption attributable to duplicates stays within acceptable bounds in all cases. This can be due to not too many existing duplicates (e.g., if minimal conflicts tend to be disjoint) or due to DynamicHS's memory-efficient storage of duplicates (cf.\ Sec.~\ref{sec:space-saving_dup_storage}). The latter requires duplicates to be reconstructed on demand in the course of pruning and associated node replacement actions (cf.\ Sec.~\ref{sec:space-saving_dup_storage} and Example~\ref{ex:advanced_techniques}). The relative computational expense of these pruning actions in comparison to the overall computation time of DynamicHS is described by the red bars in Fig.~\ref{fig:duplicates+pruning_ENT}. These tell that the relative computation time spent for pruning is negligible (less than 1\,\%) in 61\,\% of the diagnosis scenarios, marginal (less than 5\,\%) in 89\,\%, and small (less than 10\,\%) in 95\,\% of the scenarios. This overall fairly low overhead for pruning, which already includes the time for reconstructing duplicate nodes, testifies that DynamicHS's duplicate storage and reconstruction as well as its pruning techniques that completely dispense with costly reasoner calls are reasonable strategies for dealing with the statefulness of the hitting set tree.

Only in two cases, the time for pruning exceeds 20\,\% of DynamicHS's total computation time; in one of these cases (C, 10) it is very high, reaching almost 80\,\%, a possible sign that significant portions of the tree have become out-of-date through lazy updating (cf.\ Sec.~\ref{sec:lazy_updating_policy}). Note, however, that DynamicHS overall still saves time compared to HS-Tree in this scenario (cf.\ Fig.~\ref{fig:results_median_savings_ENT}) since the time spent later for more intensive pruning is counterbalanced with time saved earlier by skipping pruning actions (lazy updating).

\section{Conclusion}
\label{sec:conclusion}
In this work we proposed DynamicHS, an optimization of Reiter's seminal HS-Tree algorithm geared towards sequential diagnosis use cases. Since, in sequential diagnosis, the addressed diagnosis problem is subject to successive change in terms of newly acquired system information in the form of observations or measurements, the main rationale behind DynamicHS is the adoption of a reuse-and-adapt principle, which maintains the already produced diagnosis search data structure and appropriately updates it whenever the diagnosis problem is modified. The goal of this stateful diagnostic search is to avoid expensive redundant operations involved in the reconstruction of discarded search data structures when using (the stateless) HS-Tree.
Comprehensive experiments on a benchmark of 20 real-world diagnosis problems, from which we generated more than 6000 different sequential diagnosis problems under different settings as regards the number of diagnoses computed and the measurement selection method used, revealed that \emph{(1)}~DynamicHS leads to a (statistically significantly) better time performance than HS-Tree in almost all cases, \emph{(2)}~the time saved when using DynamicHS instead of HS-Tree is substantial in almost all cases and attains median and maximal values of 52\,\% and 89\,\%, respectively, \emph{(3)}~the memory overhead (due to the statefulness) of DynamicHS was reasonable in the majority of cases and DynamicHS was successfully applicable whenever HS-Tree was, and \emph{(4)}~the performance of DynamicHS compared to HS-Tree was particularly favorable in the hardest cases where the use of DynamicHS rather than HS-Tree could avoid time overheads of up to more than 800\,\%. Notably, DynamicHS achieves these performance improvements while preserving all desirable properties of HS-Tree. That is, like HS-Tree, DynamicHS is sound and complete, computes diagnoses in best-first order, and is generally applicable in that it is independent of the diagnosis domain, of the diagnosis problem structure, of the (monotonic) logic used to describe the diagnosed system, and of the adopted (sound and complete) logical inference engine.      

\section*{Acknowledgments}
I thank Manuel Herold for the excellent Java implementation of the DynamicHS algorithm and of the experiments. I am also grateful to Wolfgang Schmid for coordinating the execution of our experiments on our servers. This work was supported by the Austrian Science Fund (FWF), contract P-32445-N38.


%
%



\vskip 0.2in

\end{document}